\documentclass{article}

\PassOptionsToPackage{numbers}{natbib}

\usepackage[preprint]{neurips_2025}
\usepackage{amsmath}
\usepackage{algorithm}
\usepackage{algpseudocode}
\usepackage{float}      % For [H] placement
\usepackage{amssymb}    % For \mathbb
\usepackage{subcaption}
\usepackage{graphicx}
\usepackage{amsthm}
\usepackage{enumitem}
\usepackage{bm}
\usepackage{amsthm}
\usepackage{hyperref}
\usepackage{amsfonts}
\usepackage{xfrac}
\usepackage{subcaption}
\usepackage{comment}
\usepackage{float}
\usepackage{etoolbox}
\usepackage{wrapfig}
\usepackage{stmaryrd}
\newtheorem{theorem}{Theorem}
\newtheorem{lemma}{Lemma}
\newtheorem{definition}{Definition}
\newtheorem{remark}{Remark}
\usepackage[skins,breakable]{tcolorbox}
\usepackage{xcolor}
\usepackage{geometry}
\usepackage{palatino}

\DeclareMathOperator*{\argmin}{arg\,min}

\newcommand*{\inftybar}{\medskip \begin{center}\rule[.5ex]{15ex}{0.5pt} ~$\infty$~ \rule[.5ex]{15ex}{0.5pt}\end{center}}

\usepackage[utf8]{inputenc} % allow utf-8 input
\usepackage[T1]{fontenc}    % use 8-bit T1 fonts
\usepackage{hyperref}       % hyperlinks
\usepackage{url}            % simple URL typesetting
\usepackage{booktabs}       % professional-quality tables
\usepackage{amsfonts}       % blackboard math symbols
\usepackage{nicefrac}       % compact symbols for 1/2, etc.
\usepackage{microtype}      % microtypography
\usepackage{xcolor}         % colors

\newcommand{\Z}{\mathbb{Z}}

\definecolor{paleblue}{RGB}{240, 248, 255}      % Alice blue
\definecolor{paleviolet}{RGB}{245, 240, 255}    % Lavender-ish
\definecolor{palegray}{RGB}{248, 248, 250}      % Very light gray-blue

\title{One Rank at a Time: \\ Cascading Error Dynamics in Sequential Learning}

\author{%
  Mahtab Alizadeh Vandchali\thanks{Equal contribution.} \And Fangshuo (Jasper) Liao\thanks{Equal contribution.} \And Anastasios Kyrillidis \\
  Department of Computer Science\\
  Rice University\\
  \texttt{ma202@rice.edu, fl15@rice.edu, anastasios@rice.edu}
}

\begin{document}

\maketitle

\begin{abstract}
Sequential learning --where complex tasks are broken down into simpler, hierarchical components-- has emerged as a paradigm in AI. 
This paper views sequential learning through the lens of low-rank linear regression, focusing specifically on how errors propagate when learning rank-1 subspaces sequentially. 
We present an analysis framework that decomposes the learning process into a series of rank-1 estimation problems, where each subsequent estimation depends on the accuracy of previous steps. 
Our contribution is a characterization of the error propagation in this sequential process, establishing bounds on how errors --e.g., due to limited computational budgets and finite precision-- affect the overall model accuracy. 
We prove that these errors compound in predictable ways, with implications for both algorithmic design and stability guarantees. 
    % Experimental results on linear regression validate our theoretical bounds.
\end{abstract}

\clearpage

% Now define the bubble box (corrected version)
\newtcolorbox{papersummary}{
    enhanced,                    % Now this will work!
    colback=blue!5!gray!10,    % Blend: 5% blue + 10% gray + 85% white
    colframe=blue!20!gray!30,  % Matching frame
    arc=10pt,
    boxrule=1pt,
    left=10pt,
    right=10pt,
    top=8pt,
    bottom=8pt,
    fonttitle=\bfseries\large,
    title style={colback=white, colframe=gray!50},
    attach boxed title to top center={yshift=-2mm},
    boxed title style={arc=5pt, boxrule=1pt}
}

% Usage in your document
\begin{papersummary}
\begin{center}
    \textbf{Paper Meta-Analysis Card of ``One Rank at a Time: \\ Cascading Error Dynamics in Sequential Learning''}
\end{center}
\textbf{Authors:} Mahtab Alizadeh Vandchali, Fangshuo (Jasper) Liao, Anastasios Kyrillidis

\textbf{Institution:} Rice CS

\medskip
\textbf{Research genesis:} Current sequential learning approaches lack theoretical understanding of how numerical errors compound through hierarchical decomposition. While methods like LoRA demonstrate empirical success, the question of error propagation in sequential rank-1 subspace learning remains uncharacterized. 

\medskip
\textbf{Thought process:} Here, we first focus on the linear case as a foundational and more tractable setting to develop theoretical understanding. We recognized that sequential learning can be mathematically formulated as iterative rank-1 matrix deflation, where each step depends on the accuracy of previous estimations. This led us to decompose the problem into studying how approximation errors from individual rank-1 subroutines propagate through the sequential process.

\medskip
\textbf{Methodology:} The core innovation lies in characterizing error propagation through recursive bounds that depend on the spectral properties of the data matrix. The analysis decomposes overall error into ground-truth approximation, propagation, and optimization components, establishing that errors compound multiplicatively with factors determined by singular value gaps ($\mathcal{T}_k^\star$) and matrix condition numbers.

\medskip
\textbf{What remains open:} Extension to non-linear transformations and complex neural architectures represents the primary theoretical challenge. The optimal allocation of resources across sequential components lacks complete characterization.

\medskip
\textbf{Limitations:} The theoretical framework is constrained to linear low-rank regression settings, limiting direct application to modern deep learning architectures. Experimental validation focuses on relatively simple scenarios (feedforward networks, basic classification), and the analysis assumes specific spectral properties that may not hold in general practice.

\medskip
\textbf{Practical considerations:} Implementation requires careful management of iteration budgets, with theoretical results suggesting front-loading computational effort on early components. The approach offers adaptive rank determination capabilities but demands more total training iterations compared to simultaneous optimization.

\medskip
\textbf{Theoretical implications:} The analysis reveals that error propagation follows predictable mathematical patterns, challenging the view that sequential approaches are inherently less stable than simultaneous methods and providing a foundation for principled algorithm design in hierarchical learning systems.

\medskip
\begin{minipage}[t]{0.45\textwidth}
\vspace{-2.5cm}
\textbf{Date:} 05/22/2025\\
\textbf{Correspondence:} anastasios@rice.edu\\
\end{minipage}
\hfill
\begin{minipage}[b]{0.45\textwidth}
\vspace{0pt}
\begin{flushright}
\includegraphics[width=0.7\linewidth]{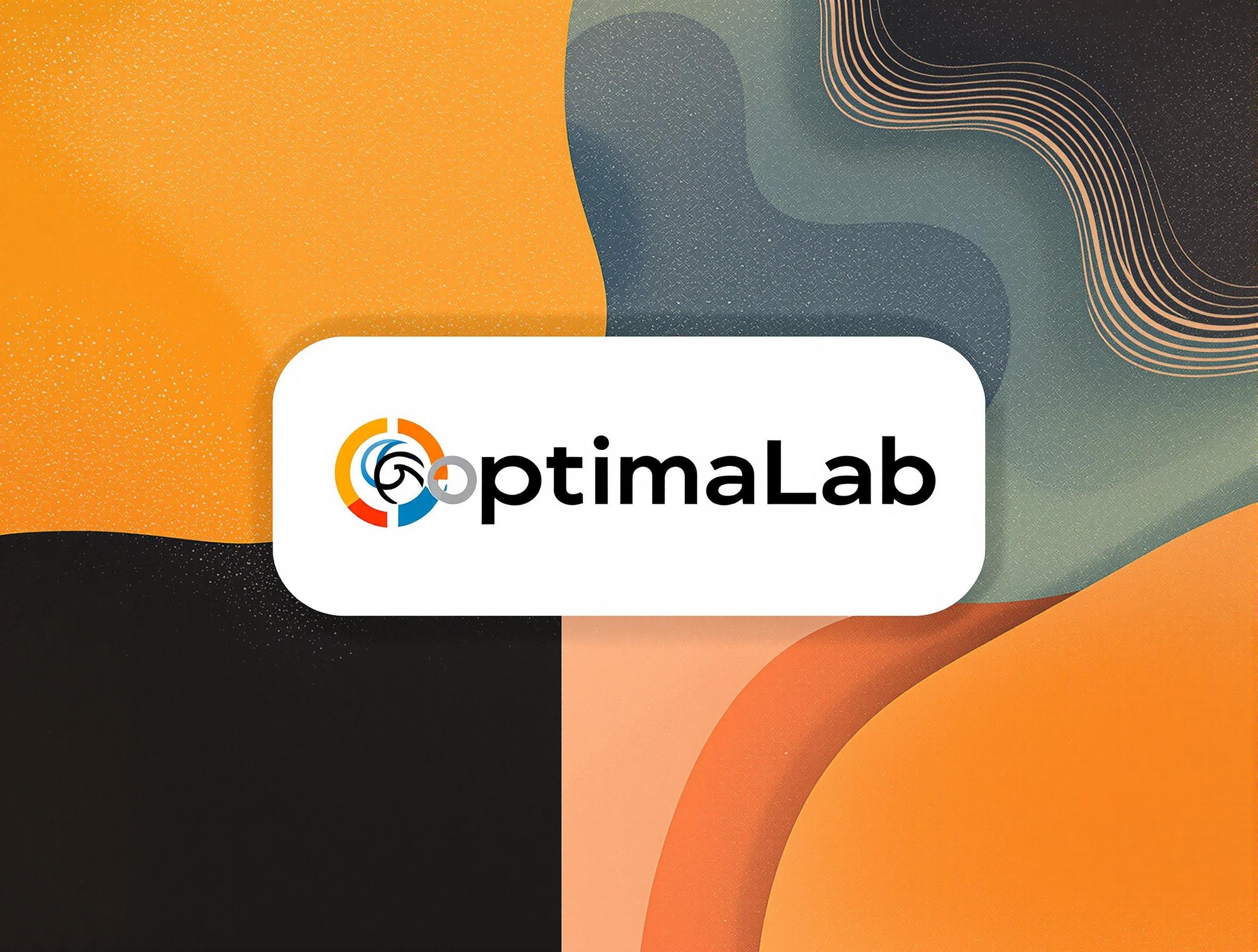}
\end{flushright}
\end{minipage}
\end{papersummary}

\clearpage

\vspace{-0.2cm}
\section{Introduction}
\vspace{-0.2cm}

Sequential learning \cite{barto2003recent, pateria2021hierarchical, sahni2017learning, ostapenko2024towards, sordoni2024joint, page2024multi} is a concept found in cognitive science that posits that learning could be structured as a series of stages or levels
\cite{shapere1964structure, richardson2019models, okano2000learning, carey1999cognitive, ericsson1993role, national2000people}. 
%Theory on sequential learning traces back to early cognitive models with an emphasis on how knowledge structures are built incrementally.
Sequential learning is especially relevant when the notion of ``skills'' is orthogonal or correlated to each other, or even layered hierarchically \cite{farajtabar2020orthogonal, chaudhry2020continual}. 
For instance, in a multitask learning environment \cite{caruana1997multitask, ruder2017overview, andrychowicz2016learning, crawshaw2020multi}, basic skills might serve for common tasks, while specialized skills might be required for specific tasks. % \textcolor{red}{Citations}. 
%Orthogonality comes into play when these skills do not interfere with each other. 

Understanding fully sequential learning is an open question, even for simple models:
researchers study not just how AI systems learn, but why they fail to learn when they do, and under what conditions they can learn better \cite{zhang2017understanding, bjorck2018understanding, baldi2013understanding,
sun2020optimization, 
liu2020understanding, 
zhou2022understanding, wang2021rethinking, yang2020rethinking}.
Just to provide a non-exhaustive list of recent efforts: \cite{zhao2024learning} presents a sequential learning strategy on videos and text for sentiment analysis, where learning features sequentially --from simpler to more complex-- led to better performance. \cite{mcalister2024sequential} combines deep learning with symbolic AI to tackle sequential learning tasks. 
\cite{bian2023multi} deals with recommendation systems, such as those used by Netflix or Amazon and propose a method for these systems to learn from sequences of user interactions with multiple types of data (e.g., text, images, videos) for better recommendations.
%Yet, theory studying the ``why'' and ``how'' is still missing.

\textbf{Our setting.}
We focus on low-rank subspaces as feature representations \cite{sanyal2018robustness}.
Low-rank models are compelling due to their interpretable solutions that capture influential factors in the data \cite{jolliffe1995rotation, koren2009matrix, vidal2014low}. 
Assuming low-rank linear regression \cite{candes2012exact, recht2010guaranteed, rohde2011estimation, ha2015robust} and given input $\mathbf{X} = \left[\mathbf{x}_1,\dots,\mathbf{x}_n\right] \in \mathbb{R}^{d \times n}$, the goal is to approximate the relationship between a dependent variable $\mathbf{Y} = \left[\mathbf{y}_1,\dots,\mathbf{y}_n\right]\in \mathbb{R}^{m \times n}$ and independent unknown variables $\mathbf{B} \in \mathbb{R}^{m \times r}$ and $\mathbf{A} \in \mathbb{R}^{r \times d}$ that result into a lower rank $r \ll \min(m, d)$ matrix $\mathbf{W} = 
\mathbf{B} \mathbf{A} \in \mathbb{R}^{m \times d}$ such that:
\begin{align*}
\mathbf{Y} \approx \mathbf{B} \mathbf{A} \mathbf{X}  = \mathbf{W} \mathbf{X}.   
\end{align*}
Here, $(\mathbf{x}_i, \mathbf{y}_i)$ represents a data sample of a dataset $\mathcal{D} := \{(\mathbf{x}_i, \mathbf{y}_i)\}_{i=1}^n$. 
Conceptually, the matrix $\mathbf{W}$ projects the original $d$ features onto a $m$-dimensional space, given that first $\mathbf{x}_i$ is passed through a $r$-dimensional ``funnel'', thus reducing the dimensionality and complexity of the model. 

In view of this, we utilize linear low-rank regression as a framework to study sequential learning processes, and how errors propagate through sequential rank-1 subspace learning. 
Algorithmically, the approach we consider relates to deflation in PCA \cite{hotelling1933analysis, mackey2008deflation, zhang2006schur, sriperumbudur2007sparse, saad1988projection, danisman2014comparison}.
That is, given a a rank-$1$ estimate of $\mathbf{W}$ based on label $\mathbf{Y}_k$
\begin{align}\label{eq:rank_1_equation}
    \mathbf{a}_k, \mathbf{b}_k = \operatorname*{arg\,min}_{\mathbf{a} \in \mathbb{R}^d, \mathbf{b} \in \mathbb{R}^m}\tfrac{1}{2}\left\| \mathbf{Y}_k - \mathbf{b} \mathbf{a}^{\top} \mathbf{X} \right\|_F^2.
\end{align}
Our approach starts off with $\mathbf{Y}_1 = \mathbf{Y}$ and obtain $\mathbf{a}_1,\mathbf{b}_1$. The matrix $\mathbf{Y}_1$ is further processed to exist on a ``subspace'' where the contributions of $(\mathbf{a}_1, \mathbf{b}_1)$ are removed:
$\mathbf{Y}_{2} := \mathbf{Y}_1 - \mathbf{b}_1 \mathbf{a}_1^\top \mathbf{X}$.
This process is repeated by applying sequentially rank-1 updates on the deflated matrix, which leads to an approximation of the second pair $(\mathbf{a}_2, \mathbf{b}_2)$, and so on.
Overall:
\begin{equation}
    \label{eq:rank1_deflation}
    \begin{gathered}
    \mathbf{Y}_1 = \mathbf{Y};\quad (\mathbf{a}_k, \mathbf{b}_k) = \texttt{rank-1}(\mathbf{Y}_k, \mathbf{X}, t); \quad \mathbf{Y}_{k+1} := \mathbf{Y}_k - \mathbf{b}_k \mathbf{a}_k^\top \mathbf{X},
    \end{gathered}
\end{equation}
where $\texttt{rank-1}(\mathbf{Y}_k, \mathbf{X}, t)$ returns an approximation of a rank-1 estimate in \eqref{eq:rank_1_equation} that minimizes the mean-squared error $\frac{1}{2} \left\| \mathbf{Y}_k - \mathbf{b} \mathbf{a}^{\top} \mathbf{X} \right\|_F^2$, using $t$ iterations.
We then estimate the subsequent subspaces by running the same rank-1 algorithm repetitively.

\textbf{Motivation.} 
While subspace tracking and estimate has a long history (see \cite{yang1995projection, vaswani2018robust, balzano2018streaming, peng2023ideal} and references to these works), low-rank subspaces have recently gained attention due to emerging applications in AI. 
Parameter-efficient fine-tuning (PEFT) methods, such as LoRA \cite{hu2021lora}, have demonstrated that representing weight updates as low-rank matrices can effectively adapt large language models, while maintaining performance. 
But even beyond AI, recommendation systems \cite{koren2009matrix} require real-time updates, and sequential low-rank modifications offer a computationally efficient way to incorporate new user-item interactions.
This validates the hypothesis that complex transformations can be approximated through a series of low-rank updates. 
However, to our knowledge, the theoretical understanding of how errors accumulate in such sequential approximations remains limited.

%The practical importance of this analysis extends beyond neural networks. 
%In recommendation systems \cite{koren2009matrix}, where real-time updates are crucial, sequential low-rank modifications offer a computationally efficient way to incorporate new user-item interactions. 
%Similarly, in online matrix completion \cite{sun2016guaranteed}, streaming data necessitates rapid updates to low-rank approximations. 
%These applications cannot afford to recompute full matrix factorizations with each update, making sequential approaches attractive. 
%However, the numerical stability and error propagation in such sequential updates directly impact the quality of recommendations and predictions.

%From a theoretical perspective, understanding error propagation in sequential rank-1 updates provides insights into the fundamental trade-offs between computational efficiency and approximation accuracy. 
%While simultaneous optimization of all rank components, as in traditional matrix factorization approaches \cite{jain2013provable}, potentially offers better theoretical guarantees, sequential methods remain compelling due to their reduced memory requirements and potential for online learning. 
%With this study, we hope our analysis bridges this gap by characterizing how approximation errors compound through the sequential process, thereby informing algorithmic design choices. % in practical applications where computational resources are constrained.

\textbf{Contributions.}
This work presents a mathematical formulation of linear low-rank regression that emphasizes its decomposition into rank-1 problems. %, providing a clear structure for sequential learning. 
We focus on the scenario where the sub-routine, $\texttt{rank-1}(\mathbf{Y}_k, \mathbf{X}, t)$, incurs numerical errors: even solving  \eqref{eq:rank1_deflation} for a single pair, our estimate is only an approximation to the true pair. 
This view offers a pathway to examine how errors from each rank-1 estimation affect subsequent estimations. 
Since each step in the deflation process depends on the accuracy of the previous steps, any error in estimating a component propagates to the next step, affecting the overall accuracy of the model. 
The following contributions are made: \vspace{-0.15cm}
\begin{itemize} [leftmargin=*]
    \item \textit{Hierarchical Learning Analysis.} We provide a theoretical analysis of hierarchical learning, illustrating how each rank-1 component builds upon the previous components. \vspace{-0.1cm}
    \item \textit{Error Propagation Study.} We provide examination of how errors propagate through the deflation process in linear low-rank regression, highlighting implications in the stability and accuracy.  \vspace{-0.1cm}
    \item \textit{Generalization Ability.} We analyze in the setting of noiseless and noisy labels how sequentially discovering rank-1 components can learn a model that enjoys provable generalization guarantee.\vspace{-0.1cm}
    \item \textit{Experimental Validation.} We validate our theory on both linear low-rank matrix regression problems and simple PEFT settings.\vspace{-0.1cm}
\end{itemize}

\vspace{-0.2cm}\section{Background}
\label{problem_setup}
\vspace{-0.2cm}
We use $\|\bm{a}\|_2$ to denote the $\ell_2$-norm of vector $\mathbf{a}$; $\|\mathbf{A}\|_2$ denotes the spectral norm, $\|\mathbf{A}\|_F$ the Frobenius norm of matrix $\mathbf{A}$; \( \text{sv}_\text{L}(\mathbf{A}) \) and \( \text{sv}_\text{R}(\mathbf{A}) \) denote the normalized top left and right singular vectors of \( \mathbf{A} \).

\textbf{Problem setup.}
Let $\mathbf{X} \in \mathbb{R}^{d \times n}$ be the input matrix of $n$ data points, each with $d$ features. 
For simplicity, we will assume that $\mathbf{X}$ is sampled entrywise from the normal distribution with zero mean and unit variance, followed by a row-wise normalizing process, unless otherwise stated. 
Let $\mathbf{Y} \in \mathbb{R}^{m \times n}$ be the output matrix based on the noiseless generative model, as in:
%\begin{align*}
    $\mathbf{Y} = \mathbf{W}^\star \mathbf{X}$,
%\end{align*}
that simulates the process of ``inserting'' data samples (columns of $\mathbf{X}$) through a low-rank linear channel $\mathbf{W}^\star \in \mathbb{R}^{m \times d}$ to obtain the corresponding column in $\mathbf{Y}$.
The goal is, then, to estimate the best low rank parameter $\mathbf{W}$ given data $(\mathbf{Y}, \mathbf{X})$ as a low-rank linear regression problem:
\begin{equation}\label{eq:3}
\begin{aligned}
& \underset{\mathbf{W} \in \mathbb{R}^{m \times d}}{\text{min}} 
& & f(\mathbf{W}) :=\tfrac{1}{2} \|\mathbf{Y} - \mathbf{W} \mathbf{X}\|_F^2
& \text{s.t.}
& & \text{rank}(\mathbf{W}) \leq r.
\end{aligned} 
\end{equation}
\textbf{Solutions.} This problem has a long history with various approaches, including convex \cite{recht2010guaranteed, lee2009guaranteed, liu2009interior}, non-convex projected-gradient descent \cite{jain2010guaranteed,lee2010admira,kyrillidis2014matrix, kyrillidis2011recipes, khanna2017iht, xu2018accelerated}, as well as matrix factorization ones \cite{burer2003nonlinear, jain2013provable, chen2015fast, zhao2015nonconvex, zheng2015convergent, tu2016low, kyrillidis2018provable, park2016non, sun2016guaranteed, bhojanapalli2016dropping, bhojanapalli2016global, park2016finding, ge2017no, hsieh2017non, kyrillidis2017provable, kim2023fast}.
In the latter, the problem turns into:
\begin{equation} \label{eq:4}
\begin{aligned}
\underset{\mathbf{A} \in \mathbb{R}^{r \times d},~\mathbf{B} \in \mathbb{R}^{m \times r}}{\text{min}} f(\mathbf{A}, \mathbf{B}) := \tfrac{1}{2} \|\mathbf{Y} -  \mathbf{B} \mathbf{A} \mathbf{X} \|_F^2, 
\end{aligned}
\end{equation}
which is related to modernized task adaptation in neural network training, like LoRA \cite{hu2021lora, ostapenko2024towards}.
Key difference in our analysis is that we study the sequential nature of learning; in contrast, in the above scenarios, one often utilizes factorized gradient descent, a low-rank solver that \textit{updates all $r$ rank-1 components simultaneously}, as follows: 
\begin{align*}
\mathbf{A}_{t+1} &= \mathbf{A}_t - \eta_{\mathbf{A}} \nabla_\mathbf{A} f(\mathbf{A}_t, \mathbf{B}_t), \quad 
\mathbf{B}_{t+1} = \mathbf{B}_t - \eta_{\mathbf{B}} \nabla_{\mathbf{B}} f(\mathbf{A}_t, \mathbf{B}_t), \quad \text{with $\eta_{\mathbf{A}}$, $\eta_{\mathbf{B}}$ learning rates.}
\end{align*}
We acknowledge solving \eqref{eq:3}-\eqref{eq:4} directly with these methods when $r$ is known is more efficient and could be preferable in terms of accuracy, yet does not fall into the sequential scenario we focus on.

\textbf{Learning rank-1 subspaces sequentially.}
Our aim is to study routines like the ones described in \eqref{eq:rank_1_equation} and \eqref{eq:rank1_deflation}.
I.e., we are interested in the \emph{sequential, rank-1-updated} linear regression setting, and our focus will be on the theoretical understanding of how errors in calculations in \eqref{eq:rank1_deflation} affect the overall performance. 
To do so, we will need to understand the behavior of both the \textit{exact sequential low-rank recovery}, as well as the \textit{inexact sequential low-rank recovery}.
We describe some simple algorithms to motivate our work.

\begin{minipage}{.49\linewidth}
    \begin{algorithm}[H]\small
        \caption{Exact Sequential Low-Rank}
        \label{alg:exact-main-alg}
        \begin{algorithmic}[1]
            \Require Input data $\mathbf{X} \in \mathbb{R}^{d \times n}$, output data $\mathbf{Y} \in \mathbb{R}^{m \times n}$, target rank $r$
            \Ensure Rank-1 components $\left\{ (\mathbf{a}^\star_k, \mathbf{b}^\star_k) \right\}_{k=1}^{r}$
            \State $\mathbf{Y}^\star_1 \gets \mathbf{Y}$
            \For{$k = 1$ \textbf{to} $r$}
                \State $
                \begin{aligned}
                (\mathbf{a}_k^\star, \mathbf{b}_k^\star) \mathrel{\gets} 
                \operatorname*{arg\,min}_{\mathbf{a} \in \mathbb{R}^d,\, \mathbf{b} \in \mathbb{R}^m} 
                \frac{1}{2} \left\| \mathbf{Y}^\star_k \!-\! \mathbf{b} \mathbf{a}^\top \mathbf{X} \right\|_F^2
                \end{aligned}
                $

                \State $\mathbf{Y}^\star_{k+1} \gets \mathbf{Y}^\star_k - \mathbf{b}_k^\star \mathbf{a}_k^{\star \top} \mathbf{X}$
            \EndFor
            \State \Return $\left\{ (\mathbf{a}_k^\star, \mathbf{b}_k^\star) \right\}_{k=1}^{r}$
        \end{algorithmic}
    \end{algorithm}
\end{minipage}
\hspace{0.1cm}
\begin{minipage}{.49\linewidth}
    \begin{algorithm}[H]\small
        \caption{Inexact Sequential Low-Rank}
        \label{alg:inexact-main-alg}
        \begin{algorithmic}[1]
            \Require Input data $\mathbf{X} \in \mathbb{R}^{d \times n}$, output data $\mathbf{Y} \in \mathbb{R}^{m \times n}$, target rank $r$, sub-routine steps $T$
            \Ensure Approx. rank-1 components $\left\{ (\mathbf{a}_k, \mathbf{b}_k) \right\}_{k=1}^{r}$
            \State $\mathbf{Y}_1 \gets \mathbf{Y}$
            \For{$k = 1$ \textbf{to} $r$}
                \State $(\mathbf{a}_k, \mathbf{b}_k) \gets \texttt{rank-1}(\mathbf{Y}_k, \mathbf{X}, t)$
                \State $\mathbf{Y}_{k+1} \gets \mathbf{Y}_k - \mathbf{b}_k \mathbf{a}_k^\top \mathbf{X}$
            \EndFor
            \State \Return $\left\{ (\mathbf{a}_k, \mathbf{b}_k) \right\}_{k=1}^{r}$
        \end{algorithmic}
        \vspace*{1.4em}
    \end{algorithm}
\end{minipage}

%\begin{algorithm}[!htp]
%\caption{Exact Sequential Low-Rank Subspaces}{\label{alg:exact-main-alg}}
%\begin{algorithmic}[1]
%\Require Input and output data $\mathbf{X} \in \mathbb{R}^{d \times n}$, $\mathbf{Y} \in \mathbb{R}^{m \times n}$.
%\Ensure Pairs $\left\{ (\mathbf{a}^\star_k, \mathbf{b}^\star_k) \right\}_{k=1}^{r}$.
%\State $\mathbf{Y}^\star_1 := \mathbf{Y}$
%\For{$k = 1, \ldots, r$}
%    \State $(\mathbf{a}_k^\star, \mathbf{b}_k^\star) = \argmin_{\mathbf{a} \in \mathbb{R}^d, \mathbf{b} \in \mathbb{R}^m} \frac{1}{2} \left\| \mathbf{Y}^\star_k - \mathbf{b} \mathbf{a}^\top \mathbf{X} \right\|_F^2$
%    \State $\mathbf{Y}^\star_{k+1} := \mathbf{Y}^\star_k - \mathbf{b}_k^\star \mathbf{a}_k^{\star \top} \mathbf{X}$
%\EndFor
%\State \Return $\left\{ (\mathbf{a}_k^\star, \mathbf{b}_k^\star) \right\}_{k=1}^{r}$
%\end{algorithmic}

%\end{algorithm}

Algorithm \ref{alg:exact-main-alg} aims to find exact low-rank subspaces by iteratively computing pairs of vectors $(\mathbf{a}_k^\star, \mathbf{b}_k^\star)$. 
It starts with the original output data $\mathbf{Y}$ as $\mathbf{Y}_1^\star$ (Line 1). 
In each iteration \( k \), an optimization problem is solved to find the \textbf{best} pair \( \left(\mathbf{a}_k^\star, \mathbf{b}_k^\star\right) \) that minimizes the Frobenius norm of the difference between the current matrix \( \mathbf{Y}_k^\star \) and the rank-1 estimate $\mathbf{b} \mathbf{a}^\top \mathbf{X}$ (see Line 3)\footnote{We note that, although $\sum_{k=1}^r\mathbf{b}_k^{\star}\mathbf{a}_k^{\star} = \mathbf{W}^{\star}$, $\left(\mathbf{b}_k^{\star},\mathbf{a}_k^{\star}\right)$ not necessarily aligns with the $k$th left- and right- singular vectors of $\mathbf{W}^{\star}$.}. By applying the singular value decomposition (SVD) to \( \mathbf{Y} \), we decompose it as: $\mathbf{Y} = \sum_{k=1}^{p} \sigma_k^\star \mathbf{u}_k^\star \mathbf{v}_k^{\star\top}$,
where \(\sigma_k^\star \) are the singular values, and \( \mathbf{u}_k^\star \) and \( \mathbf{v}_k^\star \) are the left and right singular vectors, respectively. Notice that we denote $p = \text{rank}(\mathbf{Y})$, but when executing Algorithm~\ref{alg:exact-main-alg} and Algorithm~\ref{alg:inexact-main-alg} we may choose a target rank $r\neq p$. 
%\textcolor{magenta}{Create a small lemma here for the following result.}
\begin{lemma}
    \label{lem:sequential_decomposition}
    According to the Eckart-Young-Mirsky theorem, under our defined settings and deflation method, for each \( k \), we have that $\mathbf{Y}^\star_{k}=\sum_{k' = k}^p \sigma_{k'}^\star \mathbf{u}_{k'}^\star \mathbf{v}_{k'}^{\star\top}$ and $\mathbf{b}_k^\star \mathbf{a}_k^{\star \top} \mathbf{X} = \sigma_k^\star \mathbf{u}_k^\star \mathbf{v}_k^{\star\top}$.
\end{lemma}
The proof of Lemma~\ref{lem:sequential_decomposition} is provided in Appendix~\ref{app:lemma1}. Namely, when $\mathbf{X}$ has full rank with $n\geq d$, $\mathbf{b}_k^\star$ and $\mathbf{a}_k^\star$ can be uniquely identified up to scalar multiplication. 

After determining the pair \( (\mathbf{a}_k^\star, \mathbf{b}_k^\star) \), the matrix \( \mathbf{Y}_{k+1}^\star \) is updated by subtracting the rank-1 component \( \mathbf{b}_k^\star \mathbf{a}_k^{\star \top} \mathbf{X} \) from \( \mathbf{Y}_k^\star \) (see Line 4). This iterative process continues for \( r \) iterations, generating \( r \) pairs of vectors, which collectively represent the exact low-rank subspaces. Algorithm \ref{alg:inexact-main-alg} differs from Algorithm \ref{alg:exact-main-alg} in Lines 3 and 4. In Algorithm \ref{alg:inexact-main-alg}, Line 3 executes sub-routine for $t$ iterations, denoted by $\texttt{rank-1}(\mathbf{Y}_k, \mathbf{X}, t)$, to approximate the solution of \eqref{eq:rank_1_equation} and return estimates $(\mathbf{a}_k, \mathbf{b}_k)$. 
The $t$ parameter represents the number of iterations for this approximate computation. An example of the $\texttt{rank-1}$ subroutine can be the gradient descent algorithm, which executes:
\begin{small}
\begin{equation}
    \label{eq:factor_rank_1}
    \begin{aligned}
        \mathbf{a}^{(t+1)} & = \mathbf{a}^{(t)} - \eta_{\mathbf{a}}\mathbf{X}\left(\mathbf{b}^{(t)}\mathbf{a}^{(t)\top}\mathbf{X} - \mathbf{Y}\right)^\top\mathbf{b}^{(t)}, ~~ 
        \mathbf{b}^{(t+1)} = \mathbf{b}^{(t)} - \eta_{\mathbf{b}}\left(\mathbf{b}^{(t)}\mathbf{a}^{(t)\top}\mathbf{X} - \mathbf{Y}\right)\mathbf{X}^\top\mathbf{a}^{(t)}.
    \end{aligned}
\end{equation}
\end{small}
Iterative algorithms such as (\ref{eq:factor_rank_1}) often produce numerical errors, leading to $\mathbf{b}_k\mathbf{a}_k^\top\neq \mathbf{b}_k^\star\mathbf{a}_k^{\star\top}$. Subsequently, this affects the quality of the remaining information in $\mathbf{Y}_{k+1}$ in Line 4, since the ``deflation'' step $\mathbf{Y}_k - \mathbf{b}_k \mathbf{a}_k^\top \mathbf{X}$ is based on an approximate deflated matrix $\mathbf{Y}_k$, coming from the $k-1$ iteration that is not equal to $\mathbf{Y}_k^\star$ in Algorithm \ref{alg:exact-main-alg}, as well as depends on approximate current estimates $\mathbf{b}_k \mathbf{a}_k^\top \mathbf{X}$, and not $\mathbf{b}_k^\star \mathbf{a}_k^{\star \top} \mathbf{X}$ as in the exact case. To study the influence of the numerical errors produced by (\ref{eq:factor_rank_1}), we introduce the following definition:

\begin{definition}[Numerical Error]
    \label{def:num_err}
    Let \((\overline{\mathbf{a}}_k, \overline{\mathbf{b}}_k)\) represents the exact rank-1 solution that approximates the processed label matrix $\mathbf{Y}_k$ using data $\mathbf{X}$:\vspace{-0.2cm}
    \begin{equation}
        \label{eq:rank_k_equation}
        \overline{\mathbf{a}}_k, \overline{\mathbf{b}}_k = \operatorname*{arg\,min}_{\mathbf{a} \in \mathbb{R}^d, \mathbf{b} \in \mathbb{R}^m} \frac{1}{2} \left\| \mathbf{Y}_k - \mathbf{b} \mathbf{a}^{\top} \mathbf{X} \right\|_F^2.
    \end{equation}
    and recall that $\mathbf{a}_k,\mathbf{b}_k$ are outputs of $\texttt{rank-1}\left(\mathbf{Y}_k,\mathbf{X},T\right)$. We define the numerical errors incurred at iteration $k$ from the $\texttt{rank-1}$ sub-routine as
    \begin{equation}
        \label{eq:def-delta-k}
        \bm{\delta}_k := \mathbf{b}_k \mathbf{a}^\top_k - \overline{\mathbf{b}}_k \overline{\mathbf{a}}_k^{\top} ;\quad \|\bm{\delta}_k\|_F\geq 0.
    \end{equation}
\end{definition}
\vspace{-0.2cm}Notice that the definition of $\bm{\delta}_k$ is based on \(\mathbf{Y}_k\), not $\mathbf{Y}_k^\star$; recall that \(\mathbf{Y}_k\) is constructed recursively using \(\mathbf{b}_k \mathbf{a}_k^\top\). 
When \(\mathbf{b}_k \mathbf{a}_k^\top\) is solved inexactly, we cannot guarantee that \(\mathbf{Y}_k = \mathbf{Y}^\star_k\). 
Consequently, it is almost always the case that \((\overline{\mathbf{a}}_k, \overline{\mathbf{b}}_k) \neq (\mathbf{a}_k^\star, \mathbf{b}_k^\star)\), implying \(\|\mathbf{b}_k^\star \mathbf{a}_k^{\star\top} - \overline{\mathbf{b}}_k \overline{\mathbf{a}}_k^{\top}\|_F > 0\).
However, since the \(\texttt{rank-1}\) subroutine only has access to \(\mathbf{Y}_k\), its output \((\mathbf{a}_k, \mathbf{b}_k)\) converges to \((\overline{\mathbf{a}}_k, \overline{\mathbf{b}}_k)\) instead of \((\mathbf{a}_k^\star, \mathbf{b}_k^\star)\) as the number of iterations in the \(\texttt{rank-1}\) subroutine increases.

\textbf{Related works.}
The task of sequential low-rank subspace identification has been studied in the context of Principal Component Analysis (PCA) \cite{jolliffe1995rotation, golub2013matrix}.
There are hierarchical game-theoretic approaches with multiple rank-1 players that provide a framework for understanding the decomposition of data into a hierarchy of skills or components \cite{gemp2021eigengame, gemp2021eigengame2}. 
There, each rank-1 player can be seen as an agent learning a distinct, singular skill or feature from the dataset. %, which corresponds to the principal components of the data \cite{jolliffe1995rotation}. 
%This model aligns well with the notion of learning in stages, where each stage (or player) captures a level of abstraction or complexity from the data \textcolor{red}{Citations}. 
%For instance, in a facial recognition task, the first few rank-1 players might learn basic features such as edge orientations and simple textures, which are foundational, while subsequent players may learn more complex features like specific facial attributes or expressions.
%Hierarchical game-theoretic PCA models effectively use a sequential and layered learning approach where each rank-1 player not only contributes to the dimensionality reduction but also adds to the cumulative knowledge base, facilitating a deeper understanding of the data structure. This sequential learning can be likened to building a multi-layered neural network where each layer captures different levels of abstraction. 
The game-theoretic aspect ensures that each player (or component) optimizes a particular ``deflated'' objective \cite{hotelling1933analysis, mackey2008deflation, zhang2006schur, sriperumbudur2007sparse, saad1988projection, danisman2014comparison}. 
%This approach underscores a shift viewing PCA as a hierarchical learning process that progressively uncovers and learns different ``skills'' inherent in the data.

\textit{Incremental learning of Eigenspaces.} 
%Even without an explicit formulation, the behavior of sequentially learning subspace directions have been observed and analyzed in the context of low-rank factorization and low-rank recovery. In particular, 
\cite{arora2019implicitregularizationdeepmatrix} identified that in deep matrix factorization, the low-rank components are discovered in a sequential manner. \cite{jin2023understandingincrementallearninggradient} extends this observation to the case of symmetric matrix sensing, backed with a detailed theoretical analysis. This analysis is further generalized to the asymmetric case by \cite{soltanolkotabi2023implicit}. Noticeably, this implicit sequential recovery of the low-rank components can be leveraged to efficiently compress the learned model \cite{kwon2024efficientcompressionoverparameterizeddeep}. Nevertheless, it should be noted that this sequential nature appears only when the model is deep enough or under a proper initialization. On the other hand, our work considers the simple model of low-rank linear regression by explicitly enforcing the sequential learning. A more similar work to ours is \cite{wang2023adaptive}, but their algorithmic design is more specific to the task of matrix completion.

\textit{Low-Rank Adapter (LoRA).} The sequential learning of low-rank subspaces has connections with PEFT methods like LoRA \cite{hu2021lora}. 
A stronger connection is present when LoRA is applied to the scheme of continual learning, when the low-rank adapters are learned in a sequence when new tasks come in \cite{wistuba2023continuallearninglowrank}. 
Later works impose additional orthogonality constraints between the subspaces learned by the adapters to prevent catastrophic forgetting \cite{wang2023orthogonalsubspacelearninglanguage}.
While recent theoretical work has shown that LoRA can adapt any model $f$ to accurately represent a target model if the LoRA-rank is sufficiently large \cite{zhang2023adalora}, the dynamics of how errors propagate when using lower ranks remains unexplored.
%This is particularly relevant as practical applications often employ ranks much smaller than the theoretical thresholds. 
Recent works consider a collection of LoRAs via merging, such as \cite{zhao2024merging, dimitriadis2024pareto, wu2024mixture, ostapenkotowards, huh2024training, xia2024chain}.

%The work of \cite{zhao2024merging} further explores the modularity of LoRA by introducing the LoRA-LEGO framework, which allows for the flexible merging of multiple LoRAs. This approach conceptualizes each rank of LoRA as a Minimal Semantic Unit (MSU), a fundamental building block that can be clustered and combined like LEGO blocks. By clustering similar MSUs from different LoRAs, the framework mitigates parameter interference and aligns task-specific knowledge. This method has demonstrated superior performance across multi-task and mixed-task settings, showcasing its ability to merge heterogeneous LoRAs and effectively adjust the rank of the resulting adapter. Such advancements underline the potential of LoRA merging for efficient multi-task learning and model reuse.

\vspace{-0.2cm}
\section{Error propagation during training}
\label{sec:train_err}
\vspace{-0.2cm}
Recall that Lemma~\ref{lem:sequential_decomposition} guarantees that the rank-1 components given by the exact Algorithm~\ref{alg:exact-main-alg} recovers the top-$r$ singular vector/values of $\mathbf{Y}^\star$. In this section, we study the recovery error under the inexact Algorithm~\ref{alg:inexact-main-alg}. 
To effectively compare the outputs of Algorithm \ref{alg:exact-main-alg} with those of Algorithm \ref{alg:inexact-main-alg}, we express these outputs in terms of the singular values and singular vectors of the deflated matrices $\mathbf{Y}_k^\star$ and $\mathbf{Y}_k$. 
We apply similar reasoning as in Lemma \ref{lem:sequential_decomposition} for the term \(\overline{\mathbf{b}}_k \overline{\mathbf{a}}_k^\top \mathbf{X}\) based on \eqref{eq:rank_k_equation}.
To do so, we define $\sigma_{i_k}, \mathbf{u}_{i_k}$, and $\mathbf{v}_{i_k}$ as the $i$-th top singular value and vectors pairs of the matrix \(\mathbf{Y}_k\). 
Note that $\sigma_{i_k} \neq \sigma_i^\star$ and $(\mathbf{u}_{i_k}, \mathbf{v}_{i_k}) \neq (\mathbf{u}_i^\star, \mathbf{v}_i^\star)$, for all $i$.
Then, for each \(k\), the SVD on \(\mathbf{Y}_k\) gives us:
$\mathbf{Y}_k = \sum_{i=1}^{p-k+1} \sigma_{i_k} \mathbf{u}_{i_k} \mathbf{v}_{i_k}^\top$.
Since \(\overline{\mathbf{b}}_k \overline{\mathbf{a}}_k^\top \mathbf{X}\) is also rank-$1$, Eckart-Young-Mirsky theorem implies that it is the optimal rank-1 approximation of \(\mathbf{Y}_k\) based on \eqref{eq:rank_k_equation}. Thus
$\overline{\mathbf{b}}_k \overline{\mathbf{a}}_k^\top \mathbf{X} = \sigma_{1_k} \mathbf{u}_{1_k} \mathbf{v}_{1_k}^\top$,
where \(\sigma_{1_k}\), \(\mathbf{u}_{1_k}\), and \(\mathbf{v}_{1_k}\) correspond to the top singular value and singular vectors of \(\mathbf{Y}_k\).

Recall that \( \mathbf{u}^\star_k \) and \( \mathbf{v}^\star_k \) are the top left and right singular vectors, respectively, of \( \mathbf{Y}_k^\star \). Since singular vectors are unique only up to a sign, both \( \text{nsv1}_\text{L}(\mathbf{Y}_k) \) and \( -\text{nsv1}_\text{L}(\mathbf{Y}_k) \) are valid left singular vectors, and similarly, both \( \text{nsv1}_\text{R}(\mathbf{Y}_k) \) and \( -\text{nsv1}_\text{R}(\mathbf{Y}_k) \) are valid right singular vectors of \( \mathbf{Y}_k \). So we will choose $\mathbf{u}_{1_k}$ and $\mathbf{v}_{1_k}$ to be the ones such that $0\leq\mathbf{v}_k^{\star\top}\mathbf{v}_{1_k}$ and $0\leq\mathbf{u}_k^{\star\top}\mathbf{u}_{1_k}$:
\begin{equation*}
    \mathbf{u}_{1_k}:=  \text{sv}_\text{L}(\mathbf{Y}_k)\cdot \operatorname*{arg\,min}_{s \in \{ \pm 1 \}} \| s \cdot \texttt{sv}_\text{L}(\mathbf{Y}_k) - \mathbf{u}_k^\star \|_2;\;\mathbf{v}_{1_k}:= \text{sv}_\text{R}(\mathbf{Y}_k)\cdot \operatorname*{arg\,min}_{s \in \{ \pm 1 \}} \| \cdot \texttt{sv}_\text{R}(\mathbf{Y}_k) - \mathbf{v}_k^\star \|_2.
\end{equation*}
We provide a characterization of the error propagation in the deflation methods in Algorithm~\ref{alg:inexact-main-alg} that is agnostic to the detail of the sub-routine \texttt{rank-1}, i.e., when one only has knowledge about $\|\bm{\delta}_k\|_F$.
The proof can be found in Appendix \ref{app:1}.

\begin{theorem}\label{thm:main_theorem_1}
    Let $\left\{\left(\mathbf{a}_k,\mathbf{b}_k\right)\right\}_{k=1}^r$ be the output of Algorithm~\ref{alg:inexact-main-alg}. Let $\bm{\delta}_k$ be given as in Definition~\ref{def:num_err} with $\left\|\bm{\delta}_k\right\|_F > 0$. Let $\sigma_1^\star,\dots,\sigma_{r^\star}^\star$ denote the singular values of $\mathbf{Y}$. Define the minimum singular value gap as $\mathcal{T}^\star_k := \min\left\{\min_{j>k} |\sigma^\star_k - \sigma^\star_j|, \sigma^\star_k\right\}$. Also, define an error bound $E(k)$ as:
    \[
        E(k):=\sigma_{\max}(\mathbf{X})\sum_{k'=0}^{k-1}\left\|\bm{\delta}_{k'}\right\|_F\prod_{j=k'+1}^{k-1}\left(2 + \frac{6\sigma_j^\star}{\mathcal{T}_k^\star} \right).
    \]
     If $E(k) < \tfrac{1}{2}\min_{j>k}|\sigma^\star_k - \sigma^\star_j|$, then the output of Algorithm \ref{alg:inexact-main-alg} satisfies:
    \begin{equation}
    \begin{aligned}
        \left\|\mathbf{Y}-\sum_{k=1}^r \mathbf{b}_k \mathbf{a}_k^ \top \mathbf{X}\right\|_F
         \leq \left(\sum_{k=r+1}^{p}\sigma_k^\star\right) + \sigma_{\max}(\mathbf{X})\sum_{k=1}^{r}\sum_{k'=0}^{k} \left\|\mathbf{\bm{\delta}}_{k'}\right\|_F \prod_{j = k'+1}^{k}\left(2 + \frac{6\sigma_j^\star}{\mathcal{T}_k^\star} \right)
        \end{aligned}
    \end{equation}
\end{theorem}
Theorem \ref{thm:main_theorem_1} characterizes how errors from approximately solving the rank-1 subroutine propagate through the deflation procedure in sequential low-rank approximations. 
The theorem asserts that, as long as each error \( \bm{\delta}_k \) is sufficiently small, the compounded effect of errors across the sequence remains bounded, thereby preserving the accuracy of the final low-rank approximation.
\begin{remark}
    The error bound in Theorem \ref{thm:main_theorem_1} reflects sensitivity to the eigenspectrum of the underlying data matrix. Notice that the upper bound in Theorem \ref{thm:main_theorem_1} involves a summation of summations over components that depend on the error of the sub-routine $\bm{\delta}_{k'}$. In particular, both the number of summands and the multiplicative factor of $\prod_{j = k'+1}^{k}\left(2 + \sfrac{6\sigma_j^\star}{\mathcal{T}_k^\star} \right)$ in each summand grows as $k$ increases. Notice that a slower decay of singular values —corresponding to a smaller eigengap—error propagation is amplified, making approximation steps more susceptible to the accumulation of individual errors $\bm{\delta}_{k}$, and vice versa. This dependency on the singular spectrum necessitates more precision in each step for data matrices with dense singular values to avoid error escalation.
\end{remark}

\vspace{-0.2cm}
\section{Generalization of sequential rank-$1$ update}
\label{sec:gen_err}
\vspace{-0.2cm}
Thus far, we have been focusing on constructing components $\left\{\left(\mathbf{a}_k,\mathbf{b}_k\right)\right\}_{k=1}^{r}$ such that $\sum_{k=1}^{r}\mathbf{b}_k\mathbf{a}_k^\top\mathbf{X}$ estimates $\mathbf{Y}$. Given that $\mathbf{X}$ and $\mathbf{Y}$ are considered as training data, the previous section characterized the training error of Algorithm~\ref{alg:inexact-main-alg}. 
Here, we analyze the generalization ability of Algorithm~\ref{alg:inexact-main-alg}, assuming that the data is generated based on some optimal parameter $\mathbf{W}^\star$ with rank$\left(\mathbf{W}^\star\right) = r^\star$:
\begin{equation}
    \label{eq:general_setup}
    \mathbf{Y} = \mathbf{Y}^\star + \bm{\mathcal{E}}; \quad \mathbf{Y}^\star = \mathbf{W}^\star\mathbf{X},
\end{equation}
where $\bm{\mathcal{E}} \in\mathbb{R}^{m\times n}$ denotes the label noise generated from a certain distribution, and $\mathbf{Y}^\star$ denotes the noiseless label. In noiseless case where $\bm{\mathcal{E}} = \bm{0}$, we have that $p = \text{rank}(\mathbf{Y}) = r^\star$ and, Algorithm~\ref{alg:exact-main-alg} can recover $\left\{\left(\mathbf{a}_k^\star,\mathbf{b}_k^\star\right)\right\}_{k=1}^{r}$ such that $\mathbf{W}^\star = \sum_{k=1}^{r}\mathbf{b}_k^\star\mathbf{a}_k^{\star\top}$ when $r = r^\star$. However, it may not be the case that $\mathbf{b}_k$ and $\mathbf{a}_k$ aligns with the $k$th left and right singular vector of $\mathbf{W}^\star$ under the influence of $\mathbf{X}$. In other words, each pair $\left(\mathbf{a}_k^\star,\mathbf{b}_k^\star\right)$ contains component of $\mathbf{W}^\star$ that \textit{extracts a certain information from the input data $\mathbf{X}$}. When $\bm{\mathcal{E}} \neq \bm{0}$, it is possible that $p = \text{rank}(\mathbf{Y}) > r^\star$. 

\textbf{Generalization under noiseless labels.}
As a warm up, we consider the case where the noise $\bm{\mathcal{E}} = \bm{0}$.
Intuitively, when $\mathbf{X}$ is full rank, a zero training loss would imply a perfect recovery of the optimal parameter $\mathbf{W}^\star$. 
We state a more general result below covering the case of non-zero training loss with component-wise generalization error.
\begin{theorem}
    \label{thm:gen_noiseless}
    Let $\left\{\left(\mathbf{a}_k,\mathbf{b}_k\right)\right\}_{k=1}^r$ be the output of Algorithm~\ref{alg:inexact-main-alg}. Let $\bm{\delta}_k$ be given as in Definition~\ref{def:num_err} with $\left\|\bm{\delta}_k\right\|_F > 0$. Let $\sigma_1^\star,\dots,\sigma_{r^\star}^\star$ denote the singular values of $\mathbf{Y}$. Define the minimum singular value gap as $\mathcal{T}^\star_k := \min\left\{\min_{j>k} |\sigma^\star_k - \sigma^\star_j|, \sigma^\star_k\right\}$. Also, define an error bound $E(k)$ as:
    \[
        E(k):=\sigma_{\max}(\mathbf{X})\sum_{k'=0}^{k-1}\left\|\bm{\delta}_{k'}\right\|_F\prod_{j=k'+1}^{k-1}\left(2 + \frac{6\sigma_j^\star}{\mathcal{T}_k^\star} \right).
    \]
     If $E(k) < \tfrac{1}{2}\min_{j>k}|\sigma^\star_k - \sigma^\star_j|$, and $\sigma_{\min}\left(\mathbf{X}\right) \geq 0$, then the output of Algorithm \ref{alg:inexact-main-alg} satisfies:
    \begin{equation}
        \label{eq:noiseless_comp_err}
        \left\|\mathbf{b}_k^\star\mathbf{a}_k^{\star\top}- \mathbf{b}_k\mathbf{a}_k^ \top\right\|_F \leq\kappa(\mathbf{X})\sum_{k'=0}^{k}\left\|\bm{\delta}_{k'}\right\|_F\prod_{j=k'+1}^{k}\left(2 + \frac{6\sigma_j^\star}{\mathcal{T}_k^\star} \right);\;\;\forall k\in[r].
    \end{equation}
    Moreover, the aggregation of the components $\left(\mathbf{a}_k,\mathbf{b}_k\right)$'s approximates $\mathbf{W}^\star$ as
    \begin{equation}
        \label{eq:noiseless_total_err}
        \left\| \mathbf{W}^\star - \sum_{k=1}^r \mathbf{b}_k \mathbf{a}_k^\top\right\|_F \leq \sum_{k=r+1}^{r^\star} \frac{\sigma_k^\star}{\sigma_{\min}(\mathbf{X})} + \kappa(\mathbf{X}) \sum_{k=1}^r \sum_{k'=1}^{k} \|\bm{\delta}_{k'}\|_F \prod_{j=k'+1}^{k} \left( 2 + \dfrac{6\sigma_j^\star}{\mathcal{T}_j^\star} \right).
    \end{equation}
    Here, $\kappa(\mathbf{X}) = \frac{\sigma_{\max}(\mathbf{X})}{\sigma_{\min}\left(\mathbf{X}\right)}$ denotes the condition number of $\mathbf{X}$.
\end{theorem}
The proof of Theorem~\ref{thm:gen_noiseless} is provided in Appendix~\ref{app:2}. In particular, Theorem~\ref{thm:gen_noiseless} states two results. First, (\ref{eq:noiseless_comp_err}) measures how the errors of individual components are influenced even by the numerical errors that appear when solving previous components. This bound illustrates key factors contributing to the error at each iteration. As we discussed previously, $\left(\mathbf{a}_k^\star,\mathbf{b}_k^\star\right)$ can be considered as the components of $\mathbf{W}^\star$ extracted based on the importance defined by the input data $\mathbf{X}$. From this perspective, (\ref{eq:noiseless_comp_err}) shows how well these data-dependent components are approximated by Algorithm~\ref{alg:inexact-main-alg}. Moreover, (\ref{eq:noiseless_total_err}) measures how well the inexact method approximates $\mathbf{W}^\star$, including errors due to inexact computations and limitations of representing $\mathbf{W}^\star$ with rank $r$. This bound shed light on how the components $\left(\mathbf{a}_k,\mathbf{b}_k\right)$'s collaboratively contribute to the overall generalization ability.

\textbf{Generalization under noisy labels.}
In the previous section, we studied the generalization ability of Algorithm~\ref{alg:inexact-main-alg} under a noiseless scenario with $\mathbf{Y} = \mathbf{W}^\star\mathbf{X}$, where the algorithmic choice of choosing $r = \text{rank}(\mathbf{Y})$ can be shown to be optimal. However, this argument may not hold when the labels are generated with a non-zero additive noise. In this section, we consider the noise matrix $\bm{\mathcal{E}}$ to consist of I.I.D. entries $\bm{\mathcal{E}}_{ij}\sim\mathcal{N}(0, \varepsilon^2)$, where $\varepsilon$ controls the magnitude of the noise. In this case, with a high probability, we have $p= \text{rank}(\mathbf{Y}) = m$. Let $\left\{\left(\mathbf{a}_k,\mathbf{b}_k\right)\right\}_{k=1}^r$ be the recovery result according to Algorithm~\ref{alg:inexact-main-alg}. Then we are interested in an upper bound on $\left\|\mathbf{W}^\star - \sum_{k=1}^r\mathbf{b}_k\mathbf{a}_k\right\|_F$. In particular, we have the following guarantee on the generalization error.
\begin{theorem}
    \label{thm:noisy_gen}
    Consider the scenario of finding the top-$r$ rank-1 subspaces that minimize the loss in \eqref{eq:rank_1_equation}. Let $\mathcal{T}_k^\star := \min\left\{\min_{j>k} |\sigma^\star_k - \sigma_{j_k}|, \sigma^\star_k\right\}$, and $\mathcal{T}_{\min}^\star:= \min_{k \in [1, r]}\mathcal{T}_k^\star$, if the noise scale satisfies $\varepsilon \leq O\left(\frac{\mathcal{T}_{\min}^\star}{\sqrt{n} + \sqrt{\log\sfrac{1}{\gamma}}} \right)$
    then with probability at least $1 - \gamma$, the output of Algorithm \ref{alg:inexact-main-alg} satisfies:
    \begin{equation}
        \label{eq:noise_gen}
        \begin{aligned}
            \left\|\mathbf{W}^\star - \sum_{k=1}^r \mathbf{b}_k \mathbf{a}_k^ \top\right\|_F &\leq \textcolor{teal}{\kappa(\mathbf{X})\left(\sum_{k=r+1}^{r^\star}\sigma_r\left(\mathbf{W}^\star\right) + \sum_{k=1}^{r}\sum_{k'=0}^{k} \left\|\mathbf{\bm{\delta}}_{k'}\right\|_F \prod_{j = k'+1}^{k}\left(2 + \frac{6\sigma_j^\star}{\mathcal{T}_k^\star} \right)\right)}\\
            &\quad\quad\quad + \textcolor{orange}{O\left(\frac{\varepsilon\sqrt{n\log\sfrac{1}{\gamma}}}{\sigma_{\min}(\mathbf{X})}\left(r+\sqrt{\frac
                {\min\{r^\star,r\}}{\mathcal{T}_{\min}^\star}}\right)\right)}
        \end{aligned}
    \end{equation}
\end{theorem}
The proof of Theorem~\ref{thm:noisy_gen} is given in Appendix~\ref{se:proof_noisy_gen}. In particular, Theorem~\ref{thm:noisy_gen} characterizes how Algorithm~\ref{alg:inexact-main-alg} recovers component that can generalize even under label noise. The \textcolor{teal}{first term} in the upper bound of (\ref{eq:noise_gen}) comes from the numerical errors in the inexact solving of each $\texttt{rank-1}$ subroutine. Then \textcolor{orange}{second term} demonstrates the influence of the additive noise $\bm{\mathcal{E}}$ on the generalization ability. To start, a larger noise scale $\epsilon$ implies a worse generalization error. Moreover, it should be noticed that a good choice of $r$ can greatly impact the generalization error as well: choosing $r < r^*$ can result in a larger error in the \textcolor{teal}{first term} due to the incomplete estimation of the components in $\mathbf{W}^\star$. On the other hand, since the \textcolor{orange}{second term} scales with $r$, choosing a larger $r$ can result in a larger error caused by the noise. This is the scenario where the noise is overfitted by increasing the complexity of the model. From this perspective, Theorem~\ref{thm:noisy_gen} characterizes the bias-variance trade-off in the sequential rank-1 recovery algorithm. Lastly, the requirement of the noise scale is to make sure that after adding the noise, the ordering of the rank-1 components is not changed.

\vspace{-0.2cm}
\section{Experimental results}
\vspace{-0.2cm}
\begin{figure}[t!]
    \centering
    \begin{subfigure}[b]{0.48\textwidth}        \includegraphics[width=\linewidth]{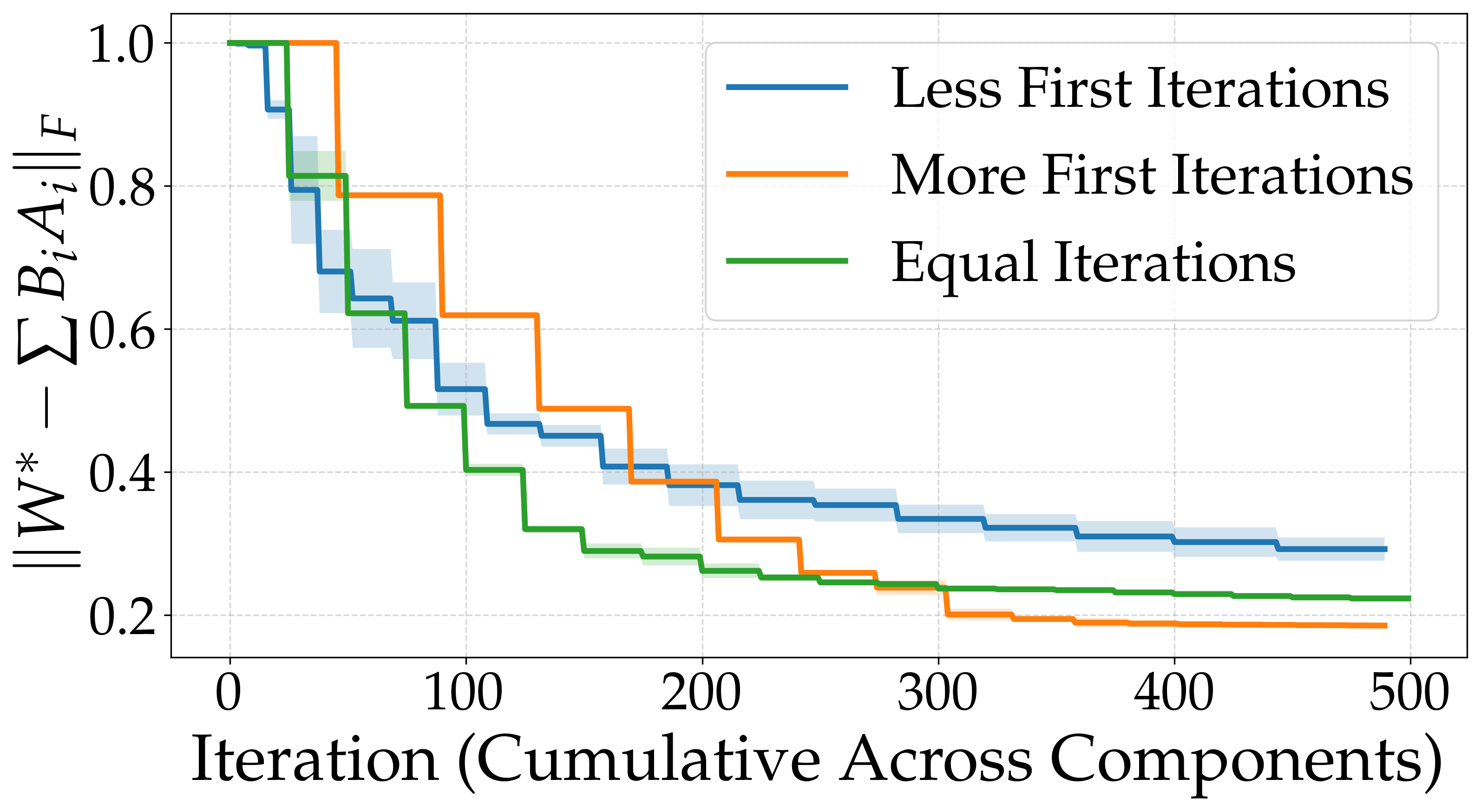}        
    \end{subfigure}
    \hfill
    \begin{subfigure}[b]{0.48\textwidth}        \includegraphics[width=\linewidth]{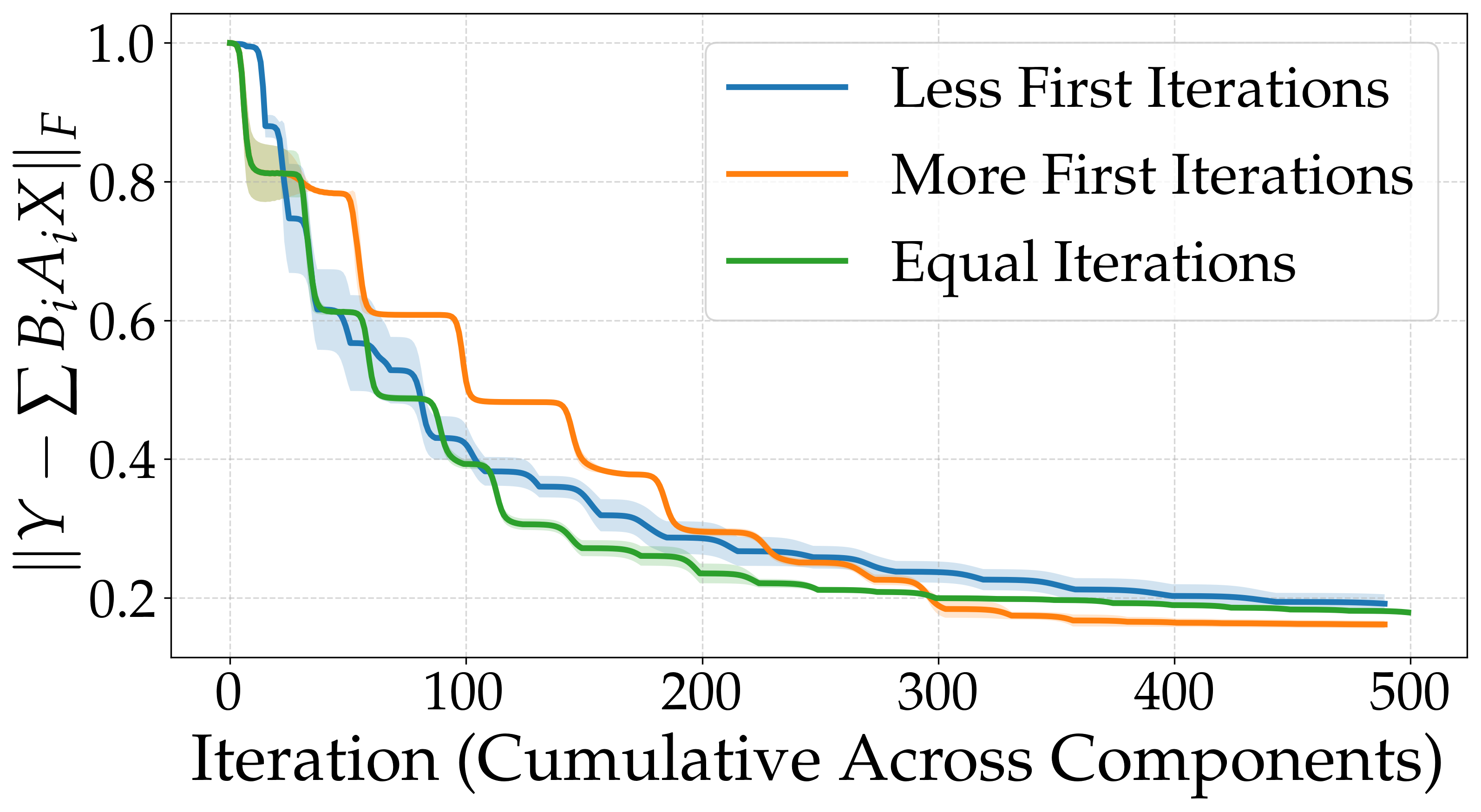}
    \end{subfigure}
    \caption{Impact of iteration allocation strategy under a fixed iteration budget. \textit{Left}: $\mathbf{W}^\star$ reconstruction error. \textit{Right}: Objective's training error.}\vspace{-0.2cm}
    \label{fig:allocation}
\end{figure}
\vspace{-0.1cm}
\subsection{Synthetic validation of theoretical setting.}
\vspace{-0.2cm}
We present experiments that validate our theory on error propagation in sequential rank-1 learning. 
Our experiments aim to demonstrate how the distribution of computational resources across rank-1 components affects the overall approximation quality, particularly focusing on how errors in early components propagate to later stages of the sequential learning process. Following the general set-up in (\ref{eq:general_setup}), we consider three different iteration allocation strategies:
\vspace{-0.2cm}
\begin{enumerate}[leftmargin=*]
\item \textbf{Equal:} Same number of optimization iterations to each rank-1 component. \vspace{-0.1cm}
\item \textbf{More First:} More iterations allocated to the earlier components and fewer to later ones. \vspace{-0.1cm}
\item \textbf{Less First:} Fewer iterations allocated to the earlier components and more to later ones. \vspace{-0.1cm}
\end{enumerate}
Our analysis dictates that errors in early components propagate to later components, suggesting that allocating more iterations to earlier components leads to better overall performance. 
Figure~\ref{fig:allocation} shows the reconstruction and training errors, respectively, for the three allocation strategies under a fixed computational budget. The results confirm our theoretical predictions. Allocating more iterations to earlier components leads to better final reconstruction and training errors compared to allocating fewer iterations initially. The equal iteration strategy performs better than the ``less first iterations'' approach, but worse than the ``more first iterations'' strategy.
This validates our theoretical finding that errors in early components propagate and compound through the sequential process. This cascading effect means that if early components are poorly approximated, their errors get magnified in subsequent components. Details and further results in the synthetic setting are deferred to Appendix~\ref{app:3}.
\vspace{-0.2cm}
\subsection{Experimental analysis using LoRA.}
\vspace{-0.2cm}
We evaluate our sequential rank-1 approach in LoRA adaptation on three standard image classification datasets: MNIST, CIFAR10, and CIFAR100. 
We design the experiments such that each dataset present a different level of challenge to assess how our sequential LoRA adaptation performs under varying initial conditions.
\textit{The purpose here is not to attain top-notch performance in these scenarios neither to claim these as ``real scenarios''; rather, to assess how sequential learning behaves on well- to --intentionally-- badly-pretrained scenarios. This is also expected, given that the baseline model is a feedforward neural network.}

\textbf{Problem setting.}
We employ a simple feedforward network as our base architecture across all experiments. 
For each dataset, we first train the baseline model on a subset of classes (in particular, the first half of available classes). 
We then apply our sequential rank-1 LoRA adaptation approach to handle the remaining classes for 3 sequential rank-1 trainings, i.e., $r=3$.

Our architecture consists of three fully-connected layers that map flattened input images to class logits. 
As usual, for MNIST, inputs are 784-dimensional (28$\times$28 grayscale images), while for CIFAR10 and CIFAR100, inputs are 3072-dimensional (32$\times$32$\times$3 RGB images). 
Hidden layers have 512 units with ReLU activations, and the output layer dimension matches the number of classes in each dataset.
We analyze three distinct scenarios: \vspace{-.1cm}
\begin{enumerate}[leftmargin=*]
    \item \textbf{MNIST (Strong Baseline)}: The baseline network achieves high accuracy ($\sim$98\%) on classes 0-4, providing a strong foundation for adaptation on the remaining 5-9 classes. \vspace{-.1cm}
    \item \textbf{CIFAR10 (Moderate Baseline)}: The baseline network reaches moderate accuracy ($\sim$40\%) on classes 0-4, representing a partially optimized model (reminder that the model is not a CNN-based model but just a FF connected network). \vspace{-.1cm}
    \item \textbf{CIFAR100 (Weak Baseline)}: The baseline network attains lower accuracy ($\sim$20\%) on classes 0-49, exemplifying a relatively poor initial representation, where LoRA models adapt over the remaining 50-99 classes. \vspace{-.1cm}
\end{enumerate}

\begin{figure}[t!]
    \centering
    \hspace{-0.5cm}
    \begin{subfigure}[b]{0.33\textwidth}        \includegraphics[width=\linewidth]{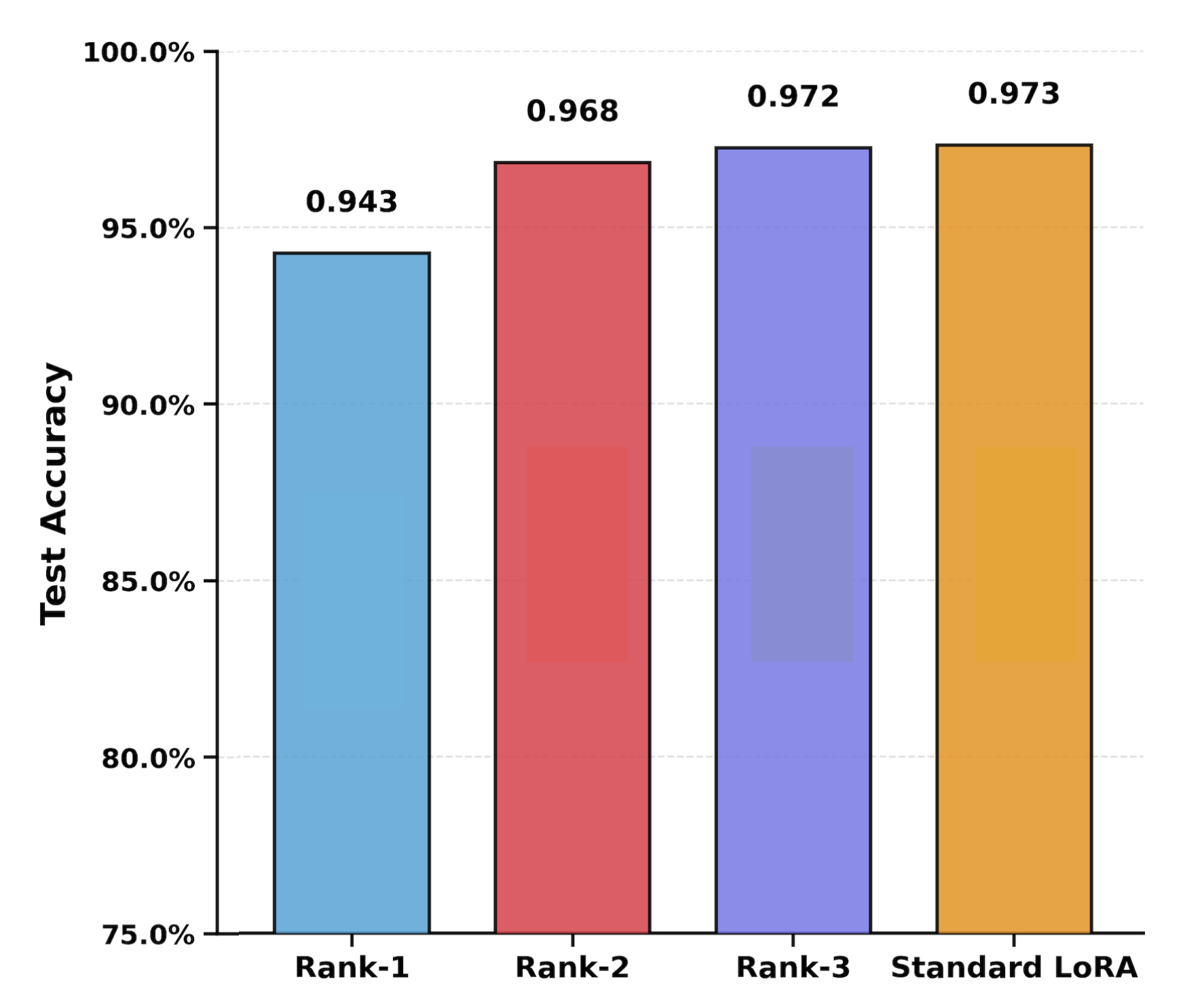}
        \label{fig:MNIST_barplot}
    \end{subfigure}
    \hspace{-0.4cm}
    \begin{subfigure}[b]{0.33\textwidth}
    \includegraphics[width=\linewidth]{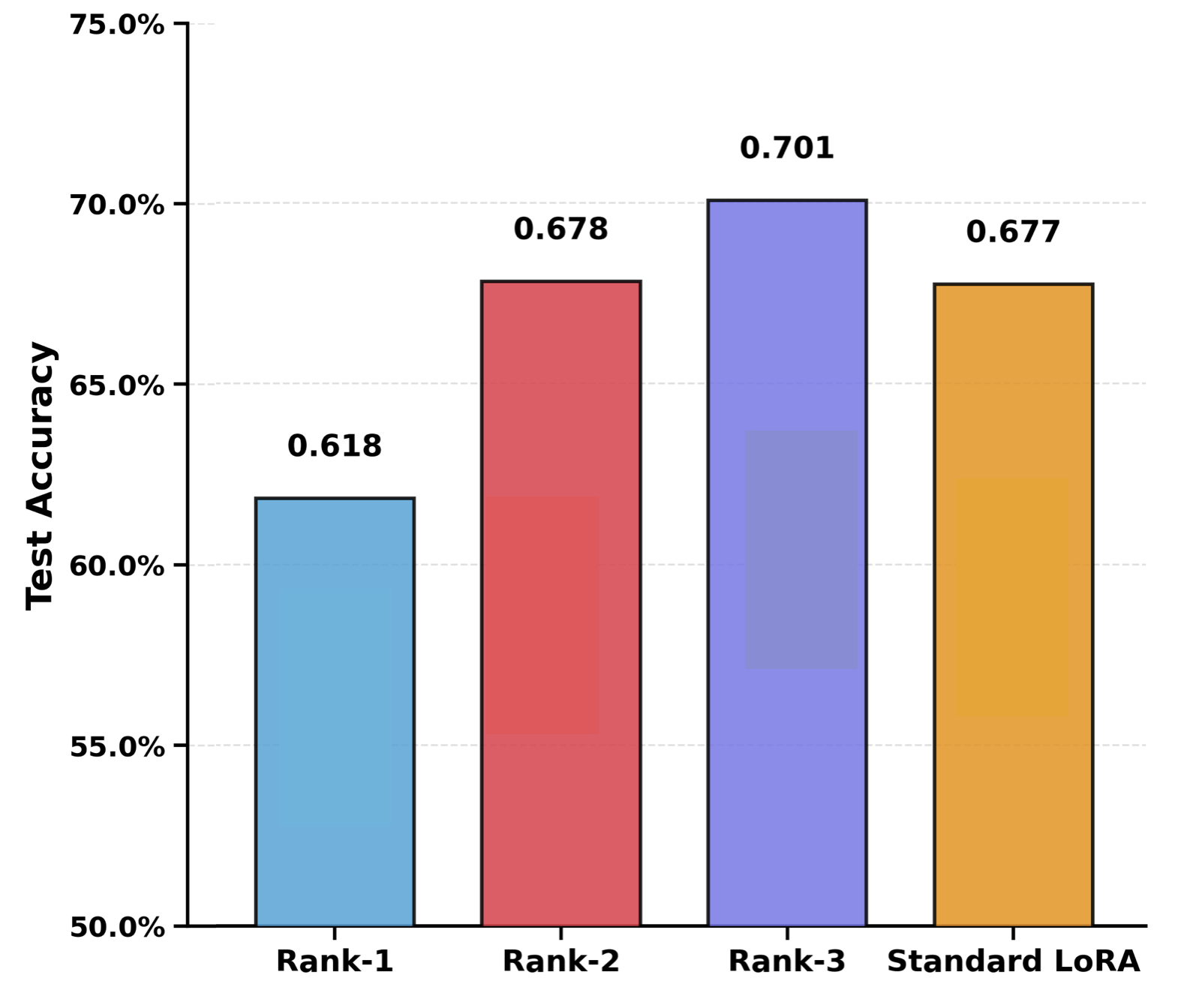}
        \label{fig:CIFAR10_barplot}
    \end{subfigure}
    \hspace{-0.3cm}
    \begin{subfigure}[b]{0.33\textwidth}
    \includegraphics[width=\linewidth]{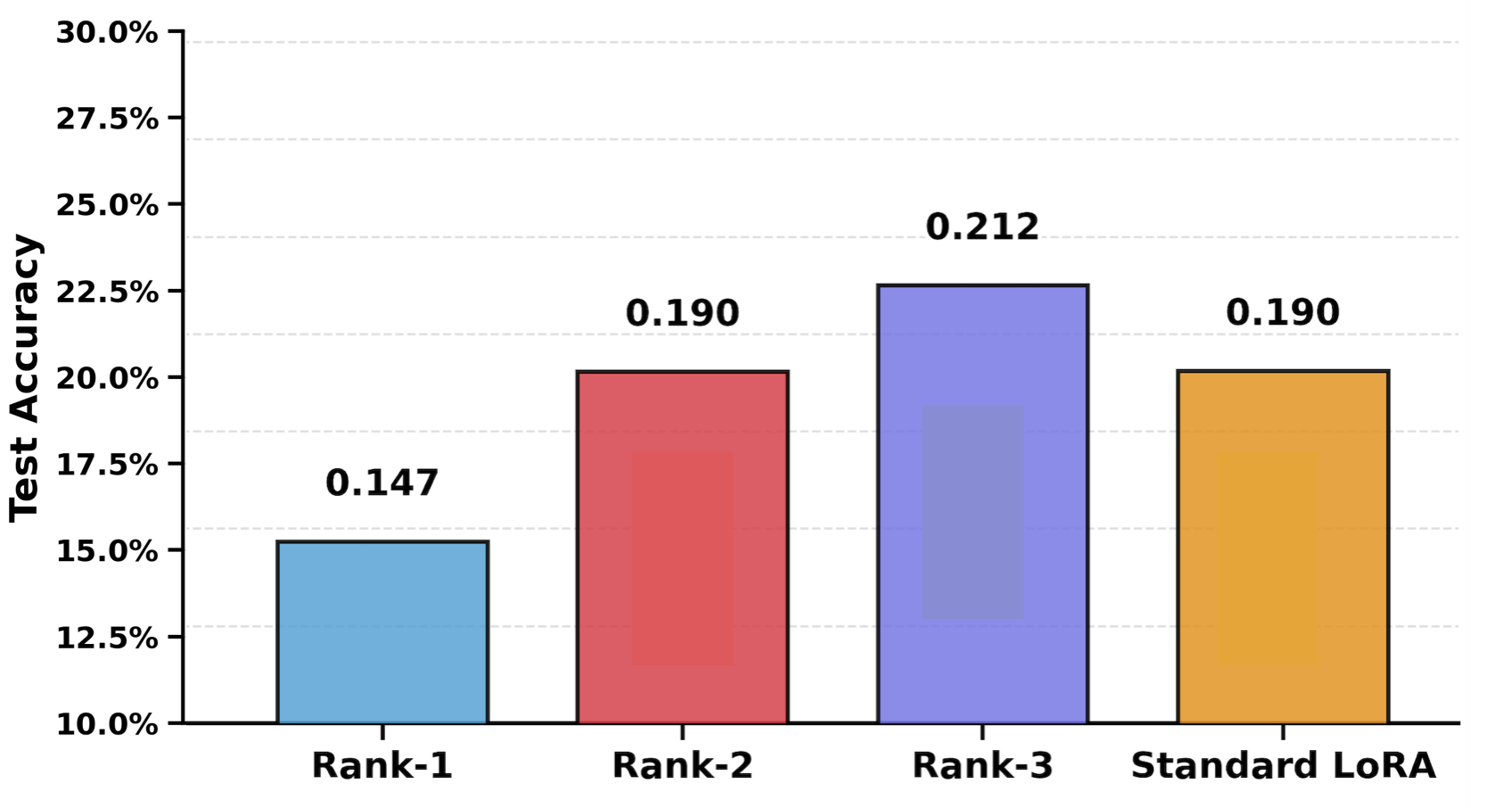}
        \label{fig:CIFAR100_barplot}
    \end{subfigure}\vspace{-0.4cm}
    \caption{Test accuracy of sequential rank-1 LoRA when adapting to new classes across the three datasets. \textit{Left:} MNIST. \textit{Center:} CIFAR10. \textit{Right:} CIFAR100. Note that, on purpose, the pretrained models are trained with good (MNIST), mediocre (CIFAR10) and bad (CIFAR100) accuracy.}
    \label{fig:barplots} \vspace{-0.3cm}
\end{figure}

\textbf{Mathematical formulation.}
%Building upon the Low-Rank Adaptation (LoRA) framework, we now present the mathematical formulation of our sequential rank-1 adaptation procedure. 
%This approach decomposes the adaptation process into a series of rank-1 updates, allowing us to empirically analyze how errors propagate through the sequential learning steps.
In standard LoRA, we parameterize the weight change during fine-tuning as a low-rank decomposition: $\Delta\mathbf{W} = \mathbf{B}\mathbf{A}^\top \in \mathbb{R}^{m \times n}$,
where $\mathbf{A} \in \mathbb{R}^{n \times r}$ and $\mathbf{B} \in \mathbb{R}^{m \times r}$ with $r \ll \min(m, n)$.
In these experiments and w.l.o.g., $r = 3$.
In our approach, instead of optimizing all $r$ components simultaneously, we optimize one rank-1 component at a time, using the residual error from previous components to guide each subsequent step.

In particular, let $\mathbf{W}_0 \in \mathbb{R}^{m \times n}$ be a pre-trained weight matrix (in our case, we have $\mathbf{W}_0$ for every layer of the pretrained fully-connected network). 
Let $\mathcal{L}$ be the task-specific loss function, $f$ is the network function, and $(\mathbf{x}, \mathbf{y})$ represents the task data.
Then, 
\begin{equation}
        \Delta \mathbf{W} = \argmin_{\Delta \mathbf{W}} \mathcal{L}(f(\mathbf{x}; \mathbf{W}_0 + \Delta \mathbf{W}), \mathbf{y}).
    \end{equation}
We define our sequential rank-1 adaptation procedure as follows:
%\begin{enumerate}
%    \item Initialize the residual matrix $\mathbf{R}_0$ based on the target task data:
%    \begin{equation}
%        \mathbf{R}_0 = \arg\min_{\mathbf{R}} \mathcal{L}(f(\mathbf{x}; \mathbf{W}_0 + \mathbf{R}), \mathbf{y})
%    \end{equation}
%    where $\mathcal{L}$ is the task-specific loss function, $f$ is the network function, and $(\mathbf{x}, \mathbf{y})$ represents the task data.
%    \item For each step $k = 1, 2, \ldots, r$:
%    \begin{enumerate}
%        \item Find the best rank-1 approximation $\mathbf{b}_k \mathbf{a}_k^\top$ of the current residual $\mathbf{R}_{k-1}$:
%        \begin{equation}
%            (\mathbf{a}_k, \mathbf{b}_k) = \arg\min_{\mathbf{a} \in \mathbb{R}^n, \mathbf{b} \in \mathbb{R}^m} \left\| \mathbf{R}_{k-1} - \mathbf{b}\mathbf{a}^\top \right\|_F^2
%        \end{equation}       
%        \item Update the residual by subtracting the rank-1 approximation:
%        \begin{equation}
%            \mathbf{R}_k = \mathbf{R}_{k-1} - \mathbf{b}_k \mathbf{a}_k^\top
%        \end{equation}
%    \end{enumerate}
%    \item The final weight adaptation is the sum of all rank-1 components:
%    \begin{equation}
%        \Delta\mathbf{W} = \sum_{k=1}^{r} \mathbf{b}_k \mathbf{a}_k^\top
%    \end{equation}
%\end{enumerate}
%In practice, directly computing the optimal $\mathbf{R}_0$ is challenging. Instead, we can reformulate the procedure in terms of the target task loss minimization:
For $k = 1, 2, \ldots, r$, find the rank-1 update that minimizes the task loss given the previously learned components: \vspace{-0.1cm}
\begin{equation}
    \label{eq:lora_obj}
    \mathbf{a}_k, \mathbf{b}_k = \argmin_{\mathbf{a} \in \mathbb{R}^n, \mathbf{b} \in \mathbb{R}^m} \mathcal{L}\Big(f\Big(\mathbf{x}; \mathbf{W}_0 + \sum_{j=1}^{k-1} \mathbf{b}_j \mathbf{a}_j^\top + \mathbf{b}\mathbf{a}^\top\Big), \mathbf{y}\Big)
\end{equation}
The final adapted model uses the weight matrix on the new data domain: $\mathbf{W} = \mathbf{W}_0 + \sum_{k=1}^{r} \mathbf{b}_k \mathbf{a}_k^\top$. We approximate the optimal rank-1 updates using (stochastic) gradient descent on (\ref{eq:lora_obj}). 
% For the $k$-th component, we initialize $\mathbf{a}_k^{(0)}$ and $\mathbf{b}_k^{(0)}$ (typically with random values) and perform $T_k$ epochs: \begin{small}
% \begin{align}
%     (\mathbf{a}_k^{(t+1)}, \mathbf{b}_k^{(t+1)}) &= (\mathbf{a}_k^{(t)}, \mathbf{b}_k^{(t)}) - (\eta_a, \eta_b) \nabla_{(\mathbf{a}, \mathbf{b})} \mathcal{L}\left(f\left(\mathbf{x}; \mathbf{W}_0 + \sum_{j=1}^{k-1} \mathbf{b}_j \mathbf{a}_j^\top + \mathbf{b}_k^{(t)}{\mathbf{a}_k^{(t)}}^\top\right), \mathbf{y}\right)
% \end{align}
% \end{small}
% The final rank-1 component is then given by $\mathbf{b}_k = \mathbf{b}_k^{(T_k)}$ and $\mathbf{a}_k = \mathbf{a}_k^{(T_k)}$.
All experiments were conducted using Google Colab Pro+ with NVIDIA A100 GPU (40GB memory).

\begin{wrapfigure}{r}{0.62\textwidth}
  \vspace{-0.8cm}
  \begin{center}    \includegraphics[width=0.62\textwidth]{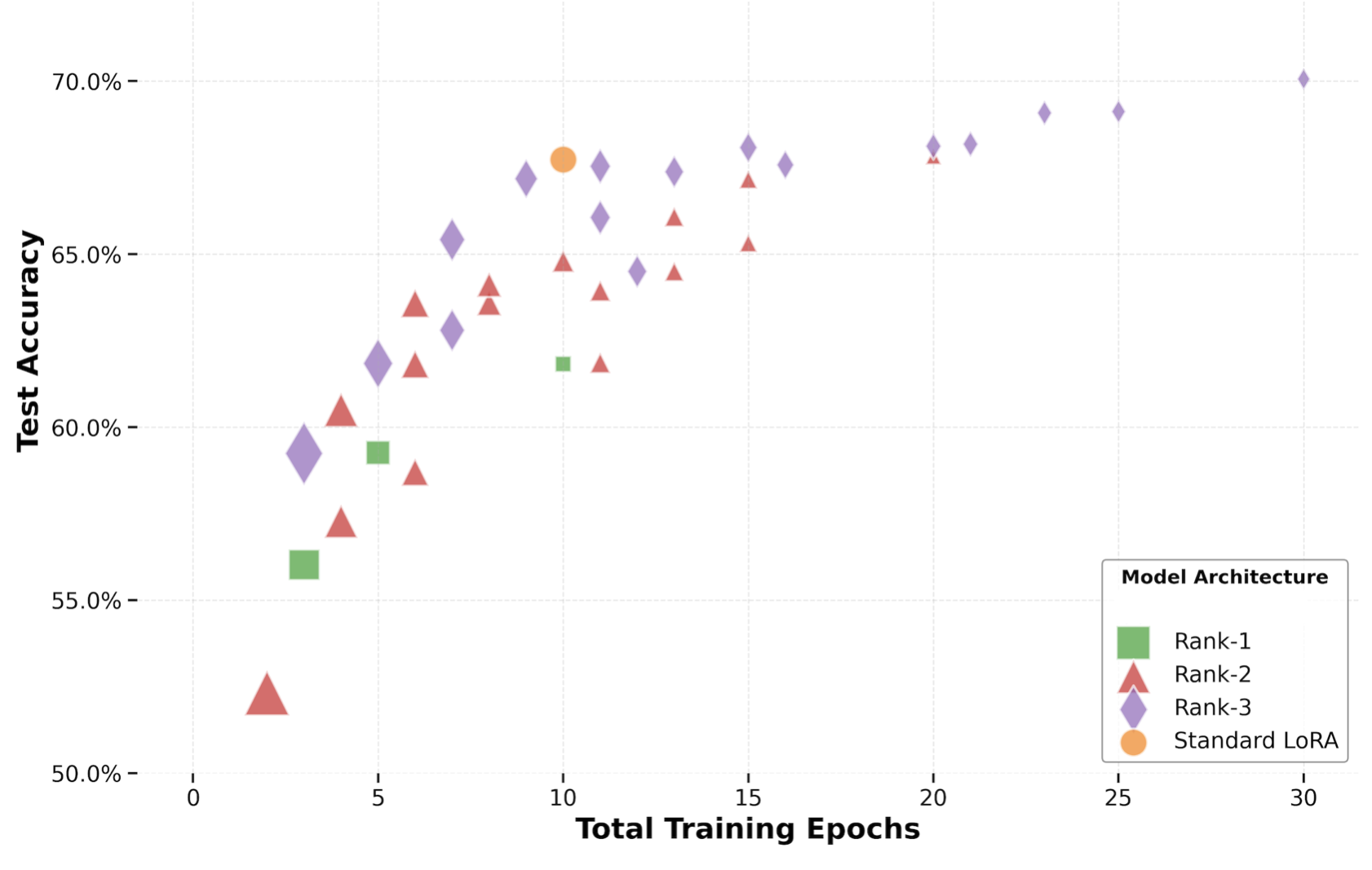}
  \end{center}
  \vspace{-0.4cm}
  \caption{Marker sizes: relative efficiency of each config.}\label{fig:efficiency_bubbleplot}
  \vspace{-0.5cm}
\end{wrapfigure}

\textbf{Adaptation performance across datasets.}
Figure \ref{fig:barplots} presents the test accuracy of sequential rank-1 LoRA components when adapting to new classes across the three datasets. 
The accuracy is measured solely on the new classes (classes ``5-9'' for MNIST and CIFAR10, classes ``50-99'' for CIFAR100), highlighting the adaptation capabilities rather than overall performance.
The standard deviation for all the LoRA experiments is maximum $\sim$1.5\% around the mean over 5 different runs, and the key message of this section remains consistent over runs.

Our results demonstrate that sequential LoRA adaptation effectively transfers knowledge across all three scenarios, though with varying degrees of success depending on the quality of the baseline model. In all cases, we observe that sequential rank-1 LoRA training works at least comparable to standard LoRA, where all $r=3$ components are trained simultaneously.

Definitely, this performance comes with a cost. 
Figure \ref{fig:efficiency_bubbleplot} displays the relationship between parameter efficiency (measured by test accuracy per training epoch) and total training epochs for different model architectures. 
Here, Rank-1 architectures correspond to just using $r=1$ for different number of epochs; Rank-2 architectures correspond to $r=2$, where the components are trained for different \textit{combinations} of total epochs (e.g., some models have been trained with 1$\rightarrow 1$ epochs, while others have been trained with $10 \rightarrow 10$ epochs; more about this in the next paragraph), and so on.
%The bubble sizes represent the relative efficiency of each configuration.

Across all datasets (see also Appendix \ref{app:4}), we observe that sequential rank-1 approaches (noted as ``Rank-1'', ``Rank-2'', and ``Rank-3'') achieve comparable parameter efficiency (with a slight loss of accuracy) compared to standard LoRA. 
Sequential rank-1 models require more total training for comparable accuracy, thus creating a tradeoff, but still maintain favorable parameter-to-performance ratios.
\textit{Yet our approach introduces an interesting property: sequential rank-1 does not require to know apriori the rank of the adaptation; one could check online whether accuracy is sufficient and stop further training. Such a property lacks in standard LoRA: either the user needs to know a good value for $r$, or one needs to consider different $r$ values from scratch before making the final decision.}

\begin{wrapfigure}{r}{0.65\textwidth}
  \vspace{-0.7cm}
  \begin{center}    \includegraphics[width=0.65\textwidth]{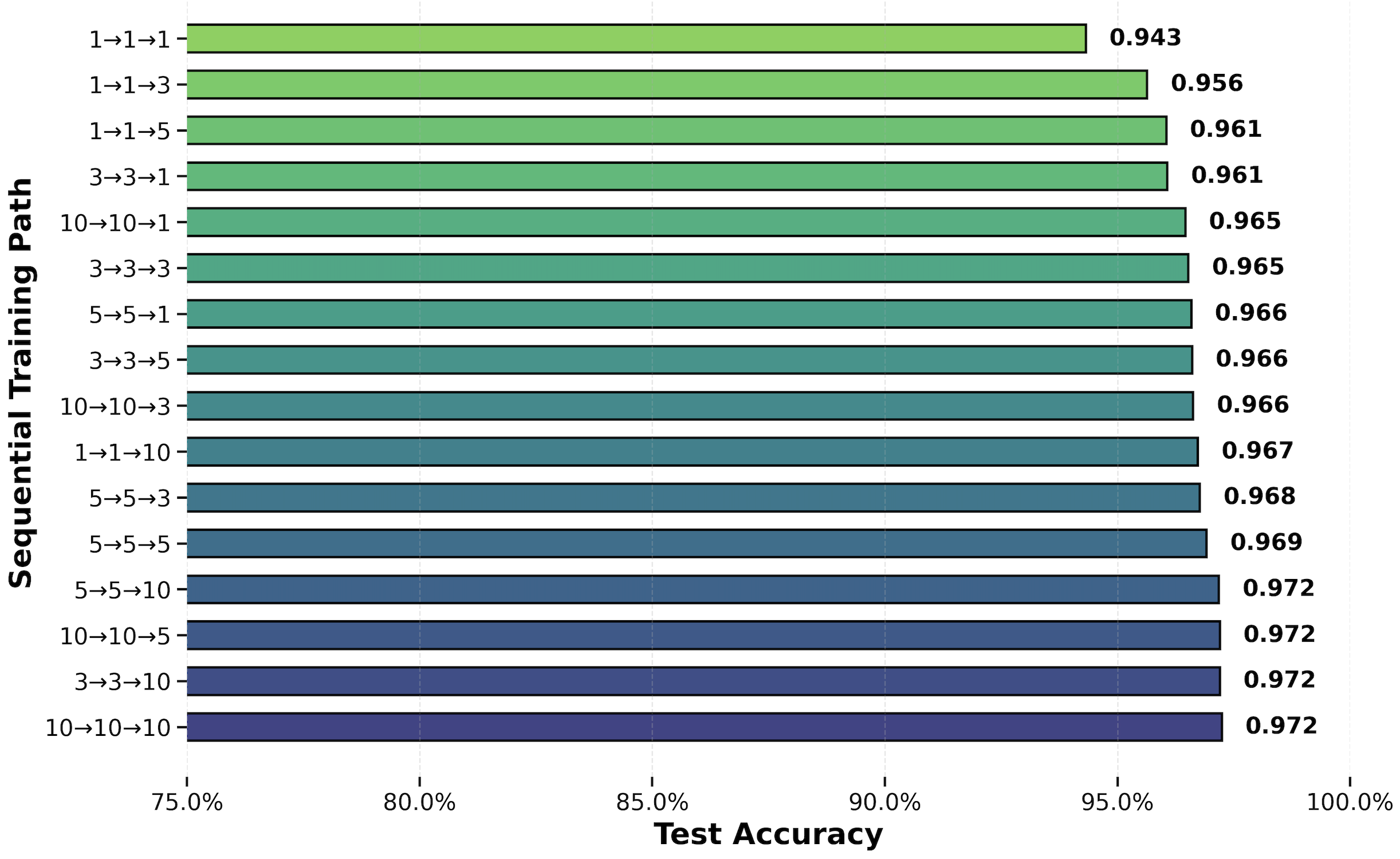}
  \end{center}
  \vspace{-0.2cm}
  \caption{$\alpha \shortrightarrow \beta \shortrightarrow \gamma$ denotes sequential training with $\alpha, \beta$ and $\gamma$ epochs for each component. Not all combinations are shown. }\label{fig:sequential_paths1}\vspace{-0.7cm}
\end{wrapfigure}
\textbf{Sequential training paths.}
Figure \ref{fig:sequential_paths1} illustrates the effectiveness of different sequential training paths for the case of MNIST, where each path represents a sequence of component training durations. 
For example, path ``1$\rightarrow$3$\rightarrow$5'' indicates a rank-3 LoRA where the 1st component is trained for 1 epoch, the 2nd component 3 epochs, and the 3rd component 5 epochs.

In all cases (the rest are provided in Appendix \ref{app:4}, but convey a similar message), it is evident that good first component implies (almost all the times) a better combined final model: front-loaded training schedules perform better, indicating that the first component captures most of the necessary adaptation, with diminishing returns for extensive training of later components. Appendix \ref{app:4} contains more results that support the above observations.

\vspace{-0.3cm}
\section{Limitations}
\vspace{-0.3cm}
Our theoretical analysis and experimental results on sequential rank-1 learning have certain limitations. First, the theoretical analysis is constrained to linear low-rank regression, leaving open challenges in extending to non-linear transformations and complex architectures. Moreover, while sequential rank-1 updates offer flexible rank determination, they may demand more training iterations than simultaneous rank optimization. Furthermore, our LoRA adaptation experiments are restricted to feedforward networks and basic classification tasks. Lastly, it remains an open question to characterize the optimal allocation of training epochs across components.

\vspace{-0.3cm}
\section{Conclusion}
\vspace{-0.3cm}
By examining error propagation in hierarchical learning frameworks, we demonstrate how the accuracy of sequential components is interconnected, with each subsequent step dependent on the precision of preceding estimations. Experimental results on low-rank linear matrix regression and LoRA adaptation (main text and Appendix) validate our hypotheses. From a practical point of view, 
%the error propagation analysis reveals a fundamental challenge in sequential rank-1 adaptation: errors made in earlier steps compound through the sequence. 
this work suggests that more computational resources (larger $T_k$) should be allocated to earlier components to reduce their approximation errors, as these have the largest impact on the final result.
To this end, our work opens up the following future directions:\\
%---\textit{Component Ordering}: The order in which we extract rank-1 components matters. Ideally, we want to extract the most significant components first, as their correct approximation is most critical.\\
---\textit{Adaptive Procedures}: We can develop adaptive procedures that adjust the number of gradient steps $T_k$ based on the estimated approximation error; this connects with learning schedules literature. \\
---\textit{Component Reoptimization}: Periodically refine earlier components after extracting new ones. \\
---\textit{Orthogonality Constraints}: Enforce orthogonality between components to reduce interference.\\
---\textit{Hybrid Approaches}: Combine sequential rank-1 updates with occasional full-rank steps.

% Reference
% For natbib users:
\bibliographystyle{unsrtnat}
\bibliography{reference}
\clearpage

\clearpage
%%%%%%%%%%%%%%%%%%%%%%%
\appendix

\section{Proof of Theorem \ref{thm:main_theorem_1}}\label{app:1}

The proof focuses on bounding two crucial quantities: $\left\|\mathbf{b}_k^\star \mathbf{a}_k^{\star\top}\mathbf{X} - \overline{\mathbf{b}}_k \overline{\mathbf{a}}_k^\top\mathbf{X}\right\|_F$ and $\left\|\mathbf{Y}_k - \mathbf{Y}^\star_k\right\|_F$. The first quantity measures the difference between the true $k$-$th$ rank-1 approximation of $Y$ and the leading rank-1 estimation of $\mathbf{Y}_k$ derived from the $k$-$th$ deflation step, which minimizes the loss specified in (\ref{eq:rank_1_equation}). The second quantity evaluates the distance between the "ground-truth" deflation matrices, $\mathbf{Y}^\star_k$, and the deflation matrices, $\mathbf{Y}_k$, obtained in practice. Together, these bounds provide a comprehensive understanding of the approximation accuracy and the effectiveness of the deflation process.

The difference between $\sum_{k=1}^r\mathbf{b}_k \mathbf{a}^\top_k \mathbf{X}$, the sum of rank-1 approximations returned by Algorithm (\ref{alg:inexact-main-alg}), and $\mathbf{Y}$, the training data label matrix, consists of three components: \vspace{-0.15cm}
\begin{itemize}
    \item \textbf{Ground-truth approximation error}: \(\left\|\mathbf{Y} - \sum_{k=1}^r \mathbf{b}_k^\star \mathbf{a}^{\star\top}_k \mathbf{X}\right\|_F\), which measures the deviation between the true product \(\mathbf{Y}\) and the sum of the exact top \(r\) rank-1 approximations of $\mathbf{Y}$. \vspace{-0.1cm}
    \item \textbf{Propagation error}: \(\left\|\sum_{k=1}^r \mathbf{b}_k^\star \mathbf{a}^{\star\top}_k \mathbf{X} - \sum_{k=1}^r \overline{\mathbf{b}}_k \overline{\mathbf{a}}^\top_k \mathbf{X}\right\|_F\), which accumulates the differences between the exact top rank-1 estimates of each "ground-truth" deflated matrix, \(\mathbf{Y}^\star_k\), and the corresponding empirical deflated matrix, \(\mathbf{Y}_k\). \vspace{-0.1cm}
    \item \textbf{Optimization error}: \(\left\|\sum_{k=1}^r \overline{\mathbf{b}}_k \overline{\mathbf{a}}^\top_k \mathbf{X} - \sum_{k=1}^r \mathbf{b}_k \mathbf{a}^\top_k \mathbf{X}\right\|_F\), which captures the cumulative numerical error from the sub-routine \texttt{rank-1} not being solved exactly. \vspace{-0.15cm}
\end{itemize}

Combining these, we can express the overall difference \(\left\|\mathbf{Y} - \sum_{k=1}^r \mathbf{b}_k \mathbf{a}^\top_k \mathbf{X}\right\|_F\) as follows:
\begin{equation}
    \begin{aligned}
        \left\|\mathbf{Y} - \sum_{k=1}^r \mathbf{b}_k \mathbf{a}^\top_k \mathbf{X}\right\|_F &\leq \left\|\mathbf{Y} - \sum_{k=1}^r \mathbf{b}_k^\star \mathbf{a}^{\star\top}_k \mathbf{X}\right\|_F+ \left\|\sum_{k=1}^r \mathbf{b}_k^\star \mathbf{a}^{\star\top}_k \mathbf{X} - \sum_{k=1}^r \overline{\mathbf{b}}_k \overline{\mathbf{a}}^\top_k \mathbf{X}\right\|_F \\&+  \underbrace{\left\|\sum_{k=1}^r \overline{\mathbf{b}}_k \overline{\mathbf{a}}^\top_k \mathbf{X} - \sum_{k=1}^r \mathbf{b}_k \mathbf{a}^\top_k \mathbf{X}\right\|_F}_{\text{$\leq \sum_{k=1}^r \|\bm{\delta}_k \mathbf{X}\|_F$}}
    \end{aligned}
\end{equation}
The last inequality follows from the definition in (\ref{eq:def-delta-k}). The norm $\left\|\bm{\delta}_k\right\|_F$ largely depends on the sub-routine \texttt{rank-1} and can be controlled as long as $t$, the number of sub-routine iterations, is sufficiently large. Also, ttThe first component equals $\sum_{k=r+1}^p \sigma^\star_k$, which is simply the difference $\sum_{k=1}^p \sigma^\star_k-\sum_{k=1}^r \sigma^\star_k$. By the triangle inequality:
\begin{equation}\label{ineq: 11}
    \begin{aligned}
        \left\|\mathbf{Y} - \sum_{k=1}^r \mathbf{b}_k \mathbf{a}^\top_k \mathbf{X}\right\|_F &\leq \sum_{k=r+1}^p \sigma^\star_k + \left\|\sum_{k=1}^r \mathbf{b}_k^\star \mathbf{a}^{\star\top}_k \mathbf{X} - \sum_{k=1}^r \overline{\mathbf{b}}_k \overline{\mathbf{a}}^\top_k \mathbf{X}\right\|_F +  \sum_{k=1}^r \left\|\bm{\delta}_k \mathbf{X}\right\|_F
        \\& \leq \sum_{k=r+1}^p \sigma^\star_k +
        \sum_{k=1}^r \left\| \mathbf{b}_k^\star \mathbf{a}^{\star\top}_k \mathbf{X} -  \overline{\mathbf{b}}_k \overline{\mathbf{a}}^\top_k \mathbf{X}\right\|_F +  \sum_{k=1}^r \left\|\bm{\delta}_k \mathbf{X}\right\|_F.
    \end{aligned}
\end{equation}
Therefore, our analysis will primarily focus on upper-bounding $\left\|\mathbf{b}_k^\star \mathbf{a}^{\star\top}_k \mathbf{X} -  \overline{\mathbf{b}}_k \overline{\mathbf{a}}^\top_k \mathbf{X}\right\|_F $. Intuitively, this depends on the difference between $\mathbf{Y}_k$ and $\mathbf{Y}_k^\star$. As a foundational step in our proof, we will first provide a characterization of $\left\|\mathbf{Y}_k - \mathbf{Y}_k^\star\right\|_F$.
\begin{lemma}
    \label{lem:matrix_diff_propagate}
    Let $\mathbf{Y}_k^\star$'s and $\mathbf{Y}_k$'s be defined as generated by Algorithm~\ref{alg:exact-main-alg} and Algorithm~\ref{alg:inexact-main-alg}, respectively. Then we have that for all $k \in [r]$,
    \begin{equation}{\label{lemma 1}}
        \|\mathbf{Y}_{k+1}-\mathbf{Y}^\star_{k+1}\|_F \leq \|\mathbf{Y}_k-\mathbf{Y}^\star_k\|_F +\|\mathbf{b}_k^\star \mathbf{a}_k^{\star\top}\mathbf{X} - \overline{\mathbf{b}}_k \overline{\mathbf{a}}_k^\top\mathbf{X}\|_F  + \|\mathbf{\bm{\delta}}_k\mathbf{X}\|_F.
    \end{equation}
\end{lemma}
Lemma \ref{lem:matrix_diff_propagate} upper-bounds $\left\|\mathbf{Y}_{k+1}-\mathbf{Y}^\star_{k+1}\right\|_F$ in terms of $\left\|\mathbf{Y}_k-\mathbf{Y}^\star_k\right\|_F$, $\left\|\mathbf{b}_k^\star \mathbf{a}_k^{\star\top}\mathbf{X} - \overline{\mathbf{b}}_k \overline{\mathbf{a}}_k^\top\mathbf{X}\right\|_F$, and $\left\|\mathbf{\bm{\delta}}_k\mathbf{X}\right\|_F$. To establish a recursive characterization of \(\left\|\mathbf{Y}_k - \mathbf{Y}_k^\star\right\|_F\), we first need to derive an upper bound for \(\left\|\overline{\mathbf{b}}_k \overline{\mathbf{a}}_k^\top \mathbf{X} - \mathbf{b}_k^\star \mathbf{a}_k^{\star\top} \mathbf{X}\right\|_F\).

\iffalse  \textcolor{red}{until here} A direct consequence of Wedin theorem to our scenario is that if $\bm{\delta}>0$, then
\begin{equation*}
    \begin{aligned}
        &\left\| \sin 
\angle(\overline{\mathbf{a}}_k , \mathbf{a}_k^\star) 
\right\|_F^2 + \left\| \sin \angle( \overline{\mathbf{b}}_k , \mathbf{b}_k^\star) \right\|_F^2 \\
&\quad \leq 
\frac{\left\| \mathbf{a}_k^{\star\top}(\mathbf{W}_k - \mathbf{W}_k^\star)
\right\|_F^2 + \left\| (\mathbf{W}_k - \mathbf{W}_k^\star) \mathbf{b}_k^\star \right\|_F^2}{\bm{\delta}^2}.
    \end{aligned}
\end{equation*}
Considering that $\|\mathbf{a}_k^\star\| =|\mathbf{b}_k^\star\| =  1$, and since $\sin 
\angle(\overline{\mathbf{a}}_k , \mathbf{a}_k^\star)$ and $\sin 
\angle(\overline{\mathbf{b}}_k , \mathbf{b}_k^\star)$ are scalars, we'll get
\begin{equation}
\label{ineq:wedin}
    \begin{aligned}
    &\sin^2\angle(\overline{\mathbf{a}}_k , \mathbf{a}_k^\star)  + \sin ^2\angle( \overline{\mathbf{b}}_k , \mathbf{b}_k^\star) \\
&\quad \leq 
\frac{2\left\| (\mathbf{W}_k - \mathbf{W}_k^\star)
\right\|_F^2 }{\bm{\delta}^2}.
    \end{aligned}
\end{equation} 
\fi
\begin{lemma}\label{lem:sum_of_rank_1_diff}
Let $\mathcal{T}_k := \min\left\{\min_{j>k} |\sigma^\star_k - \sigma_{j_k}|, \sigma^\star_k\right\}$. If $\|\mathbf{Y}^\star_k - \mathbf{Y}_k\|_2 < \min_{j>k}|\sigma^\star_k - \sigma^\star_j|$,
then for all \( k \in [r] \), we have:
    \begin{equation}{\label{lemma 2}}
        \left\|\mathbf{b}_k^\star \mathbf{a}_k^{\star\top}\mathbf{X} - \overline{\mathbf{b}}_k \overline{\mathbf{a}}_k^\top \mathbf{X}\right\|_F \leq 
        \left(\frac{3 \sigma_k^\star}{\mathcal{T}_k} + 1\right)\left\|\mathbf{Y}_k^\star - \mathbf{Y}_k\right\|_F.
    \end{equation}
\end{lemma}
Plugging \ref{lemma 2} into \ref{lemma 1} gives a recurrence that depends purely on $\|\bm{\delta}_k\|_F$ and the spectrum of $\mathbf{Y}$:
\begin{equation}\label{lemma2and3}
\begin{aligned}
        \left\|\mathbf{Y}_{k+1}-\mathbf{Y}^\star_{k+1}\right\|_F  &\leq 
        \left\|\mathbf{Y}_k-\mathbf{Y}^\star_k\right\|_F +
        \left(\frac{3 \sigma_k^\star}{\mathcal{T}_k} + 1\right)\left\|\mathbf{Y}_k^\star - \mathbf{Y}_k\right\|_F
        + \left\|\mathbf{\bm{\delta}}_k\mathbf{X}\right\|_F \\& \leq \left(\frac{3 \sigma_k^\star}{\mathcal{T}_k} + 2\right)\left\|\mathbf{Y}_k^\star - \mathbf{Y}_k\right\|_F
        + \left\|\mathbf{\bm{\delta}}_k\mathbf{X}\right\|_F.
        \end{aligned}
\end{equation}
Unrolling this recurrence gives a closed-form upper bound for $ \|\mathbf{Y}_{k}-\mathbf{Y}^\star_{k}\|_F$. Combining this upper bound with \ref{lem:sum_of_rank_1_diff} and plugging the result into \ref{ineq: 11} gives us and upper bound for $\|\mathbf{Y} - \sum_{k=1}^r \mathbf{b}_k \mathbf{a}^\top_k \mathbf{X}\|_F$.

\section{Missing proofs from Appendix~\ref{app:1} and a generalized theorem}

\subsection{Proof of Lemma \ref{lem:sequential_decomposition}} \label{app:lemma1}
\begin{proof}
    Let $\left\{ (\mathbf{a}_k^\star, \mathbf{b}_k^\star) \right\}_{k=1}^{r}$ denote the outputs of Algorithm \ref{alg:exact-main-alg}. According to the third line in Algorithm \ref{alg:exact-main-alg}, starting from \( k = 1 \), we have:
\begin{equation*}
    (\mathbf{a}_1^\star, \mathbf{b}_1^\star) = \argmin_{\mathbf{a} \in \mathbb{R}^d, \mathbf{b} \in \mathbb{R}^m} \frac{1}{2} \left\| \mathbf{Y}^\star_1 - \mathbf{b} \mathbf{a}^\top \mathbf{X} \right\|_F^2.
\end{equation*}
Given that $\mathbf{b} \mathbf{a}^\top$ is a rank-1 matrix, the product $\mathbf{b} \mathbf{a}^\top \mathbf{X}$ is also rank-1. Therefore, according to the Eckart-Young-Mirsky theorem, the best rank-1 approximation of $\mathbf{Y}_1^\star =\mathbf{Y} = \sum_{i = 1}^p \sigma_i^\star \mathbf{u}_i^\star \mathbf{v}_i^{\star\top}$, based on its Singular Value Decomposition, is $\sigma_1^\star \mathbf{u}_1^\star \mathbf{v}_1^{\star\top}$. Thus, $\mathbf{b}_1^\star \mathbf{a}_1^{\star \top} \mathbf{X} = \sigma_1^\star \mathbf{u}_1^\star \mathbf{v}_1^{\star\top}$. 
Next, the deflation step in line 4 yields:
\begin{equation*}
    \mathbf{Y}^\star_{2} := \mathbf{Y}^\star_1 - \mathbf{b}_1^\star \mathbf{a}_1^{\star \top} \mathbf{X} = \mathbf{Y}^\star_1 - \sigma_1^\star \mathbf{u}_1^\star \mathbf{v}_1^{\star\top} = \sum_{i=2}^p \sigma_i^\star \mathbf{u}_i^\star \mathbf{v}_i^{\star\top}.
\end{equation*}
Now, assuming by induction that $\mathbf{Y}^\star_{k}=\sum_{i = k}^p \sigma_i^\star \mathbf{u}_i^\star \mathbf{v}_i^{\star\top}$ and $\mathbf{b}_k^\star \mathbf{a}_k^{\star \top} \mathbf{X} = \sigma_k^\star \mathbf{u}_k^\star \mathbf{v}_k^{\star\top}$, we obtain for \( k+1 \):
\begin{equation*}
    \mathbf{Y}^\star_{k+1} = \mathbf{Y}^\star_k - \mathbf{b}_k^\star \mathbf{a}_k^{\star \top} \mathbf{X} = \sum_{i = k}^p \sigma_i^\star \mathbf{u}_i^\star \mathbf{v}_i^{\star\top} - \sigma_k^\star \mathbf{u}_k^\star \mathbf{v}_k^{\star\top}=  \sum_{i = k+1}^p \sigma_i^\star \mathbf{u}_i^\star \mathbf{v}_i^{\star\top}.
\end{equation*}
Returning to line 3, we have:
\begin{equation*}
    (\mathbf{a}_{k+1}^\star, \mathbf{b}_{k+1}^\star) = \argmin_{\mathbf{a} \in \mathbb{R}^d, \mathbf{b} \in \mathbb{R}^m} \frac{1}{2} \left\| \mathbf{Y}^\star_{k+1} - \mathbf{b} \mathbf{a}^\top \mathbf{X} \right\|_F^2.
\end{equation*}
By the same reasoning as for \( k=1 \), $\mathbf{b}_{k+1}^\star \mathbf{a}_{k+1}^{\star\top}\mathbf{X}$ is equal to the first rank-1 approximation of $\mathbf{Y}_{k+1}^\star = \sum_{i = k+1}^p \sigma_i^\star \mathbf{u}_i^\star \mathbf{v}_i^{\star\top}$, which is $\sigma_{k+1}^\star \mathbf{u}_{k+1}^\star \mathbf{v}_{k+1}^{\star\top}$.
\end{proof}
\subsection{Proof of Lemma \ref{lem:matrix_diff_propagate}}
\begin{proof}
    As we defined in (\ref{eq:def-delta-k}), we'll have:
    \begin{equation*}
    \mathbf{b}_k \mathbf{a}^\top_k =  \overline{\mathbf{b}}_k \overline{\mathbf{a}}_k^ \top+\bm{\delta}_k
\end{equation*}
Plugging this into (\ref{eq:rank1_deflation}) we'll get:
\begin{equation*}
    \mathbf{Y}_{k+1} = \mathbf{Y}_k - (\overline{\mathbf{b}}_k \overline{\mathbf{a}}_k^ \top+\bm{\delta}_k) \mathbf{X}
\end{equation*}
Now let $\mathbf{\Delta}_k = \left\|\mathbf{Y}_k - \mathbf{Y}_k^\star\right\|_F$
\begin{equation*}
    \begin{aligned}
        \mathbf{\Delta}_{k+1} = &\|\mathbf{Y}_{k+1} - \mathbf{Y}_{k+1}^\star\|_F \\
        =& \|\mathbf{Y}_k - (\overline{\mathbf{b}}_k \overline{\mathbf{a}}_k^ \top+\bm{\delta}_k) \mathbf{X} - \mathbf{Y}_k^\star + \mathbf{b}_k^\star \mathbf{a}_k^{\star\top}\mathbf{X}\|_F \\
        =& \|\mathbf{Y}_k - \mathbf{Y}_k^\star
        -(\overline{\mathbf{b}}_k \overline{\mathbf{a}}_k^ \top-\mathbf{b}_k^\star \mathbf{a}_k^{\star\top}) \mathbf{X}-\bm{\delta}_k\mathbf{X}\|_F \\
        \leq & 
        \|\mathbf{Y}_k - \mathbf{Y}_k^\star\|_F + \|(\overline{\mathbf{b}}_k \overline{\mathbf{a}}_k^ \top-\mathbf{b}_k^\star \mathbf{a}_k^{\star\top}) \mathbf{X}\|_F + \|\bm{\delta}_k\mathbf{X}\|_F \\
        \leq & 
        \mathbf{\Delta}_k + \|\overline{\mathbf{b}}_k \overline{\mathbf{a}}_k^ \top\mathbf{X}-\mathbf{b}_k^\star \mathbf{a}_k^{\star\top} \mathbf{X}\|_F + \left\|\bm{\delta}_k\mathbf{X}\right\|_F
    \end{aligned}
\end{equation*}
\end{proof}
\subsection{Proof of Lemma \ref{lem:sum_of_rank_1_diff}}
\begin{proof}
  We start by observing:
   \begin{equation*}
   \begin{aligned}
       \left\|\mathbf{b}_k^\star \mathbf{a}_k^{\star\top}\mathbf{X} - \overline{\mathbf{b}}_k \overline{\mathbf{a}}_k^\top \mathbf{X}\right\|_F &= \|\sigma_k^\star \mathbf{u}_k^\star \mathbf{v}_k^{\star\top} - \sigma_{1_k} \mathbf{u}_{1_k} \mathbf{v}_{1_k}^\top\|_F\\
       & = \|\sigma_k^\star (\mathbf{u}_k^\star \mathbf{v}_k^{\star\top} - \mathbf{u}_{1_k} \mathbf{v}_{1_k}^\top) + (\sigma_k^\star - \sigma_{1_k}) \mathbf{u}_{1_k} \mathbf{v}_{1_k}^\top\|_F \\
       &\leq \|\sigma_k^\star\|_F \cdot \|\mathbf{u}_k^\star \mathbf{v}_k^{\star\top} - \mathbf{u}_{1_k} \mathbf{v}_{1_k}^\top\|_F + \|\sigma_k^\star - \sigma_{1_k}\|_F \cdot \underbrace{\|\mathbf{u}_{1_k}\|_2}_{=1} \cdot \underbrace{\|\mathbf{v}_{1_k}^\top\|_2}_{=1} \\
       &\leq |\sigma_k^\star| \cdot \|\mathbf{u}_k^\star \mathbf{v}_k^{\star\top} - \mathbf{u}_{1_k} \mathbf{v}_{1_k}^\top\|_F + |\sigma_k^\star - \sigma_{1_k}|.
   \end{aligned}
   \end{equation*}
   
   Therefore,
   \begin{equation} \label{ineq: rank-1-difference}
   \|\mathbf{b}_k^\star \mathbf{a}_k^{\star\top}\mathbf{X} - \overline{\mathbf{b}}_k \overline{\mathbf{a}}_k^\top \mathbf{X}\|_F \leq \sigma_k^\star \cdot \|\mathbf{u}_k^\star \mathbf{v}_k^{\star\top} - \mathbf{u}_{1_k} \mathbf{v}_{1_k}^\top\|_F + |\sigma_k^\star - \sigma_{1_k}|.
   \end{equation}

Now, we express the term \(\|\mathbf{u}_k^\star \mathbf{v}_k^{\star\top} - \mathbf{u}_{1_k} \mathbf{v}_{1_k}^\top\|_F\) as follows:
\begin{equation*}
    \begin{aligned}
        \|\mathbf{u}_k^\star \mathbf{v}_k^{\star\top} - \mathbf{u}_{1_k} \mathbf{v}_{1_k}^\top\|_F &= \|\mathbf{u}_k^\star (\mathbf{v}_k^{\star\top} - \mathbf{v}_{1_k}^\top) + (\mathbf{u}_k^\star - \mathbf{u}_{1_k}) \mathbf{v}_{1_k}^\top\|_F\\
        & \leq \underbrace{\|\mathbf{u}_k^\star\|_2}_{=1} \cdot \|\mathbf{v}_k^{\star\top} - \mathbf{v}_{1_k}^\top\|_2 + \| \mathbf{u}_k^\star - \mathbf{u}_{1_k} \|_2 \cdot \underbrace{\|\mathbf{v}_{1_k}^\top\|_2}_{=1} \\
        &\leq \|\mathbf{v}_k^{\star\top} - \mathbf{v}_{1_k}^\top\|_2 + \|\mathbf{u}_k^\star - \mathbf{u}_{1_k}\|_2
    \end{aligned}
\end{equation*}

Substituting this result into inequality \eqref{ineq: rank-1-difference} gives:
\begin{equation}{\label{ineq: rank-1 distance}}
\|\mathbf{b}_k^\star \mathbf{a}_k^{\star\top}\mathbf{X} - \overline{\mathbf{b}}_k \overline{\mathbf{a}}_k^\top \mathbf{X}\|_F 
    \leq \sigma_k^\star \left(\|\mathbf{v}_k^{\star\top} - \mathbf{v}_{1_k}^\top\|_2 + \|\mathbf{u}_k^\star - \mathbf{u}_{1_k}\|_2\right) + \underbrace{|\sigma_k^\star - \sigma_{1_k}|}_{\leq \|\mathbf{Y}_k^\star - \mathbf{Y}_k\|_F}.
\end{equation}

The last inequality follows from Weyl's theorem. To complete the bound, we will now find an upper bound for \(\|\mathbf{v}_k^\star - \mathbf{v}_{1_k}\| + \|\mathbf{u}_k^\star - \mathbf{u}_{1_k}\|\). To do so, we will define two angles:

\begin{equation*}
    \alpha:=\angle\left\{\mathbf{v}_k^\star,\mathbf{v}_{1_k}\right\} \quad \text{and} \quad \beta:=\angle\left\{\mathbf{u}_k^\star,\mathbf{u}_{1_k}\right\}
\end{equation*}

We know that 
\begin{equation*}
    \cos \alpha = \dfrac{\mathbf{v}_k^{\star\top}\mathbf{v}_{1_k}}{\|\mathbf{v}_k^\star\|_2\|\mathbf{v}_{1_k}\|_2}= \mathbf{v}_k^{\star\top}\mathbf{v}_{1_k} \leq 1
\end{equation*}
since $\|\mathbf{v}_k^\star\|_2=1$ and $\|\mathbf{v}_{1_k}\|_2=1$. Therefore 
\begin{equation*}
    \sin^2{\alpha} = 1 - \cos^2{\alpha} 
    = 1 - (\mathbf{v}_k^{\star\top}\mathbf{v}_{1_k})^2 \Rightarrow  (\mathbf{v}_k^{\star\top}\mathbf{v}_{1_k})^2= 1- \sin^2{\alpha}
\end{equation*}
We use the expansion of the square of $\|\mathbf{v}_k^\star - \mathbf{v}_{1_k}\|_2$ to get:
\begin{equation*}
    \begin{aligned}
        \|\mathbf{v}_k^\star - \mathbf{v}_{1_k}\|_2^2 & = \underbrace{\|\mathbf{v}_k^\star\|_2^2}_{\text{$=1$}} +\underbrace{\|\mathbf{v}_{1_k}\|_2^2}_{\text{$=1$}} - 2 (\mathbf{v}_k^{\star\top}\mathbf{v}_{1_k})\\
        & = 2- 2(\mathbf{v}_k^{\star\top}\mathbf{v}_{1_k})\\
        & \leq 2- 2(\mathbf{v}_k^{\star\top}\mathbf{v}_{1_k})^2 = 2-2(1-\sin^2{\bm{\alpha}})\\
        & = 2 \sin^2{\alpha}
    \end{aligned}
\end{equation*}
Thus
\begin{equation}{\label{ineq: 12}}
     \|\mathbf{v}_k^\star - \mathbf{v}_{1_k}\|_2^2 \leq 2 \sin^2{\alpha}
\end{equation}
Following the same procedure for $\beta$, we start with:
\begin{equation*}
    \cos \beta = \dfrac{\mathbf{u}_k^{\star\top}\mathbf{u}_{1_k}}{\|\mathbf{u}_k^\star\|_2\|\mathbf{u}_{1_k}\|_2}= \mathbf{u}_k^{\star\top}\mathbf{u}_{1_k} \leq 1
\end{equation*}
since $\|\mathbf{u}_k^\star\|_2=1$ and $\|\mathbf{u}_{1_k}\|_2=1$. Therefore 
\begin{equation*}
    \sin^2{\beta} = 1 - \cos^2{\beta} 
    = 1 - (\mathbf{u}_k^{\star\top}\mathbf{u}_{1_k})^2 \Rightarrow  (\mathbf{u}_k^{\star\top}\mathbf{u}_{1_k})^2= 1- \sin^2{\beta}
\end{equation*}
We use the expansion of the square of $\|\mathbf{u}_k^\star - \mathbf{u}_{1_k}\|_2$ to get:
\begin{equation*}
    \begin{aligned}
    \left\|\mathbf{u}_k^\star - \mathbf{u}_{1_k}\right\|_2^2 & = \underbrace{\|\mathbf{u}_k^\star\|_2^2}_{\text{$=1$}} +\underbrace{\|\mathbf{u}_{1_k}\|_2^2}_{\text{$=1$}} - 2 (\mathbf{u}_k^{\star\top}\mathbf{u}_{1_k})\\
    & = 2- 2(\mathbf{u}_k^{\star\top}\mathbf{u}_{1_k})\\
    & \leq 2- 2(\mathbf{u}_k^{\star\top}\mathbf{u}_{1_k})^2\\
    & = 2-2(1-\sin^2{\beta}) = 2 \sin^2{\beta}
    \end{aligned}
\end{equation*}
Thus
\begin{equation}{\label{ineq:13}}
     \|\mathbf{u}_k^\star - \mathbf{u}_{1_k}\|_2^2 \leq 2 \sin^2{\beta}
\end{equation}
Using the fact that $a + b \leq \sqrt{2a^2 + 2b^2}$, \ref{ineq: 12}, and \ref{ineq:13} we'll get:
\begin{equation}{\label{ineq: 14}}
    \left\|\mathbf{v}_k^\star-\mathbf{v}_{1_k}\right\|_2+ \|\mathbf{u}_k^\star-\mathbf{u}_{1_k}\|_2 \leq \sqrt{2\|\mathbf{v}_k^\star-\mathbf{v}_{1_k}\|_2^2 + 2\|\mathbf{u}_k^\star-\mathbf{u}_{1_k}\|_2^2} \leq 2\sqrt{\sin^2{\alpha}+ \sin^2{\beta}}
\end{equation}
To proceed, we observe that under the assumption \( \|Y^\star_k - Y_k\|_2 < \min_{j>k}|\sigma^\star_k - \sigma^\star_j| \), Weyl's inequality implies that
\begin{equation*}
    \min_{j>k} |\sigma^\star_k - \sigma_{j_k}| \geq \min_{j>k} |\sigma^\star_k - \sigma^\star_j| - |\sigma_{j_k} - \sigma^\star_j| > \|Y^\star_k - Y_k\|_2 - \|Y^\star_k - Y_k\|_2 = 0.
\end{equation*}
We know \( \sigma^\star_k > 0 \). Since we define \( \mathcal{T}_k := \min\left\{\min_{j>k} |\sigma^\star_k - \sigma_{j_k}|, \sigma^\star_k\right\} \), and both terms are positive, \( \mathcal{T}_k \) is also positive. This satisfies the conditions of Wedin's theorem, allowing us to apply it.

Now since $\alpha:=\angle\left\{\mathbf{v}_k^\star,\mathbf{v}_{1_k}\right\}$ and $ \beta:=\angle\left\{\mathbf{u}_k^\star,\mathbf{u}_{1_k}\right\}$, and because of the fact that $\mathbf{v}_k^\star$ and $\mathbf{u}_k^\star$ are the first right and left singular vectors of $\mathbf{Y}_k^\star$, and $\mathbf{v}_{1_k}$ and $\mathbf{u}_{1_k}$ are the first right and left singular vectors of $\mathbf{Y}_k$, we can apply Wedin's theorem to get:
\begin{equation*}
\begin{aligned}
    \sin^2{\alpha} + \sin^2{\beta} &\leq \dfrac{
    \|\mathbf{v}_k^{\star \top }(\mathbf{Y}_k^\star - \mathbf{Y}_k)\|_F^2 + \|(\mathbf{Y}_k^\star - \mathbf{Y}_k)\mathbf{u}_k^\star\|_F^2}{\mathcal{T}_k^2}\\
    & \leq 
    \dfrac{
    \|\mathbf{v}_k^{\star \top }\|^2_F\|\mathbf{Y}_k^\star - \mathbf{Y}_k\|_F^2 + \|\mathbf{Y}_k^\star - \mathbf{Y}_k\|^2_F\|\mathbf{u}_k^\star\|_F^2}{\mathcal{T}_k^2}\\
    & \leq \dfrac{2\|\mathbf{Y}_k^\star - \mathbf{Y}_k\|^2_F}{\mathcal{T}_k^2}
    \end{aligned}
\end{equation*}
Plugging this in \ref{ineq: 14}, we'll have:
\begin{equation*}
    \|\mathbf{v}_k^\star-\mathbf{v}_{1_k}\|+ \|\mathbf{u}_k^\star-\mathbf{u}_{1_k}\| \leq 2\sqrt{2}\dfrac{\|\mathbf{Y}_k^\star - \mathbf{Y}_k\|_F}{\mathcal{T}_k} \leq 3\dfrac{\|\mathbf{Y}_k^\star - \mathbf{Y}_k\|_F}{\mathcal{T}_k}
\end{equation*}
Plugging this to \ref{ineq: rank-1 distance} and using Weyl's inequality, we'll get:
\begin{equation}
    \|\mathbf{b}_k^\star \mathbf{a}_k^{\star\top}\mathbf{X} - \overline{\mathbf{b}}_k \overline{\mathbf{a}}_k^\top \mathbf{X}\|_F 
    \leq \sigma_k^\star \left(3\dfrac{\|\mathbf{Y}_k^\star - \mathbf{Y}_k\|_F}{\mathcal{T}_k}\right) + \|\mathbf{Y}_k^\star - \mathbf{Y}_k\|_F =  \left(\frac{3 \sigma_k^\star}{\mathcal{T}_k} + 1\right)\|\mathbf{Y}_k^\star - \mathbf{Y}_k\|
\end{equation}
\end{proof}
\subsection{Proof of a generalization of Theorem \ref{thm:main_theorem_1}}
In this section, we proof a more general form of Theorem~\ref{thm:main_theorem_1}, stated below.
\begin{theorem}\label{thm:main_theorem_1_gen}
    Let $\left\{\left(\mathbf{a}_k,\mathbf{b}_k\right)\right\}_{k=1}^r$ be the output of Algorithm~\ref{alg:inexact-main-alg}. Let $\bm{\delta}_k$ be given as in Definition~\ref{def:num_err} with $\left\|\bm{\delta}_k\right\|_F > 0$. Let $\mathbf{Y} = \sum_{k=1}^{r^\star}\sigma_{k}^\star\mathbf{u}_k^\star\mathbf{v}_k^{\star\top}$ be the SVD of $\mathbf{Y}$, with $\sigma_1^\star\geq \dots \geq \sigma_{r^\star}^\star$ and define $\mathbf{Y}^{(r')} = \sum_{k=1}^{r'}\sigma_{k}^\star\mathbf{u}_k^\star\mathbf{v}_k^{\star\top}$. Define the minimum singular value gap $\mathcal{T}^\star_k := \min\left\{\min_{j>k} |\sigma^\star_k - \sigma^\star_j|, \sigma^\star_k\right\}$. Define error bound $E(k)$ as:
    \[
        E(k):=\sigma_{\max}(\mathbf{X})\sum_{k'=0}^{k-1}\left\|\bm{\delta}_{k'}\right\|_F\prod_{j=k'+1}^{k-1}\left(\dfrac{6\sigma^\star_j}{\mathcal{T}_j^\star}+2\right)
    \]
     If $\left\|\bm{\delta}_k\right\|_F$'s are small enough such that $E(k) < \dfrac{1}{2}\min_{j>k}|\sigma^\star_k - \sigma^\star_j|$, then for any $r \leq r' \leq r^\star$, the output of Algorithm \ref{alg:inexact-main-alg} satisfies:
    \begin{equation}
    \begin{aligned}
        \left\|\mathbf{Y}^{(r')}-\sum_{k=1}^r \mathbf{b}_k \mathbf{a}_k^ \top \mathbf{X}\right\|_F
         \leq \sum_{k=r+1}^{r'}\sigma_k^\star + \sigma_{\max}(\mathbf{X})\sum_{k=1}^{r}\sum_{k'=0}^{k} \left\|\mathbf{\bm{\delta}}_{k'}\right\|_F \prod_{j = k'+1}^{k}\left(2 + \frac{6\sigma_j^\star}{\mathcal{T}_k^\star} \right)
        \end{aligned}
    \end{equation}
\end{theorem}
Notice that Theorem~\ref{thm:main_theorem_1_gen} naturally implies Theorem~\ref{thm:main_theorem_1} by taking $r' = r^\star$. We give the proof of Theorem~\ref{thm:main_theorem_1_gen} below.

\begin{proof}
We decompose the approximation error as
\begin{equation}
    \label{eq: Thm3.1}
    \begin{aligned}
        \left\|\mathbf{Y}^{(r')} - \sum_{k=1}^r \mathbf{b}_k \mathbf{a}^\top_k \mathbf{X}\right\|_F &\leq \left\|\mathbf{Y}^{(r')} - \sum_{k=1}^r \mathbf{b}_k^\star \mathbf{a}^{\star\top}_k \mathbf{X}\right\|_F+ \left\|\sum_{k=1}^r \mathbf{b}_k^\star \mathbf{a}^{\star\top}_k \mathbf{X} - \sum_{k=1}^r \overline{\mathbf{b}}_k \overline{\mathbf{a}}^\top_k \mathbf{X}\right\|_F \\&+  \left\|\sum_{k=1}^r \overline{\mathbf{b}}_k \overline{\mathbf{a}}^\top_k \mathbf{X} - \sum_{k=1}^r \mathbf{b}_k \mathbf{a}^\top_k \mathbf{X}\right\|_F
    \end{aligned}
\end{equation}
By Lemma~\ref{lem:sequential_decomposition}, we must have that
\begin{equation}
    \label{eq:bound_t1}
    \left\|\mathbf{Y}^{(r')} - \sum_{k=1}^r \mathbf{b}_k^\star \mathbf{a}^{\star\top}_k \mathbf{X}\right\|_F = \left\|\sum_{k=1}^{r'}\sigma_k^\star\mathbf{u}_k^\star\mathbf{v}_k^\star - \sum_{k=1}^r\sigma_k^\star\mathbf{u}_k^\star\mathbf{v}_k^\star\right\|_F \leq \sum_{k=r+1}^{r'}\sigma_k^\star
\end{equation}
Moreover, by Definition~\ref{def:num_err}, we have
\begin{equation}
    \label{eq:bound_t2}
    \left\|\sum_{k=1}^r \overline{\mathbf{b}}_k \overline{\mathbf{a}}^\top_k \mathbf{X} - \sum_{k=1}^r \mathbf{b}_k \mathbf{a}^\top_k \mathbf{X}\right\|_F \leq \sum_{k=1}^{r}\left\|\bm{\delta}_k\mathbf{X}\right\|_F\leq \sigma_{\max}(\mathbf{X})\sum_{k=1}^{r}\left\|\bm{\delta}_k\right\|_F
\end{equation}
Therefore, it suffice to study the second term in (\ref{eq: Thm3.1}). Towards this end, we use the result in Lemma~\ref{lem:matrix_diff_propagate} and Lemma~\ref{lem:sum_of_rank_1_diff} to obtain
\begin{equation}
    \label{eq:y_prop}
    \left\|\mathbf{Y}_{k+1}-\mathbf{Y}^\star_{k+1}\right\|_F \leq \left\|\mathbf{Y}_k-\mathbf{Y}^\star_k\right\|_F + \left\|\mathbf{b}_k^\star \mathbf{a}_k^{\star\top}\mathbf{X} - \overline{\mathbf{b}}_k \overline{\mathbf{a}}_k^\top\mathbf{X}\right\|_F  + \left\|\mathbf{\bm{\delta}}_k\mathbf{X}\right\|_F
\end{equation}
\begin{equation}
    \label{eq:rank_1_diff_y}
    \left\|\mathbf{b}_k^\star \mathbf{a}_k^{\star\top}\mathbf{X} - \overline{\mathbf{b}}_k \overline{\mathbf{a}}_k^\top \mathbf{X}\right\|_F \leq 
    \left(\frac{3 \sigma_k^\star}{\mathcal{T}_k} + 1\right)\left\|\mathbf{Y}_k^\star - \mathbf{Y}_k\right\|_F.
\end{equation}
Combining (\ref{eq:y_prop}) and (\ref{eq:rank_1_diff_y}), we have that
\[
    \left\|\mathbf{Y}_{k+1}-\mathbf{Y}^\star_{k+1}\right\|_F \leq \left(\frac{3 \sigma_k^\star}{\mathcal{T}_k} + 2\right)\left\|\mathbf{Y}_k^\star - \mathbf{Y}_k\right\|_F + \left\|\mathbf{\bm{\delta}}_k\mathbf{X}\right\|_F
\]
Let the sequence $\left\{Q_k\right\}_{k=1}^{r}$ be defined as
\[
    Q_{k+1} = a_kQ_k + b_k;\quad Q_0 = 0;\quad a_k = 2 + \frac{3\sigma_k^\star}{\mathcal{T}_k};\quad b_k = \|\mathbf{\bm{\delta}}_k\mathbf{X}\|
\]
Then by inequality \ref{lemma2and3} we must have that $\|\mathbf{Y}_k - \mathbf{Y}_k^\star\|_F \leq Q_k $ for all $k$. Invoking lemma \ref{lem:prod_sum_seq} gives:
\begin{equation} \label{ineq 24}
\begin{aligned}
    \|\mathbf{Y}_k - \mathbf{Y}_k^\star\|_F &\leq \sum_{k'=0}^{k-1} \left\|\mathbf{\bm{\delta}}_{k'}\mathbf{X}\right\|_F \prod_{j = k'+1}^{k-1}\left(2 + \frac{3\sigma_j^\star}{\mathcal{T}_k} \right)\\
    &\leq \underbrace{\sigma_{\max}(\mathbf{X})\sum_{k'=0}^{k-1}\left\|\bm{\delta}_{k'}\right\|_F\prod_{j=k'+1}^{k-1}\left(2+\dfrac{3\sigma^\star_j}{\mathcal{T}_j}\right)}_{\text{$:=\hat{E}(k)$}}
\end{aligned}
\end{equation}
We define the right hand side to be equal to $\hat{E}(k)$. Enforcing $\hat{E}(k) \leq \dfrac{1}{2} \min_{j>k}|\sigma_k^\star-\sigma_j^\star|$ gives that
\begin{equation}
    \label{eq:ub_y}
    \|\mathbf{Y}_k - \mathbf{Y}_k^\star\|_2 \leq \dfrac{1}{2} \min_{j>k}|\sigma_k^\star-\sigma_j^\star| \leq \min_{j>k}|\sigma_k^\star-\sigma_j^\star|
\end{equation}
So the condition in Lemma \ref{lem:sum_of_rank_1_diff} is met. Combining (\ref{eq:rank_1_diff_y}) and (\ref{ineq 24}), and notice that $\left\|\bm{\delta}_k\mathbf{X}\right\|_F \leq \sigma_{\max}(\mathbf{X})\left\|\bm{\delta}_k\right\|_F$, we have:
\begin{equation}
    \label{ineq: 21}
    \left\|\mathbf{b}_k^\star \mathbf{a}_k^{\star\top}\mathbf{X} - \overline{\mathbf{b}}_k \overline{\mathbf{a}}_k^\top \mathbf{X}\right\|_F \leq \sigma_{\max}(\mathbf{X})
    \left(\frac{3 \sigma_k^\star}{\mathcal{T}_k} + 1\right)
    \sum_{k'=0}^{k-1} \left\|\mathbf{\bm{\delta}}_{k'}\right\|_F \prod_{j = k'+1}^{k-1}\left(2 + \frac{3\sigma_j^\star}{\mathcal{T}_k} \right)
\end{equation}
Notice that $1 + \frac{3\sigma_k^\star}{\mathcal{T}_k} \leq 2 + \frac{3\sigma_k^\star}{\mathcal{T}_k}$. Therefore, \ref{ineq: 21} becomes:
\begin{equation}
    \label{ineq: 22}
    \left\|\mathbf{b}_k^\star \mathbf{a}_k^{\star\top}\mathbf{X} - \overline{\mathbf{b}}_k \overline{\mathbf{a}}_k^\top \mathbf{X}\right\|_F \leq 
    \sigma_{\max}(\mathbf{X})\sum_{k'=0}^{k-1} \left\|\mathbf{\bm{\delta}}_{k'}\right\|_F \prod_{j = k'+1}^{k}\left(2 + \frac{3\sigma_j^\star}{\mathcal{T}_k} \right)
\end{equation}
Combining (\ref{eq:bound_t1}), (\ref{eq:bound_t2}), and (\ref{ineq: 22}) gives
\begin{align*}
    \left\|\mathbf{Y}^{(r')} - \sum_{k=1}^r \mathbf{b}_k \mathbf{a}^\top_k \mathbf{X}\right\|_F & \leq \sum_{k=r+1}^{r'}\sigma_k^\star + \sigma_{\max}(\mathbf{X})\sum_{k=1}^{r}\left\|\bm{\delta}_k\right\|_F\\
    &\quad\quad\quad+ \sigma_{\max}(\mathbf{X})\sum_{k=1}^{r}\sum_{k'=0}^{k-1} \left\|\mathbf{\bm{\delta}}_{k'}\right\|_F \prod_{j = k'+1}^{k}\left(2 + \frac{3\sigma_j^\star}{\mathcal{T}_k} \right)\\
    & = \sum_{k=r+1}^{r'}\sigma_k^\star + \sigma_{\max}(\mathbf{X})\sum_{k=1}^{r}\sum_{k'=0}^{k} \left\|\mathbf{\bm{\delta}}_{k'}\right\|_F \prod_{j = k'+1}^{k}\left(2 + \frac{3\sigma_j^\star}{\mathcal{T}_k} \right)
\end{align*}
Finally, due to (\ref{eq:ub_y}) and the Weyl's inequality, we must have that $\left|\sigma_{j_k} - \sigma_k^\star\right|\leq \frac{1}{2}\min_{j > k}\left|\sigma_k^\star - \sigma_j^\star\right|$. Thus, we have that $\mathcal{T}_k\geq \frac{1}{2}\mathcal{T}_k^\star$. This allows us to define
\[
    E(k) = \sigma_{\max}(\mathbf{X})\sum_{k'=0}^{k-1}\left\|\bm{\delta}_{k'}\right\|_F\prod_{j=k'+1}^{k-1}\left(\dfrac{6\sigma^\star_j}{\mathcal{T}_j^\star}+2\right)
\]
and obtain that $E(k)\geq \hat{E}(k)$. Thus, enforcing $E(k) \leq \frac{1}{2}\min_{j > k}\left|\sigma_k^\star- \sigma_j^\star\right|$ suffices. Moreover, we have
\[
    \left\|\mathbf{Y}^{(r')} - \sum_{k=1}^r \mathbf{b}_k \mathbf{a}^\top_k \mathbf{X}\right\|_F \leq \sum_{k=r+1}^{r'}\sigma_k^\star + \sigma_{\max}(\mathbf{X})\sum_{k=1}^{r}\sum_{k'=0}^{k} \left\|\mathbf{\bm{\delta}}_{k'}\right\|_F \prod_{j = k'+1}^{k}\left(2 + \frac{6\sigma_j^\star}{\mathcal{T}_k^\star} \right)
\]
which completes the proof.
\end{proof}

\section{Proof of Theorem \ref{thm:gen_noiseless}}\label{app:2}

\textbf{Proof overview}. A rough idea of showing this theorem can build upon our characterization of the training error. Let $\hat{\mathbf{b}}_k^\star$'s be output of Algorithm~\ref{alg:exact-main-alg} when using $\mathbf{Y}^\star$ as the label. We consider the following orthonormal basis of $\mathbb{R}^m$ extended from $\hat{\mathbf{b}}_k^\star$'s:
\[
    \hat{\mathbf{b}}_1,\dots,\hat{\mathbf{b}}_m;\;\; \hat{\mathbf{b}}_i = \hat{\mathbf{b}}_k^\star / \|\hat{\mathbf{b}}_k^\star\|_2 \text{ if } k\leq r^\star
\]
Let $\hat{\mathbf{B}}\in\mathbb{R}^{m\times r^\star}$ consist of $\hat{\mathbf{b}}_1,\dots\hat{\mathbf{b}}_{r^\star}$, and let $\hat{\mathbf{B}}^\perp\in\mathbb{R}^{m\times (m - r^\star)}$ consist of $\hat{\mathbf{b}}_{r^\star+1},\dots\hat{\mathbf{b}}_{m}$. Then, we can write $\mathbf{Y}^\star$ as:
\[
    \mathbf{Y} = \mathbf{W}^\star\mathbf{X} + \hat{\mathbf{B}}\hat{\mathbf{B}}^\top\bm{\mathcal{E}} + \hat{\mathbf{B}}^{\perp}\hat{\mathbf{B}}^{\perp\top}\bm{\mathcal{E}} = \hat{\mathbf{B}}(\bm{\Sigma}\hat{\mathbf{A}}^\top\mathbf{X} + \bm{\mathcal{E}}_1) + \hat{\mathbf{B}}^\perp\bm{\mathcal{E}}_2,
\]
where $\bm{\mathcal{E}}_1\in\mathbb{R}^{r^\star\times n}$ and $\bm{\mathcal{E}}_2\in\mathbb{R}^{(m-r^\star)\times n}$ are noise matrices with I.I.D. Gaussian entries. Therefore, based on the above decomposition, $\bm{\mathcal{E}}_1$ can be seen as the unavoidable noise, which will adds up to the training error, and $\bm{\mathcal{E}}_2$ is the error that can be avoided if we solve for only the top $r^\star$ components. Of course, $\hat{\mathbf{B}}(\bm{\Sigma}\hat{\mathbf{A}}^\top\mathbf{X} + \bm{\mathcal{E}}_1)$ is not the truncated top-$r^\star$ SVD of $\mathbf{Y}$ since $\bm{\Sigma}\hat{\mathbf{A}}^\top\mathbf{X} + \bm{\mathcal{E}}_1$ does not have orthogonal rows. However, when $\bm{\mathcal{E}}_1$ is small, this term approximates the truncated SVD well enough. Base on this intuition, we have the following lemma:
\begin{lemma}
    \label{lem:noise_comp}
    Let $\mathbf{Y}^\star$ to have the SVD $\mathbf{Y}^\star = \sum_{k=1}^r\hat{\sigma}_k\hat{\mathbf{u}}\hat{\mathbf{v}}^\star$, and let $\mathbf{Y} = \mathbf{Y}^\star + \bm{\mathcal{E}}$ to have SVD $\mathbf{Y} = \sum_{k=1}^m\sigma_k^\star\mathbf{u}_k^\star\mathbf{v}_k^\star$. Let $\mathbf{Y}^{\star(\hat{m})}$ and $\mathbf{Y}^{(\hat{m})}$ be the truncated rank-$\hat{m}$ SVD of $\mathbf{Y}^\star$ and $\mathbf{Y}$, respectively. Then with probability at least $1 - \delta$ we have that
    \[
        \left\|\mathbf{Y}^{\star(\hat{m})} - \mathbf{Y}^{(\hat{m})} \right\|_F \leq O\left(\varepsilon\sqrt{n\log\sfrac{1}{\delta}}\left(\hat{m}+\sqrt{\frac
        {\min\{r,\hat{m}\}}{\mathcal{T}_{\min}^\star}}\right)\right)
    \]
\end{lemma}
The proof of Lemma~\ref{lem:noise_comp} is given in Appendix~\ref{sec:proof_lem_noise_comp}. With the help of Lemma~\ref{lem:noise_comp}, the proof of Theorem~\ref{thm:noisy_gen} involves choosing a reference label $\mathbf{Y}^{(r)}$ that involves only the relevant noise that will be fitted by Algorithm~\ref{alg:inexact-main-alg}. We then control the generalization error by estimating the difference between $\mathbf{W}^\star\mathbf{X}$ and $\mathbf{Y}^{\star(r)}$, and the difference between $\mathbf{Y}^{\star(r)}$ and $\mathbf{Y}^{(r)}$ using Lemma~\ref{lem:noise_comp}.

\subsection{More details in the proof of Theorem~\ref{thm:gen_noiseless}}
\label{sec:proof_gen_noiseless}
Utilizing Lemma~\ref{lem:sum_of_rank_1_diff}, we could get that
\[
    \left\|\mathbf{b}_k^\star\mathbf{a}_k^{\star\top}\mathbf{X} - \mathbf{b}_k\mathbf{a}_k^\top\mathbf{X}\right\|_F \leq \left(2 + \frac{6\sigma_k^\star}{\mathcal{T}_k^\star}\right) + \sigma_{\max}(\mathbf{X})\left\|\bm{\delta}_k\right\|_F
\]
Plugging in (\ref{ineq 24}), we have that
\[
    \left\|\mathbf{b}_k^\star\mathbf{a}_k^{\star\top}\mathbf{X} - \mathbf{b}_k\mathbf{a}_k^\top\mathbf{X}\right\|_F \leq \sigma_{\max}(\mathbf{X})\sum_{k'=0}^{k}\left\|\bm{\delta}_{k'}\right\|_F\prod_{j=k'+1}^{k}\left(2 + \frac{6\sigma_j^\star}{\mathcal{T}_k^\star} \right)
\]
Since $\sigma_{\min}(\mathbf{X})  > 0 $, we can then have
\begin{gather*}
    \left\|\mathbf{b}_k^\star\mathbf{a}_k^{\star\top}\mathbf{X} - \mathbf{b}_k\mathbf{a}_k^\top\mathbf{X}\right\|_F \geq \sigma_{\min}(\mathbf{X})\left\|\mathbf{b}_k^\star\mathbf{a}_k^{\star\top} - \mathbf{b}_k\mathbf{a}_k^\top\right\|_F\\\Rightarrow \left\|\mathbf{b}_k^\star\mathbf{a}_k^{\star\top} - \mathbf{b}_k\mathbf{a}_k^\top\right\|_F \leq \frac{1}{\sigma_{\min}(\mathbf{X})} \left\|\mathbf{b}_k^\star\mathbf{a}_k^{\star\top}\mathbf{X} - \mathbf{b}_k\mathbf{a}_k^\top\mathbf{X}\right\|_F
\end{gather*}
This implies that
\[
    \left\|\mathbf{b}_k^\star\mathbf{a}_k^{\star\top} - \mathbf{b}_k\mathbf{a}_k^\top\right\|_F \leq \kappa(\mathbf{X})\sum_{k'=0}^{k}\left\|\bm{\delta}_{k'}\right\|_F\prod_{j=k'+1}^{k}\left(2 + \frac{6\sigma_j^\star}{\mathcal{T}_k^\star} \right)
\]
which proves the first statememt. To prove the second statement, we directly use Theorem~\ref{thm:main_theorem_1} to get that
\begin{align*}
    \left\|\mathbf{W}^\star - \sum_{k=1}^r\mathbf{b}_k\mathbf{a}_k^\top\right\|_F & \leq \frac{1}{\sigma_{\min}(\mathbf{X})}\left\|\mathbf{W}^\star\mathbf{X} - \sum_{k=1}^r\mathbf{b}_k\mathbf{a}_k^\top\mathbf{X}\right\|_F\\
    & = \frac{1}{\sigma_{\min}(\mathbf{X})}\left\|\mathbf{Y} - \sum_{k=1}^r\mathbf{b}_k\mathbf{a}_k^\top\mathbf{X}\right\|_F\\
    & \leq \sum_{k=r+1}^{r^\star} \frac{\sigma_k^\star}{\sigma_{\min}(\mathbf{X})} + \kappa(\mathbf{X}) \sum_{k=1}^r \sum_{k'=1}^{k} \|\bm{\delta}_{k'}\|_F \prod_{j=k'+1}^{k} \left( 2 + \dfrac{6\sigma_j^\star}{\mathcal{T}_j^\star} \right)
\end{align*}
\subsection{Proof of Lemma~\ref{lem:noise_comp}}
\label{sec:proof_lem_noise_comp}
\begin{proof}
    By Lemma~\ref{lem:gaus_concentration}, we have that with probability $1 - \delta$
    \[
        \sigma_{\max}\left(\bm{\mathcal{E}}\right) \leq O\left(\varepsilon\left(\sqrt{n} + \sqrt{\log\frac{1}{\delta}}\right)\right)
    \]
    To start, by Weyl's inequality, we have that
    \[
        \left|\hat{\sigma}_k - \sigma_k^\star\right| \leq \sigma_{\max}\left(\bm{\mathcal{E}}\right) \leq O\left(\varepsilon\left(\sqrt{n} + \sqrt{\log\frac{1}{\delta}}\right)\right)
    \]
    Therefore, taking $\varepsilon \leq O\left(\frac{\mathcal{T}_{\min}^\star}{\sqrt{n} + \sqrt{\log\frac{1}{\delta}}}\right)$ ensures that $\min\left\{\min_{j\neq k}\left\{\left|\sigma_j - \sigma_k^\star\right|\right\},\sigma_k^\star\right\}\geq \frac{1}{2}\mathcal{T}_{\min}^\star$.
    Thus, by Wedin's Theorem, we have that
    \[
        \hat{\mathbf{u}}_k^\top\mathbf{u}_k^\star+\hat{\mathbf{v}}_k^\top\mathbf{v}_k^\star \leq 2 - \frac{2}{\mathcal{T}_k^\star}\left(\left\|\bm{\mathcal{E}}^\top\mathbf{u}_k^\star\right\|_2^2 + \left\|\bm{\mathcal{E}}\mathbf{v}_k^\star\right\|_2^2\right)
    \]
    We will consider two cases. First, when $\hat{m}\leq r$, we have
    \begin{align*}
        \left\|\mathbf{Y}^{\star(\hat{m})} - \mathbf{Y}^{(\hat{m})} \right\|_F & = \left\|\sum_{k=1}^{\hat{m}}\left(\sigma_k^\star\mathbf{u}^\star\mathbf{v}^\star - \hat{\sigma}_k\hat{\mathbf{u}}_k\hat{\mathbf{v}}_k\right)\right\|_F\\
        & \leq \left\|\sum_{k=1}^{\hat{m}}\left(\hat{\sigma}_k - \sigma_k^\star\right)\mathbf{u}_k\mathbf{v}_k\right\|_F + \left\|\hat{\mathbf{U}}_{\hat{m}}\bm{\Sigma}_{\hat{m}}^\star\hat{\mathbf{V}}_{\hat{m}}^\top -\mathbf{U}_{\hat{m}}^{\star}\bm{\Sigma}_{\hat{m}}^\star\hat{\mathbf{V}}_{\hat{m}}^{\top\star}\right\|_F\\
        & \leq \sum_{r=1}^{\hat{m}}\left|\hat{\sigma}_k - \sigma_k^\star\right| + \left\|\left(\hat{\mathbf{U}}_{\hat{m}} - \mathbf{U}_{\hat{m}}^\star\right)\bm{\Sigma}_{\hat{m}}^\star\right\|_F + \left\|\left(\hat{\mathbf{V}}_{\hat{m}} - \mathbf{V}_{\hat{m}}^\star\right)\bm{\Sigma}_{\hat{m}}^\star\right\|_F\\
        & \leq \sum_{r=1}^{\hat{m}}\left|\hat{\sigma}_k - \sigma_k^\star\right| + \sigma_{1}^\star\left(\left\|\hat{\mathbf{U}}_{\hat{m}} - \mathbf{U}_{\hat{m}}^\star\right\|_F + \left\|\hat{\mathbf{V}}_{\hat{m}} - \mathbf{V}_{\hat{m}}^\star\right\|_F\right)\\
        & \leq \sum_{r=1}^{\hat{m}}\left|\hat{\sigma}_k - \sigma_k^\star\right| + 2\sigma_{1}^\star\left(\left\|\hat{\mathbf{U}}_{\hat{m}} - \mathbf{U}_{\hat{m}}^\star\right\|_F^2 + \left\|\hat{\mathbf{V}}_{\hat{m}} - \mathbf{V}_{\hat{m}}^\star\right\|_F^2\right)^{\frac
        {1}{2}}
    \end{align*}
    Notice that
    \begin{align*}
        \left\|\hat{\mathbf{U}}_{\hat{m}} - \mathbf{U}_{\hat{m}}^\star\right\|_F^2 & = 2{\hat{m}} - 2\left\langle\hat{\mathbf{U}}_{\hat{m}} ,\mathbf{U}_{\hat{m}}^\star\right\rangle  = 2{\hat{m}}-2\sum_{k=1}^{{\hat{m}}}\hat{\mathbf{u}}_{k}^\top\mathbf{u}_{k}^\star\\
        \left\|\hat{\mathbf{V}}_{\hat{m}} - \mathbf{V}_{\hat{m}}^\star\right\|_F^2 & = 2{\hat{m}} - 2\left\langle \hat{\mathbf{V}}_{\hat{m}} ,\mathbf{V}_{\hat{m}}^\star\right\rangle  = 2{\hat{m}}-2\sum_{k=1}^{{\hat{m}}}\hat{\mathbf{v}}_{k}^\top\mathbf{v}_{k}^\star
    \end{align*}
    Therefore
    \begin{align*}
        \left\|\hat{\mathbf{U}}_{\hat{m}} - \mathbf{U}_{\hat{m}}^\star\right\|_F^2 + \left\|\hat{\mathbf{V}}_{\hat{m}} - \mathbf{V}_{\hat{m}}^\star\right\|_F^2 & \leq 4k - 2\sum_{k=1}^{\hat{m}}\left(\hat{\mathbf{u}}_{k}^\top\mathbf{u}_{k}^\star+\hat{\mathbf{v}}_{k}^\top\mathbf{v}_{k}^\star\right)\\
        & \leq \frac{4}{\mathcal{T}_{\min}^\star}\sum_{k=1}^{\hat{m}}\left(\left\|\bm{\mathcal{E}}^\top\mathbf{u}_k^\star\right\|_2^2 + \left\|\bm{\mathcal{E}}\mathbf{v}_k^\star\right\|_2^2\right)\\
        & = \frac{4}{\mathcal{T}_{\min}^\star}\left(\left\|\bm{\mathcal{E}}^\top\mathbf{U}_{\hat{m}}^\star\right\|_F^2 + \left\|\bm{\mathcal{E}}\mathbf{V}_{\hat{m}}^\star\right\|_F^2\right)
    \end{align*}
     Since $\bm{\mathcal{E}}\in\mathbb{R}^{m\times n}$ contains I.I.D. Gaussian entries from $\mathcal{N}(0, \varepsilon^2)$, we must have that $\mathbf{U}_k^\star\bm{\mathcal{E}}\in\mathbb{R}^{\hat{m}\times n}$ and $\bm{\mathcal{E}}\mathbf{V}_k^\star\in\mathbb{R}^{m\times \hat{m}}$ contains I.I.D. Gaussian entries from $\mathcal{N}(0, \varepsilon^2)$. By Lemma~\ref{lem:gaus_concentration}, we have that with probability at least $1-\delta$, it holds that
    \[
        \left\|\mathbf{U}_{\hat{m}}^{\star\top}\bm{\mathcal{E}}\right\|_F^2 + \left\|\bm{\mathcal{E}}\mathbf{V}_{\hat{m}}^\star\right\|_F^2\leq O\left(\varepsilon^2(m+n)\hat{m}\log\sfrac{1}{\delta}\right)
    \]
    Thus, we have
    \[
        \left\|\hat{\mathbf{U}}_{\hat{m}} - \mathbf{U}_{\hat{m}}^\star\right\|_F^2 + \left\|\hat{\mathbf{V}}_{\hat{m}} - \mathbf{V}_{\hat{m}}^\star\right\|_F^2 \leq O\left(\frac{\varepsilon^2}{\mathcal{T}_{\min}^\star}\hat{m}(m+n)\log\sfrac{1}{\delta}\right)
    \]
    Combining the results above, we have
    \[
        \left\|\mathbf{Y}^{\star(\hat{m})} - \mathbf{Y}^{(\hat{m})} \right\|_F \leq O\left(\varepsilon\hat{m}\left(\sqrt{n} + \sqrt{\log\sfrac{1}{\delta}}\right)\right) + O\left(\frac{\varepsilon}{\sqrt{\mathcal{T}_{\min}^\star}}\sqrt{r(m+n)\log\sfrac{1}{\delta}}\right)
    \]
    Next, we consider the case $\hat{m}\geq r$. In this case, we have
    \begin{align*}
        \left\|\mathbf{Y}^{\star(\hat{m})} - \mathbf{Y}^{(\hat{m})} \right\|_F & \leq \left\|\mathbf{Y}^{\star} - \mathbf{Y}^{(r)}\right\|_F + \left\|\sum_{k=r+1}^{\hat{m}}\sigma_k\mathbf{u}_k\mathbf{v}_k^\top\right\|_F\\
        & \leq O\left(\varepsilon r\left(\sqrt{n} + \sqrt{\log\sfrac{1}{\delta}}\right)\right) + O\left(\frac{\varepsilon}{\sqrt{\mathcal{T}_{\min}^\star}}\sqrt{r(m+n)\log\sfrac{1}{\delta}}\right) + \sum_{k=r+1}^{\hat{m}}\sigma_k
    \end{align*}
    Notice that by Weyl's inequality, for all $k\geq r$
    \[
        \sigma_k \leq \left|\sigma_k - 0\right| \leq \sigma_{\max}(\bm{\mathcal{E}}) \leq O\left(\varepsilon\left(\sqrt{n} + \sqrt{\log\frac{1}{\delta}}\right)\right)
    \]
    This gives
    \[
        \left\|\mathbf{Y}^{\star(\hat{m})} - \mathbf{Y}^{(\hat{m})} \right\|_F \leq O\left(\varepsilon\hat{m}\left(\sqrt{n} + \sqrt{\log\sfrac{1}{\delta}}\right)\right) + O\left(\frac{\varepsilon}{\sqrt{\mathcal{T}_{\min}^\star}}\sqrt{r(m+n)\log\sfrac{1}{\delta}}\right)
    \]
    Combining the two cases, and using $m \leq n$, we have that
    \[
        \left\|\mathbf{Y}^{\star(\hat{m})} - \mathbf{Y}^{(\hat{m})} \right\|_F \leq O\left(\varepsilon\sqrt{n\log\sfrac{1}{\delta}}\left(\hat{m}+\sqrt{\frac
        {\min\{r,\hat{m}\}}{\mathcal{T}_{\min}^\star}}\right)\right)
    \]
\end{proof}

\subsection{Proof of Theorem~\ref{thm:noisy_gen}}
\label{se:proof_noisy_gen}
Given the SVD of $\mathbf{Y}$ and $\mathbf{Y}^\star$ as $\mathbf{Y}=\sum_{k=1}^{p}\sigma_k^\star\mathbf{u}_k^\star\mathbf{v}_k^\star$ and $\mathbf{Y}^\star = \sum_{k=1}^{r^\star}\hat{\sigma}_k\hat{\mathbf{u}}_k\hat{\mathbf{v}}_k^\top$, we define
\[
    \mathbf{Y}^{(r)} = \sum_{k=1}^r\sigma_k^\star\mathbf{u}_k^\star\mathbf{v}_k^\star;\;\;\mathbf{Y}^{\star(r)} = \sum_{k=1}^{\min\{r,r^\star\}}\hat{\sigma}_k\hat{\mathbf{u}}_k\hat{\mathbf{v}}_k^\top
\]
Then we can decompose the error $\left\|\mathbf{W}^\star\mathbf{X} - \sum_{k=1}^r\mathbf{b}_k\mathbf{a}_k^\top\mathbf{X}\right\|_F$ as
\[
    \left\|\mathbf{W}^\star\mathbf{X} - \sum_{k=1}^r\mathbf{b}_k\mathbf{a}_k^\top\mathbf{X}\right\|_F \leq \left\|\mathbf{W}^\star\mathbf{X} - \mathbf{Y}^{\star(r)}\right\|_F + \left\|\mathbf{Y}^{\star(r)} - \mathbf{Y}^{(r)}\right\|_F +\left\|\mathbf{Y}^{(r)} - \sum_{k=1}^r\mathbf{b}_k\mathbf{a}_k^\top\mathbf{X}\right\|_F
\]
We will analyze each of the three terms individually. To start, for the first term, we notice that $\mathbf{Y}^{\star(r)}$ is precisely the truncated SVD of $\mathbf{W}^\star\mathbf{X}$ when $r < r^\star$. Therefore
\[
    \left\|\mathbf{W}^\star\mathbf{X} - \mathbf{Y}^{\star(r)}\right\|_F \leq \sum_{k=r+1}^{r^\star}\sigma_k\left(\mathbf{W}^\star\mathbf{X}\right) \leq \sigma_{\max}\left(\mathbf{X}\right)\sum_{k=r+1}^{r^\star}\sigma_r\left(\mathbf{W}^\star\right)
\]

For the second term, by Lemma~\ref{lem:noise_comp}, we have that with probability at least $1-\gamma$
\[
    \left\|\mathbf{Y}^{\star(r)} - \mathbf{Y}^{(r)}\right\|_F \leq O\left(\varepsilon\sqrt{n\log\sfrac{1}{\gamma}}\left(r+\sqrt{\frac
        {\min\{r^\star,r\}}{\mathcal{T}_{\min}^\star}}\right)\right)
\]
Lastly, by Theorem~\ref{thm:main_theorem_1_gen}, we have that
\[
    \left\|\mathbf{Y}^{(r)}-\sum_{k=1}^r \mathbf{b}_k \mathbf{a}_k^ \top \mathbf{X}\right\|_F
     \leq \sigma_{\max}(\mathbf{X})\sum_{k=1}^{r}\sum_{k'=0}^{k} \left\|\mathbf{\bm{\delta}}_{k'}\right\|_F \prod_{j = k'+1}^{k}\left(2 + \frac{6\sigma_j^\star}{\mathcal{T}_k^\star} \right)
\]
Combining the above equations, we have that
\begin{align*}
    \left\|\mathbf{W}^\star\mathbf{X} - \sum_{k=1}^r \mathbf{b}_k \mathbf{a}_k^ \top \mathbf{X}\right\|_F & \leq \sigma_{\max}\left(\mathbf{X}\right)\sum_{k=r+1}^{r^\star}\sigma_r\left(\mathbf{W}^\star\right) + \sigma_{\max}(\mathbf{X})\sum_{k=1}^{r}\sum_{k'=0}^{k} \left\|\mathbf{\bm{\delta}}_{k'}\right\|_F \prod_{j = k'+1}^{k}\left(2 + \frac{6\sigma_j^\star}{\mathcal{T}_k^\star} \right)\\
    &\quad\quad\quad + O\left(\kappa\sqrt{n\log\sfrac{1}{\gamma}}\left(r+\sqrt{\frac
        {\min\{r^\star,r\}}{\mathcal{T}_{\min}^\star}}\right)\right)
\end{align*}
This gives that
\begin{align*}
    \left\|\mathbf{W}^\star - \sum_{k=1}^r \mathbf{b}_k \mathbf{a}_k^ \top\right\|_F &\leq \kappa(\mathbf{X})\left(\sum_{k=r+1}^{r^\star}\sigma_r\left(\mathbf{W}^\star\right) + \sum_{k=1}^{r}\sum_{k'=0}^{k} \left\|\mathbf{\bm{\delta}}_{k'}\right\|_F \prod_{j = k'+1}^{k}\left(2 + \frac{6\sigma_j^\star}{\mathcal{T}_k^\star} \right)\right)\\
    &\quad\quad\quad + O\left(\frac{\varepsilon\sqrt{n\log\sfrac{1}{\gamma}}}{\sigma_{\min}(\mathbf{X})}\left(r+\sqrt{\frac
        {\min\{r^\star,r\}}{\mathcal{T}_{\min}^\star}}\right)\right)
\end{align*}
which completes the proof.

\section{Supporting theorems and lemmas}
\label{sec:auxiliary_lem}
\begin{lemma}
    \label{lem:prod_sum_seq}
    Consider a sequence of quantities $\left\{Q_k\right\}_{k=1}^\infty$ satisfying
    \[
        Q_{k+1} = a_kQ_k + b_k
    \]
    with some $a_k, b_k \geq 0$ for all $k \in \Z^+$. Set $b_0 = Q_1$. Then we have that
    \[
        Q_k = \sum_{k'=0}^{k-1}b_{k'}\prod_{j=k'+1}^{k-1}a_j
    \]
\end{lemma}
\begin{proof}
    We shall prove by induction. For the base case, let $k= 1$. In this case, we have that
    \[
        Q_1 = \sum_{k'=0}^0b_{k'}\prod_{j=k'+1}^{0}a_j = b_0 = Q_1
    \]
    For the inductive case, assume that the property holds for $k$. Then we have that
    \[
        Q_{k+1} = a_kQ_k+ b_k = a_k\cdot \sum_{k'=0}^{k-1}b_{k'}\prod_{j=k'+1}^{k-1}a_j + b_k = \sum_{k'=0}^{k}b_{k'}\prod_{j=k'+1}^{k}a_j
    \]
    This proves the inductive step and finishes the proof.
\end{proof}

\begin{lemma}
    \label{lem:gaus_concentration}
    Let $\mathbf{M}\in\mathbb{R}^{m\times n}$ be a matrix containing I.I.D. Gaussian entries from $\mathcal{N}(0,1)$. Then we have that with probability at least $1 - \delta$, the following holds
    \begin{itemize}
        \item $\sigma_{\max}(\mathbf{M}) \leq O\left(\sqrt{m} + \sqrt{n} + \sqrt{\log\sfrac{1}{\delta}}\right)$
        \item $\left\|\mathbf{M}\right\|_F \leq O\left(\sqrt{mn\log\sfrac{1}{\delta}}\right)$
    \end{itemize}
\end{lemma}
\begin{proof}
    By standard results of Gaussian random matrices and vectors, we have that
    \begin{itemize}
        \item $\mathbb{P}\left(\sigma_{\max}(\mathbf{M})\leq O\left(\sqrt{m} + \sqrt{n} + t_1\right)\right) \geq 1 - \exp\left(-t_1^2\right)$
        \item $\mathbb{P}\left(\left\|\mathbf{M}\right\|_F \leq t_2\right) \geq 1 - 2\exp\left(-\frac{t_2^2}{2mn}\right)$
    \end{itemize}
    Take $t_1 = \sqrt{\log\frac{2}{\delta}}$ and $t_2 = \sqrt{2mn\log\frac{4}{\delta}}$ finishes the proof.
\end{proof}

\begin{theorem}[Eckart-Young-Mirsky Theorem] \label{Eckart-Young-Mirsky}
Let \( \mathbf{A} \in \mathbb{R}^{m \times n} \) be a matrix with singular value decomposition \( \mathbf{A} = \mathbf{U} \mathbf{\Sigma}\mathbf{V}^\top \), where \( \mathbf{U} \) and \( \mathbf{V} \) are orthogonal matrices and \( \mathbf{\Sigma} \) is a diagonal matrix with singular values \( \sigma_1 \geq \sigma_2 \geq \dots \geq \sigma_{\min(m, n)} \geq 0 \). For any integer \( k \leq \min(m, n) \), let \( \mathbf{A}_k = \mathbf{U}_k \mathbf{\Sigma}_k \mathbf{V}_k^\top \) be the best rank-\( k \) approximation of \( \mathbf{A} \), where \( \mathbf{U}_k \) and \( \mathbf{V}_k \) consist of the first \( k \) columns of \( \mathbf{U} \) and \( \mathbf{V} \), and \( \mathbf{\Sigma}_k \) is the diagonal matrix of the largest \( k \) singular values.

Then \( \mathbf{A}_k \) minimizes the approximation error in both the Frobenius norm and the spectral norm:
\[
\mathbf{A}_k = \argmin_{\mathbf{B}, \, \text{rank}(\mathbf{B}) = k} \| \mathbf{A} - \mathbf{B} \|_F 
\]
\end{theorem}

\begin{theorem}[Wedin Theorem (\cite{10.1007/BF01932678})]
Let \( \mathbf{M} \), \( \tilde{\mathbf{M}} \in \mathbb{R}^{m \times n} \) be two matrices with rank-\( r \) SVDs:
\begin{equation*}
\mathbf{M} = 
\begin{bmatrix} 
\mathbf{U}_1 & \mathbf{U}_2 
\end{bmatrix} 
\begin{bmatrix} 
\mathbf{\Sigma}_1 & 0 \\ 
0 & \mathbf{\Sigma}_2 
\end{bmatrix}
\begin{bmatrix} 
\mathbf{V_1}^\top \\ 
\mathbf{V_2}^\top 
\end{bmatrix}, \quad \text{and} \quad
\tilde{\mathbf{M}} = \mathbf{M} + \mathbf{\Delta} = 
\begin{bmatrix} 
\tilde{\mathbf{U}}_1 & \tilde{\mathbf{U}}_2 
\end{bmatrix} 
\begin{bmatrix} 
\tilde{\mathbf{\Sigma}}_1 & 0 \\ 
0 & \tilde{\mathbf{\Sigma}}_2 
\end{bmatrix}
\begin{bmatrix} 
\tilde{\mathbf{V}}_1^\top \\ 
\tilde{\mathbf{V}}_2^\top 
\end{bmatrix}.
\end{equation*}

If $\delta = \min \left\{ \min_{1 \leq i \leq r, r+1 \leq j \leq n} |\sigma_i - \tilde{\sigma}_j|, \min_{1 \leq i \leq r} \sigma_i \right\} > 0$, then:

\begin{equation*}
\left\| \sin 
\theta(\tilde{\mathbf{U}}_1, \mathbf{U}_1) 
\right\|_F^2 + \left\| \sin \theta(\tilde{\mathbf{V}}_1, \mathbf{V}_1) \right\|_F^2 
\leq \frac{\left\| \mathbf{U}_1^\top \mathbf{\Delta} \right\|_F^2 + \left\| \mathbf{\Delta} \mathbf{V}_1 \right\|_F^2}{\delta^2}
\end{equation*}
\end{theorem}
\begin{theorem}[Weyl's Theorem for Singular Values(\cite{Weyl1912})]
Let \( \mathbf{M} \) and \( \bm{\Delta} \) be \( m \times n \) matrices. If \( \tilde{\mathbf{M}} = \mathbf{M} + \bm{
\Delta} \), then the singular values \(\sigma_i\) of \( \mathbf{M} \) and the singular values \(\tilde{\sigma}_i\) of \(\tilde{\mathbf{M}}\) satisfy the following inequality for all \(i = 1, 2, \ldots, \min(m, n)\):
\[
\left|\tilde{\sigma}_i - \sigma_i \right|\leq \left\|\bm{\Delta}\right\|_2,
\]
\end{theorem}

\section{Experimental analysis on linear matrix regression.}\label{app:3}

We present experiments that validate our theory on error propagation in sequential rank-1 learning. 
Our experiments aim to demonstrate how the distribution of computational resources across rank-1 components affects the overall approximation quality, particularly focusing on how errors in early components propagate to later stages of the sequential learning process.

\textbf{Problem setting.}
Per our theory, we consider the low-rank linear regression problem of finding $\mathbf{W}^\star \in \mathbb{R}^{m \times n}$ with rank $\leq r$ such that $\mathbf{Y} = \mathbf{W}^\star \mathbf{X} + \bm{\mathcal{E}}$ where $\bm{\mathcal{E}}$ is the noise term.  
This corresponds to finding a low-rank approximation of $\mathbf{W}^\star$.
We investigate the following settings: \begin{enumerate}[leftmargin=*]
%\item \textit{Matrix Sizes and Ranks:} We experiment with different matrix dimensions and target ranks to understand how these parameters affect the sequential learning process and error propagation. \vspace{-0.15cm}
\item \textit{Singular Value Profiles:} We vary the singular value distribution of $\mathbf{W}^\star$ to analyze how the spectrum of ground truth influences error propagation. \vspace{-0.1cm}
\item \textit{Noise Variations:} We introduce different types and levels of noise to assess the robustness of sequential rank-1 learning to perturbations. \vspace{-0.1cm}
\item \textit{Iteration allocation strategies:} We evaluate three different iteration allocation strategies: 
\vspace{-0.1cm}
\begin{enumerate}[leftmargin=*]
\item \textbf{Equal:} Same number of optimization iterations to each rank-1 component. \vspace{-0cm}
\item \textbf{More First:} More iterations allocated to the earlier components and fewer to later ones. \vspace{-0cm}
\item \textbf{Less First:} Fewer iterations allocated to the earlier components and more to later ones. \vspace{-0cm}
\end{enumerate}
\end{enumerate}

To ensure statistical robustness, all experiments are repeated across 5 independent trials. 
We report the mean performance across these trials, and visualize variability using shaded bands that represent the standard deviation.

We consider matrix dimensions $\mathbf{W}^\star \in \mathbb{R}^{500 \times 1000}$; we observed that experiments varying the dimensions of $\mathbf{W}^\star$ do not introduce any additional value to the main messages of this section. 
We generate $\mathbf{W}^\star$ with different singular value profiles, as in:
\begin{itemize}[leftmargin=*]
\item \textbf{Uniform:} $\sigma_i = 10$ for all $i = 1, \ldots, r^\star$;
\item \textbf{Exponential decay:} $\sigma_i = 100 \cdot \left( \frac{1}{100} \right)^{\frac{i - 1}{r^\star - 1}}$ for $i = 1, \ldots, r^\star$;
\item \textbf{Power-law decay:} $\sigma_i = \frac{100}{i^2}$ for $i = 1, \ldots, r^\star$;
\end{itemize}
where $r^\star$ is the true rank of $\mathbf{W}^\star$. 
Without loss of generality, we fix the rank $r^\star$ to $20$ as we did not observe unexpected behaviors in the performance of the algorithms for different rank values.

We also consider several noise scenarios to evaluate robustness: $i)$ Noiseless; $ii)$ Gaussian where $\bm{\mathcal{E}}$ has i.i.d. entries from $\mathcal{N}(0,\kappa)$ with $\kappa \in \{0.01, 0.05, 0.1\}$, and $iii)$ Sparse where $\bm{\mathcal{E}}$ is a sparse matrix (in our case 5\% of entries are non-zeros) with non-zero entries from $\mathcal{N}(0,\kappa)$ with $\kappa \in \{1, 10\}$.
Per our theory, \(\mathbf{X}\) is sampled from a standard Gaussian distribution, \(\mathcal{N}(0, 1)\).

%\subsection{Results and Analysis}

\textbf{Effect of singular value profile.}
We study whether the singular value profile of $\mathbf{W}^\star$ has impact on error propagation through the term $\mathcal{T}_k^\star = \min\{\min_{j>k} |\sigma_k^\star - \sigma_j^\star|, \sigma_k^\star\}$ in our error bound. Figure~\ref{fig:SVDecay} shows the singular value decay patterns of both $\mathbf{W}^\star$ and the resulting $\mathbf{Y}$ under different spectral profiles. Figure~\ref{fig:SVError} illustrates the training and reconstruction errors under these three profiles.

To ensure a fair comparison across different spectral profiles, we normalize the singular values of $\mathbf{W}^\star$ such that all generated matrices have the same Frobenius norm. This avoids artificially inflating or deflating error magnitudes due to differences in matrix scale rather than the structure of singular value decay.

\begin{figure}[h]
    \centering
    \begin{subfigure}[b]{0.495\textwidth}
        \includegraphics[width=\linewidth]{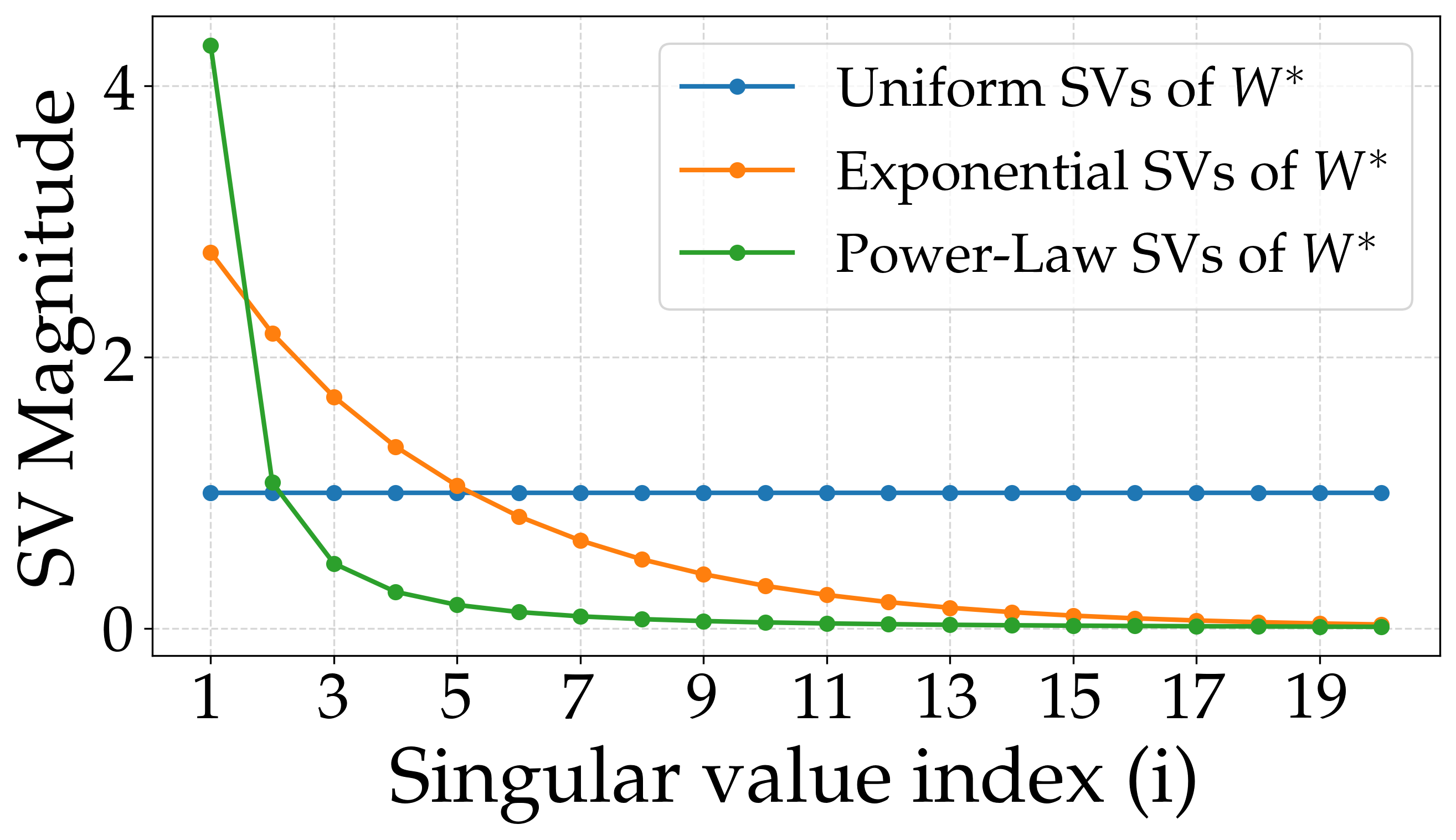}
        \label{fig:SVofW*}
    \end{subfigure}
    \hfill
    \begin{subfigure}[b]{0.495\textwidth}
        \includegraphics[width=\linewidth]{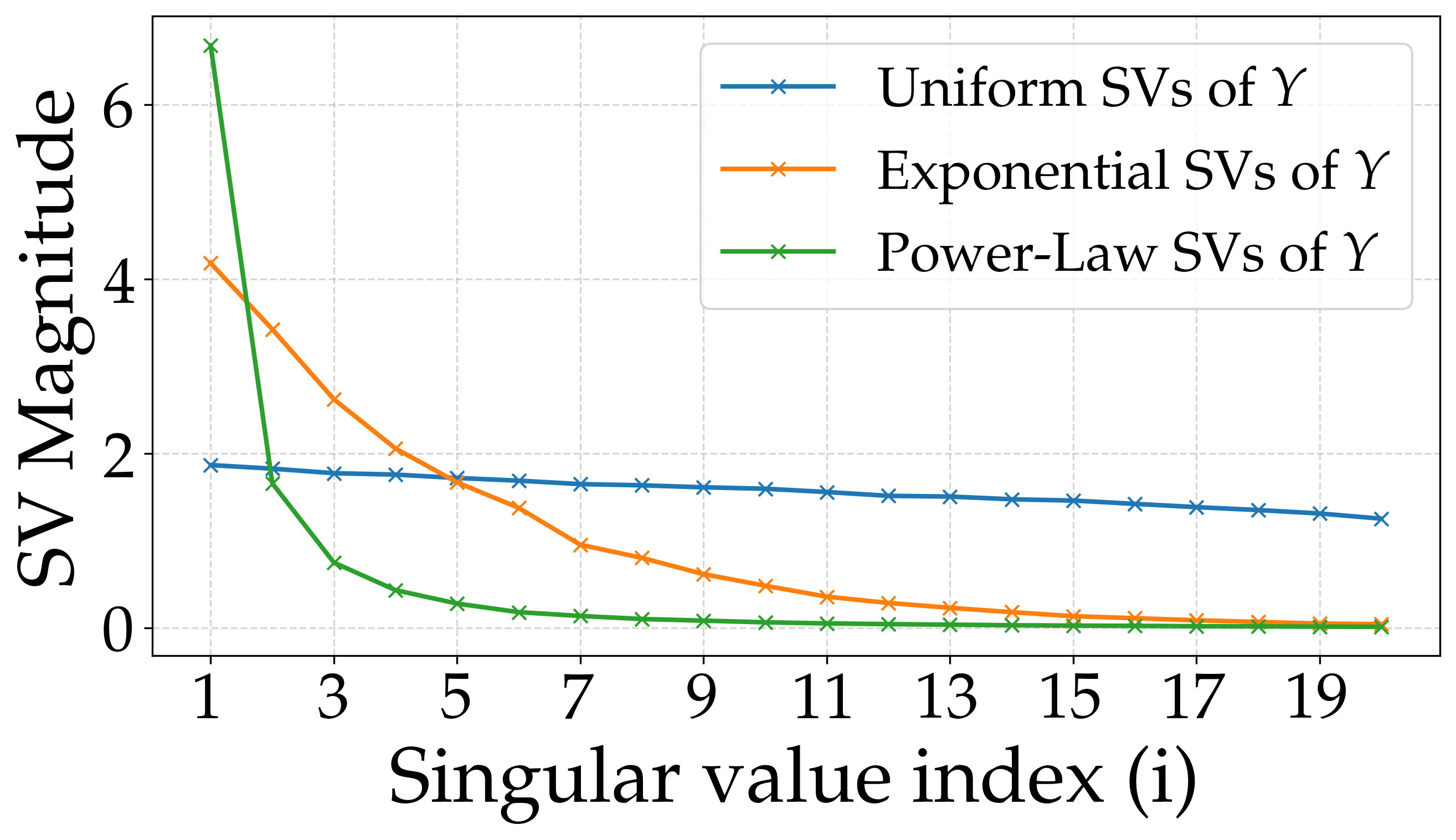}
        \label{fig:SVofY}
    \end{subfigure}
    \caption{Comparison of singular value decay under different profiles. \textit{Left:} Singular values of $\mathbf{W}^\star$. \textit{Right:} Singular values of $\mathbf{Y} = \mathbf{W}^\star \mathbf{X}$.}
    \label{fig:SVDecay}
\end{figure}

\begin{figure}[h]
    \centering
    \begin{subfigure}[b]{0.495\textwidth}
        \includegraphics[width=\linewidth]{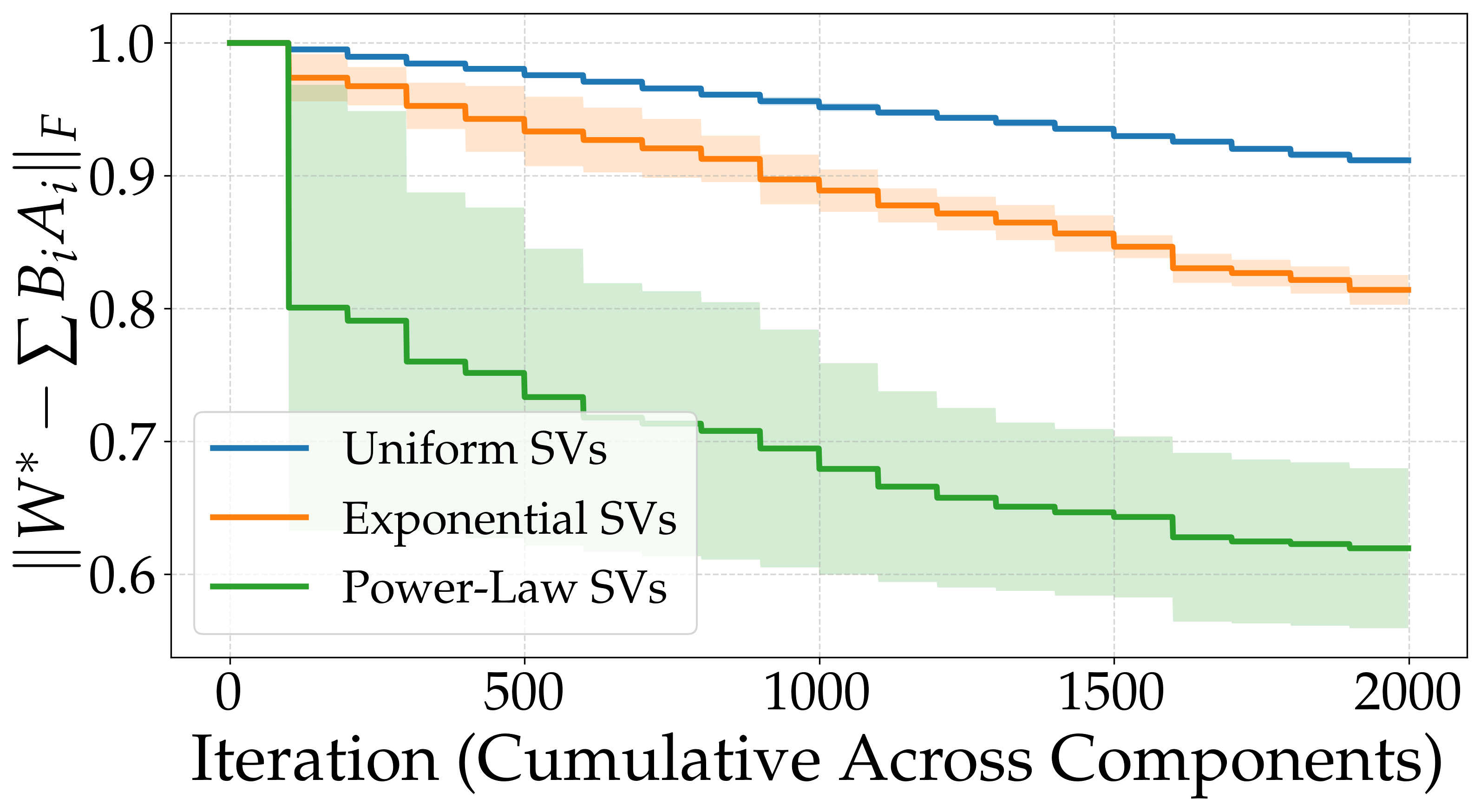}\label{fig:sv_reconstruction_error_comparison}
    \end{subfigure}
    \hfill
    \begin{subfigure}[b]{0.495\textwidth}
        \includegraphics[width=\linewidth]{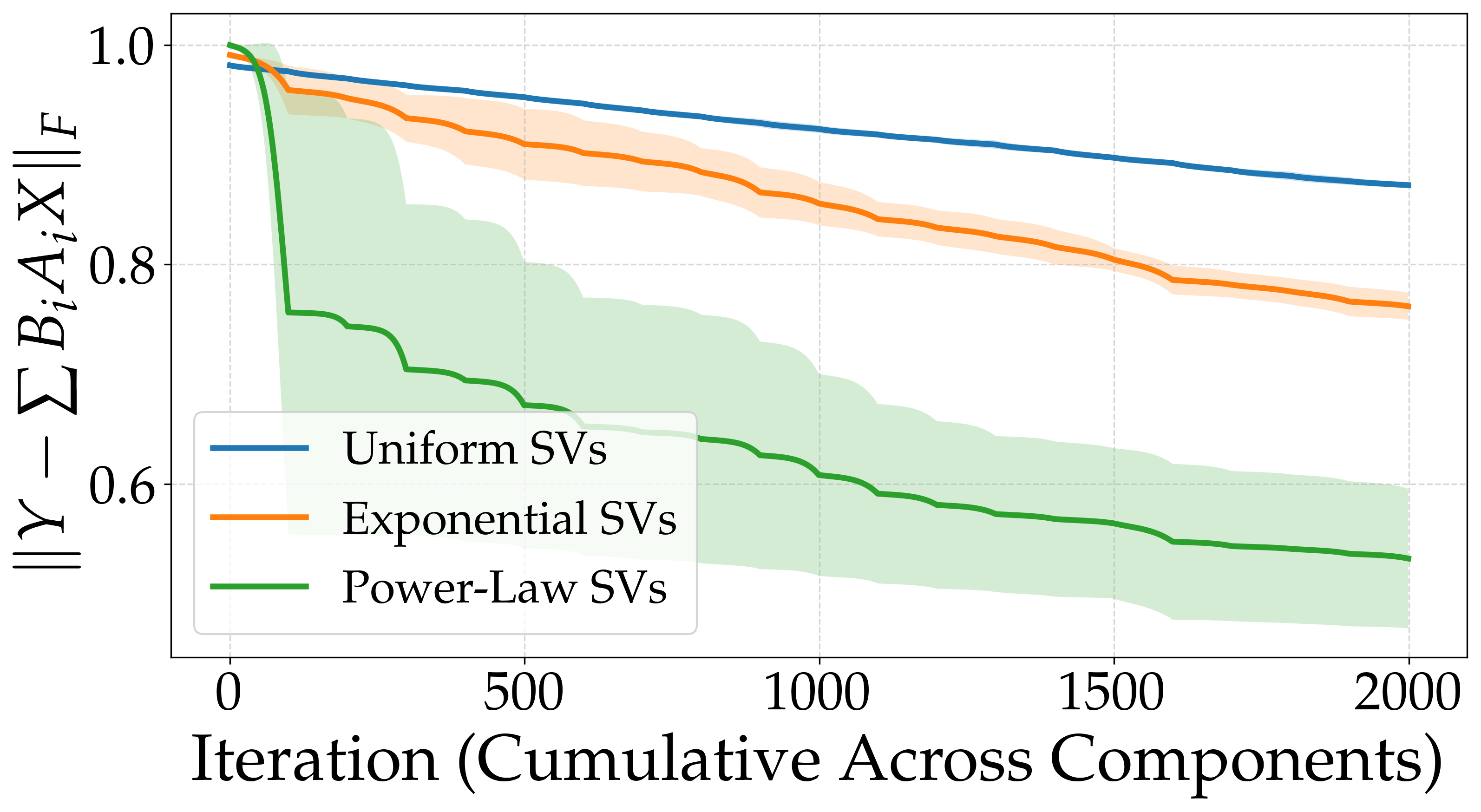}\label{fig:sv_training_error_comparison}
    \end{subfigure}
    \caption{Effect of singular value profile on sequential learning performance. \textit{Left}: $\mathbf{W}^\star$ reconstruction error. \textit{Right}: Objective's training error.}
    \label{fig:SVError}
\end{figure}

\textit{Observations:} The power-law decay profile shows the best performance, followed by the exponential decay, with the uniform profile performing worst.
This matches the theoretical insight that large singular value gaps reduce the compounding of downstream error.
Notably, power-law decay starts steep at the head—its first few singular values are significantly larger—creating large gaps for early components.
In contrast, exponential decay is smoother initially and decays more evenly.
Uniform singular values exhibit no decay, leading to minimal or zero gaps throughout.

\textbf{Impact of noise level.}
Our theoretical analysis extends to noisy settings through Theorem~\ref{thm:noisy_gen}, which characterizes how additive noise impacts generalization performance. Figure~\ref{fig:noise} illustrates the effect of increasing noise levels $\kappa$ on both the training and reconstruction error under Gaussian and sparse noise levels.

\begin{figure}[h]
    \centering
    \begin{subfigure}[b]{0.495\textwidth}
        \includegraphics[width=\linewidth]{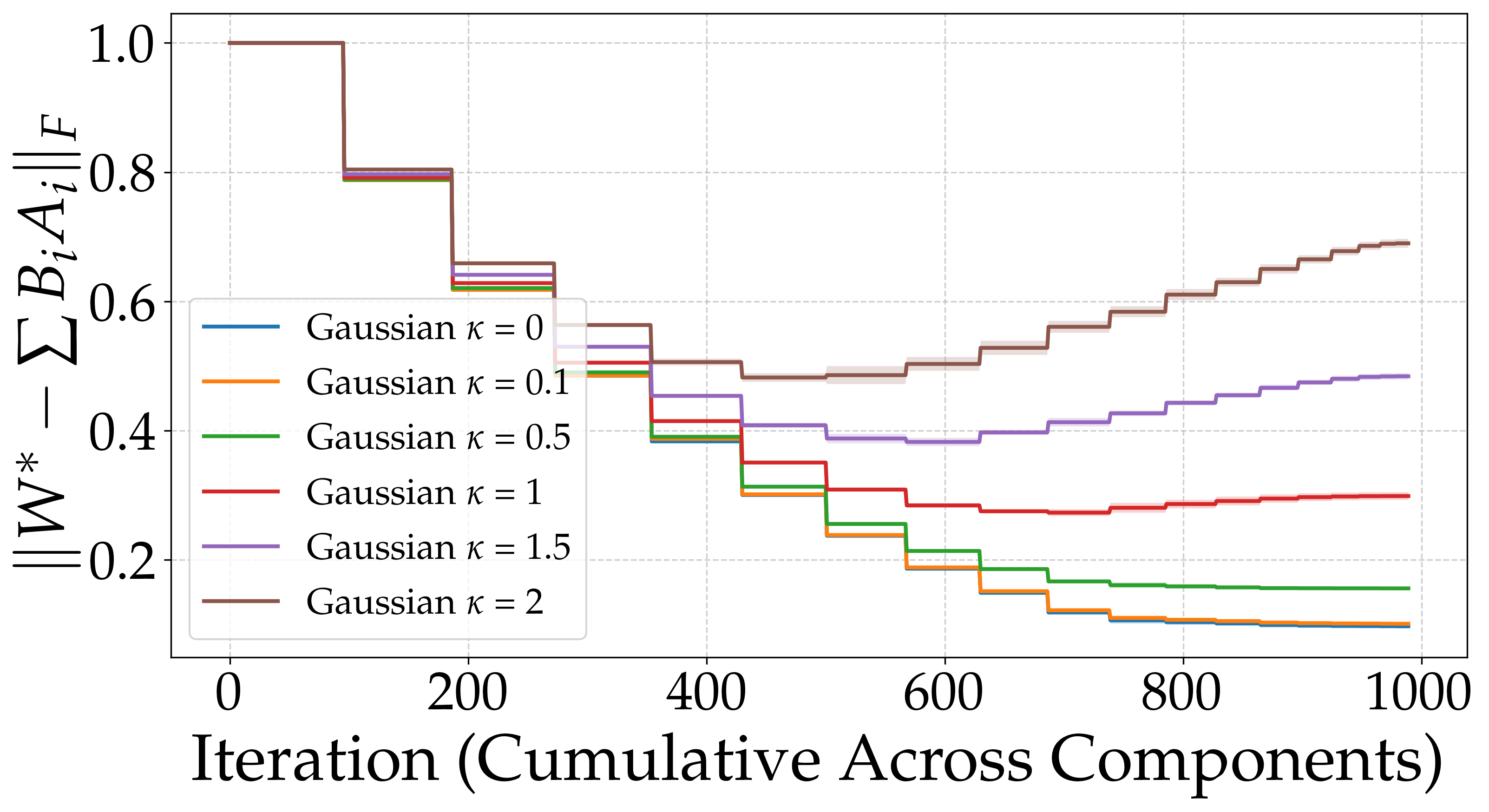}
        \label{fig:Gaussian_noise_effect}
    \end{subfigure}
    \hfill
    \begin{subfigure}[b]{0.495\textwidth}
        \includegraphics[width=\linewidth]{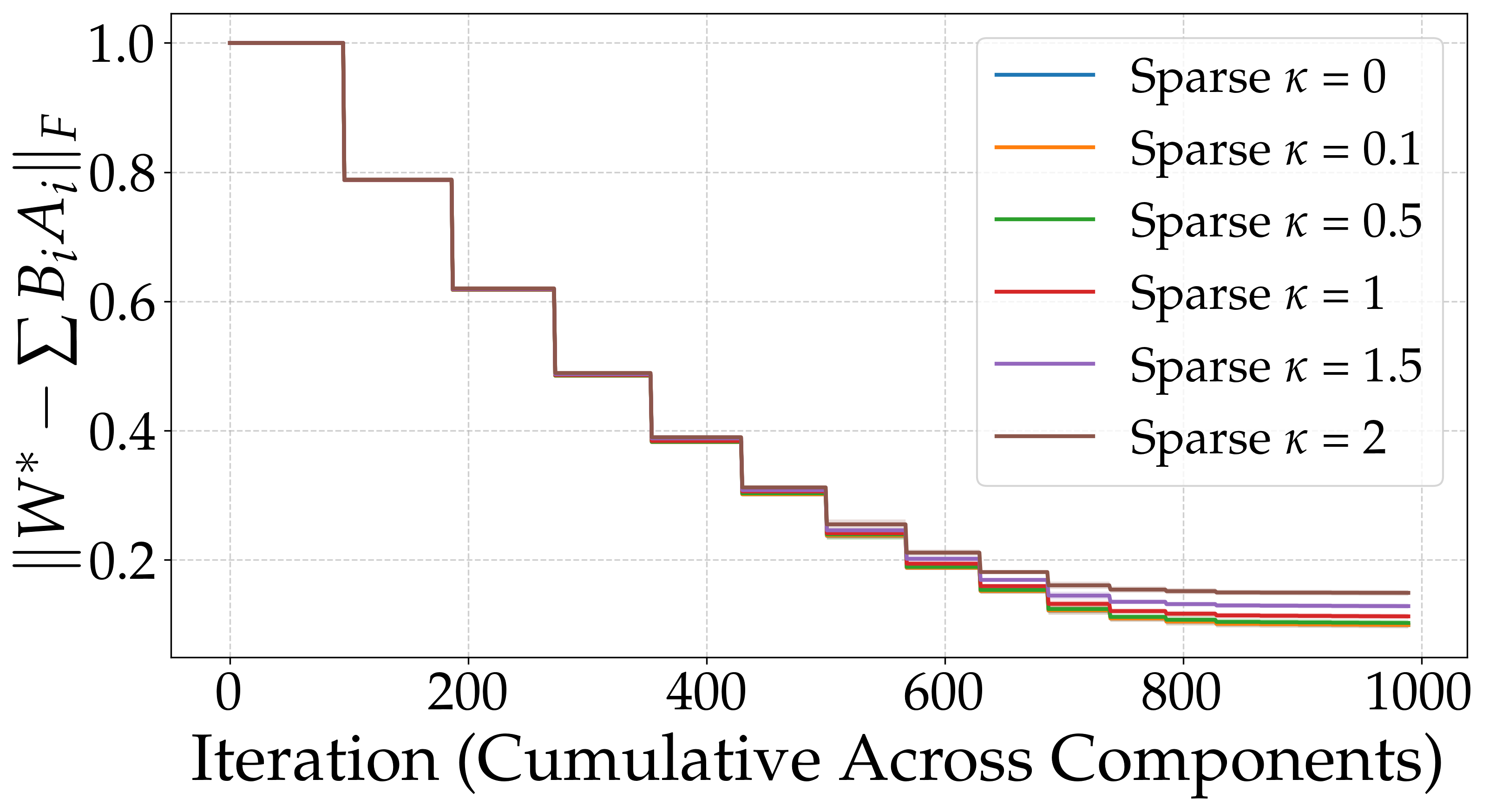}
        \label{fig:Sparse_noise_effect}
    \end{subfigure}
    \caption{Impact of noise level. \textit{Left:} Gaussian noise. \textit{Right:} Sparse noise.}
    \label{fig:noise}
\end{figure}

\textit{Observations:} As expected, increasing the noise level $\kappa$ leads to higher reconstruction error in both Gaussian and sparse settings. Higher noise levels tend to corrupt the smaller singular values of $\mathbf{Y}$, making it difficult to distinguish low-rank structure from noise. This can lead to overfitting in later components of the sequential learner, as the algorithm begins to capture noise rather than signal.

\textbf{Effect of iteration allocation strategies in noisy settings.}
To investigate mitigation strategies, we first evaluate how different iteration allocation strategies perform under noisy conditions. Figure~\ref{fig:allocation_under_noise} shows that the "more-first" strategy consistently outperforms others across varying noise levels by concentrating effort where it matters most—early in the sequence.

\begin{figure}[h]
    \centering
    \begin{subfigure}[b]{0.48\textwidth}
        \includegraphics[width=\linewidth]{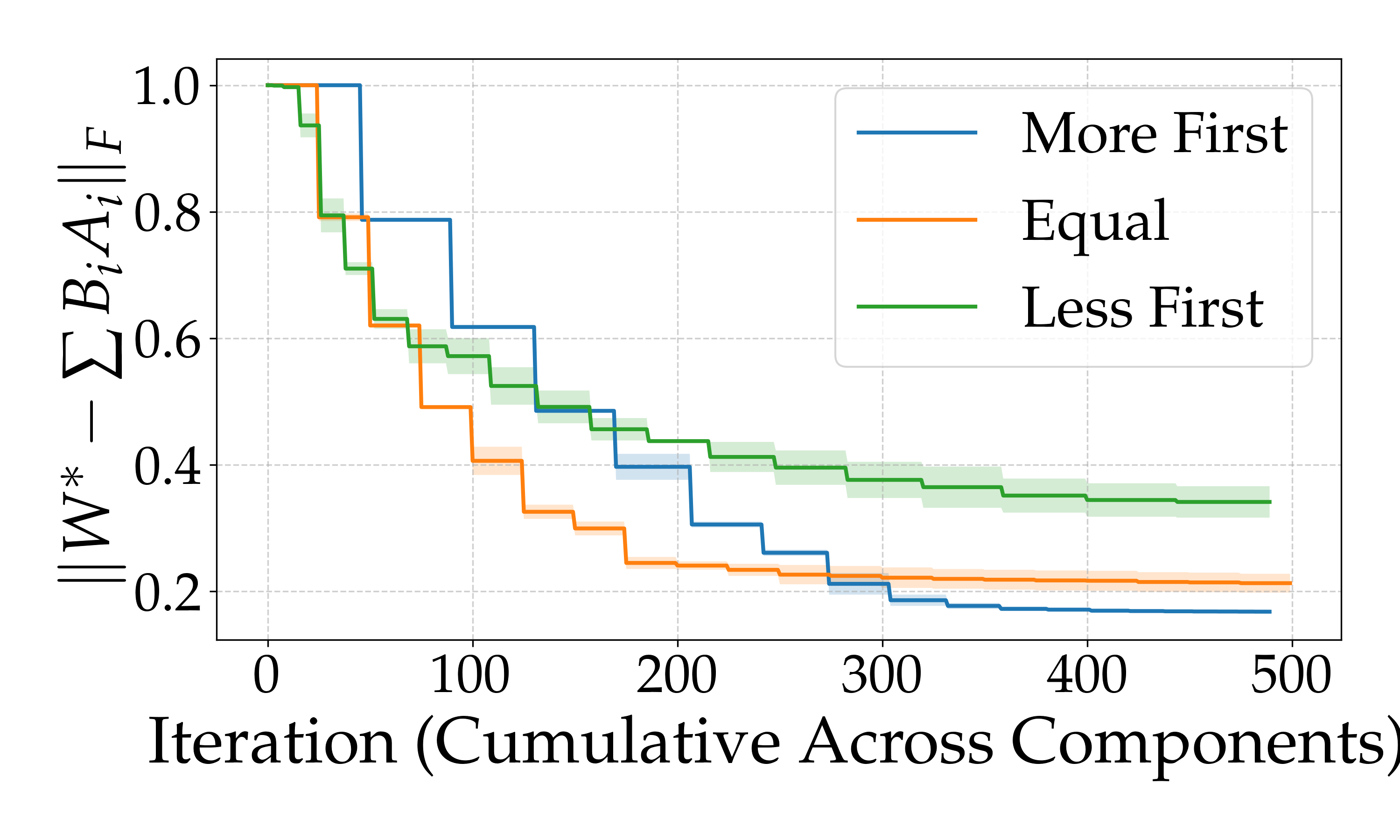}
        \caption{$\kappa = 0.1$}
    \end{subfigure}
    \hfill
    \begin{subfigure}[b]{0.48\textwidth}
        \includegraphics[width=\linewidth]{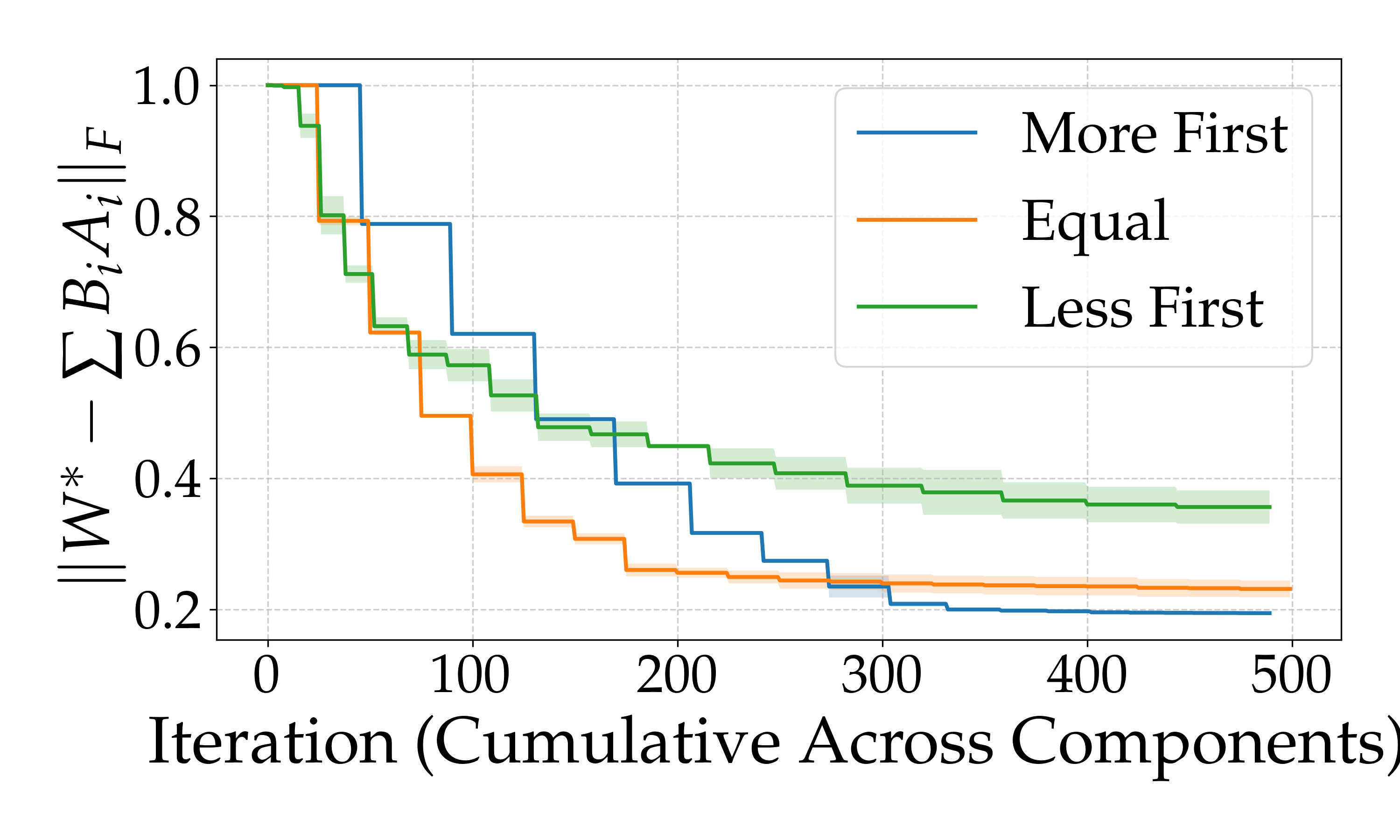}
        \caption{$\kappa = 0.5$}
    \end{subfigure}

    \vspace{0.5em}

    \begin{subfigure}[b]{0.48\textwidth}
        \includegraphics[width=\linewidth]{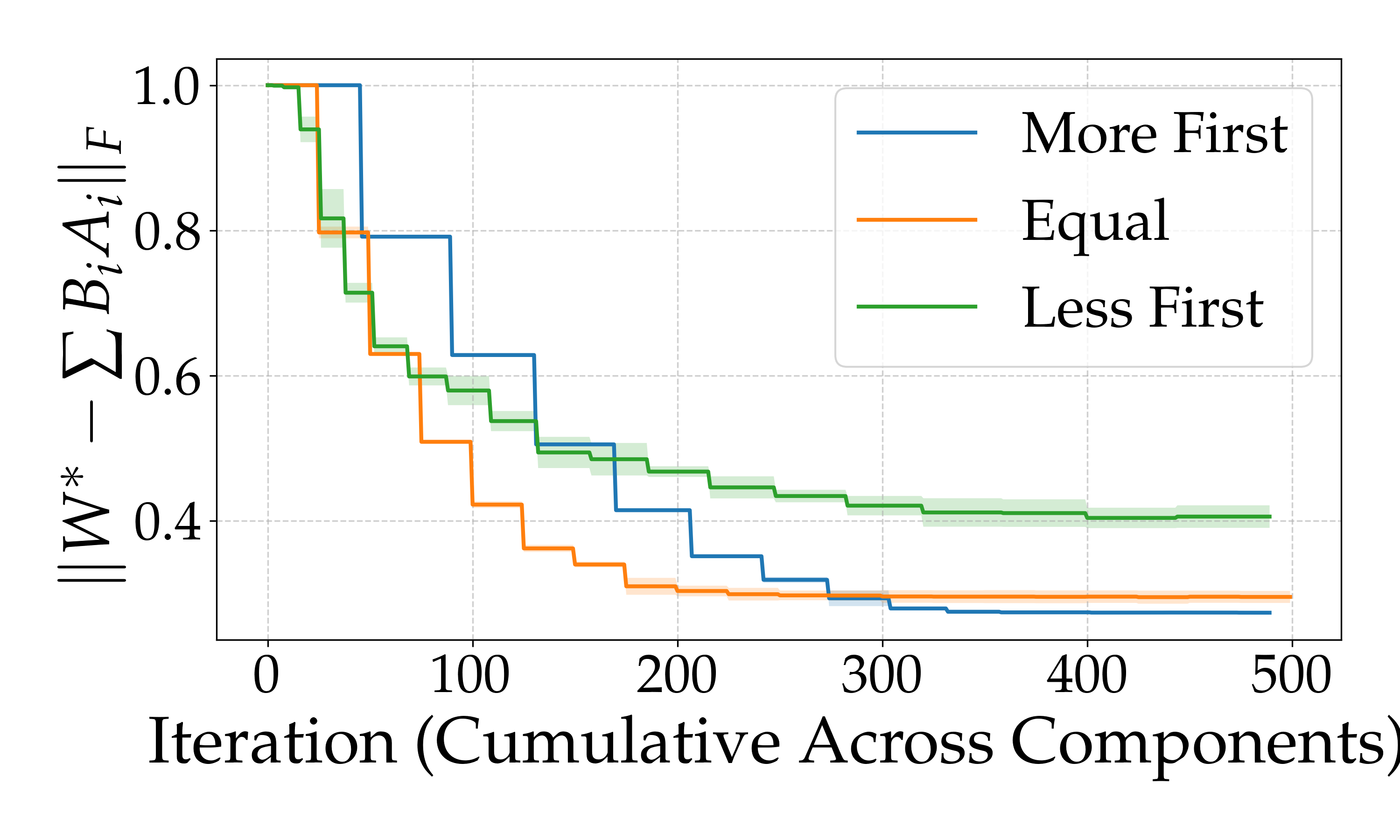}
        \caption{$\kappa = 1$}
    \end{subfigure}
    \hfill
    \begin{subfigure}[b]{0.48\textwidth}
        \includegraphics[width=\linewidth]{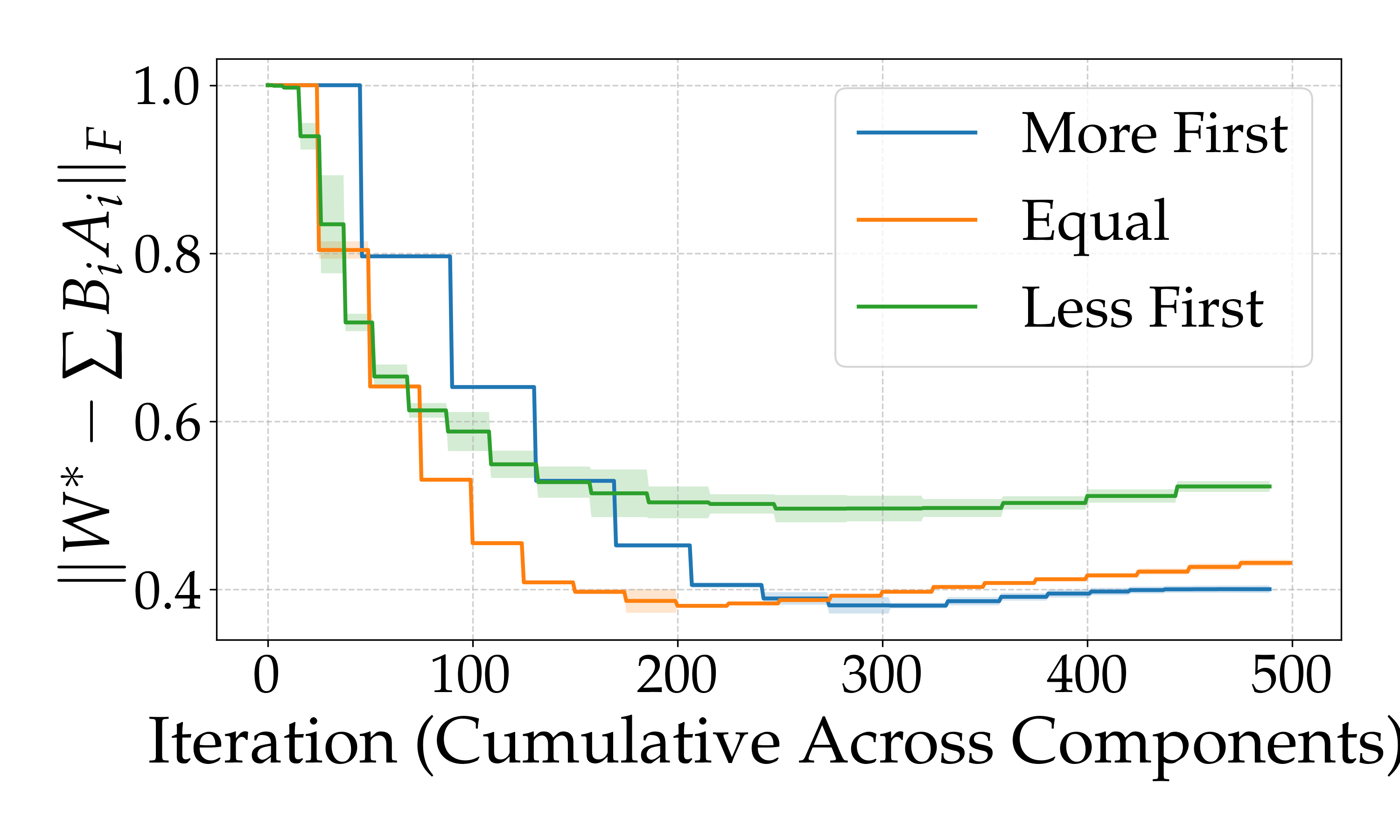}
        \caption{$\kappa = 1.5$}
    \end{subfigure}
    \caption{Comparison of iteration allocation strategies under different noise levels. The "more-first" strategy achieves better reconstruction error across all $\kappa$ values.}
    \label{fig:allocation_under_noise}
\end{figure}
\textit{Observation:} Even in noisy settings, the \textit{more-first} strategy consistently outperforms \textit{equal}, which in turn outperforms \textit{less-first}, across all noise levels~$\kappa$. This highlights the importance of prioritizing early iterations to mitigate error amplification under noise.

\textbf{Effect of singular value profiles in noisy settings.}
We further examine how spectral decay influences robustness under noise. Using the \textit{more-first} allocation strategy, Figure~\ref{fig:svp_under_noise} shows that power-law decay consistently achieves lower reconstruction error compared to exponential and uniform profiles across all noise levels~$\kappa$.
\begin{figure}[h]
    \centering
    \begin{subfigure}[b]{0.48\textwidth}
        \includegraphics[width=\linewidth]{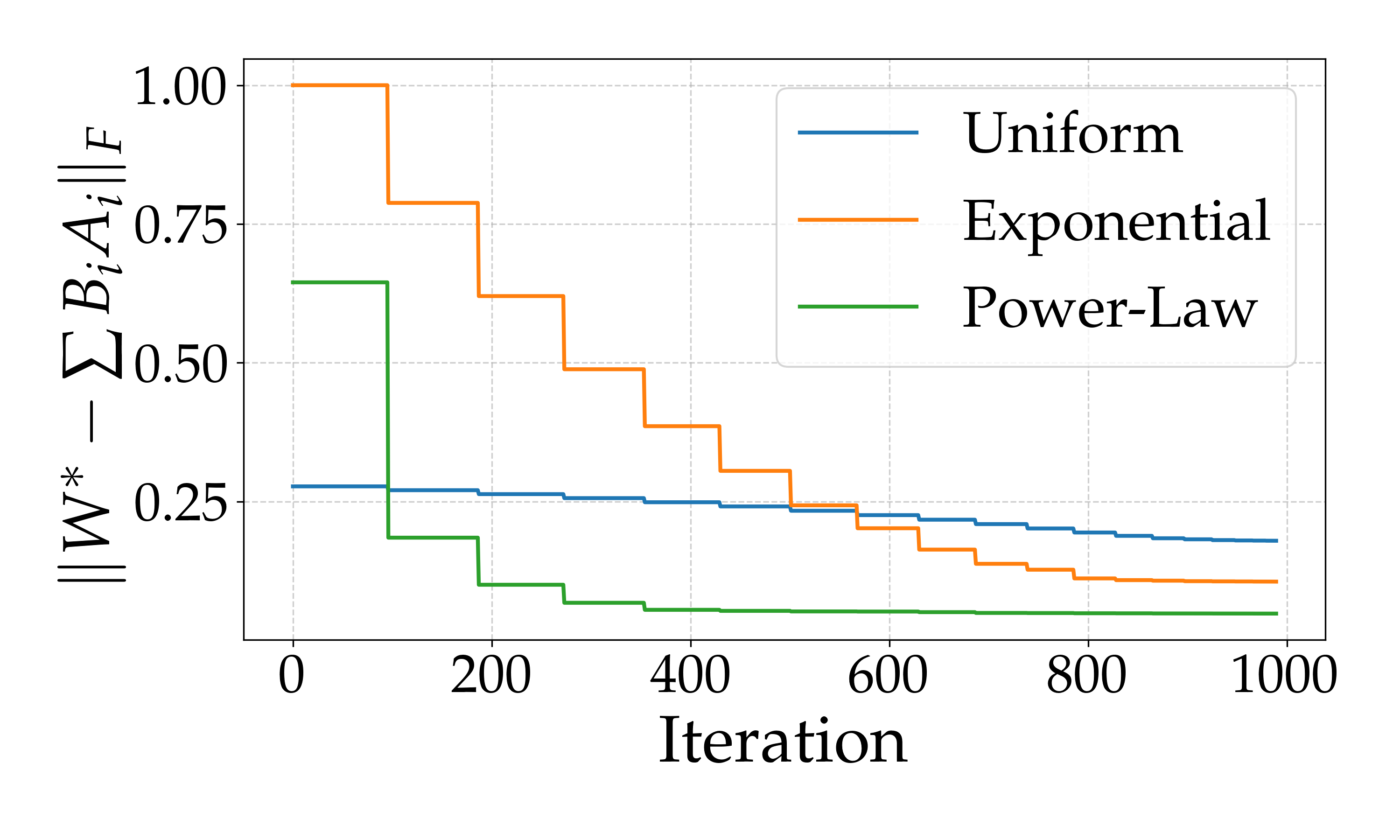}
        \caption{$\kappa = 0.05$}
    \end{subfigure}
    \hfill
    \begin{subfigure}[b]{0.48\textwidth}
        \includegraphics[width=\linewidth]{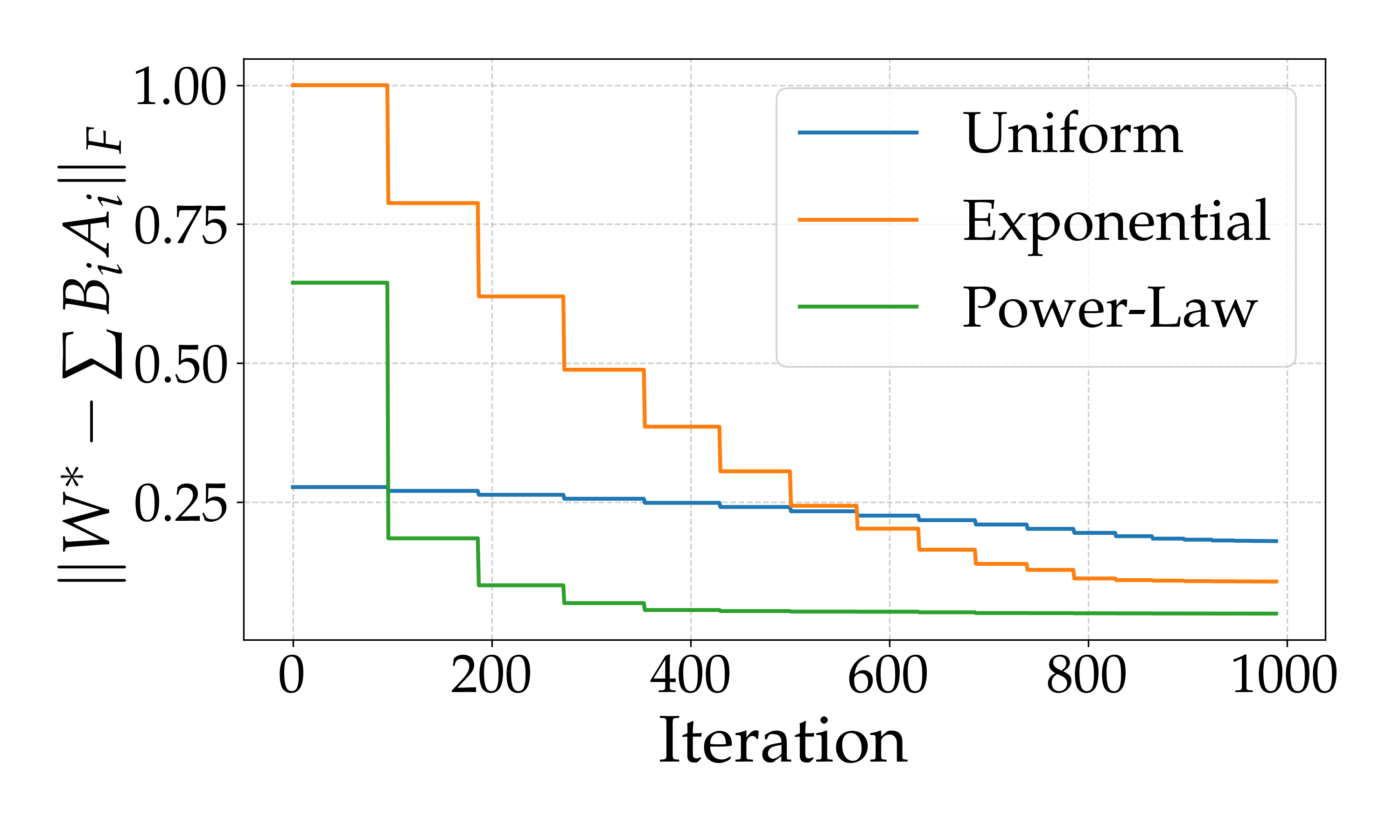}
        \caption{$\kappa = 0.1$}
    \end{subfigure}

    \vspace{0.5em}

    \begin{subfigure}[b]{0.48\textwidth}
        \includegraphics[width=\linewidth]{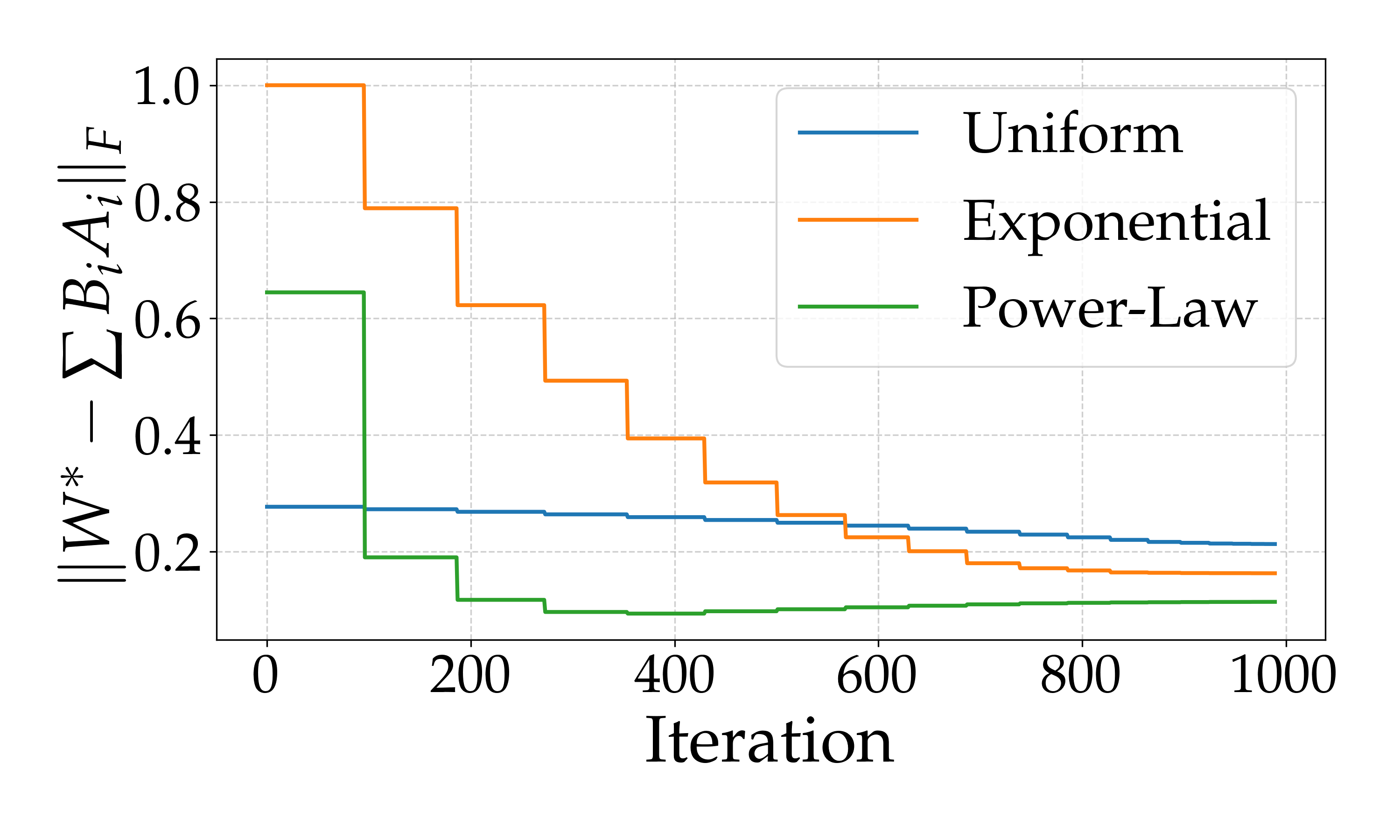}
        \caption{$\kappa = 0.5$}
    \end{subfigure}
    \hfill
    \begin{subfigure}[b]{0.48\textwidth}
        \includegraphics[width=\linewidth]{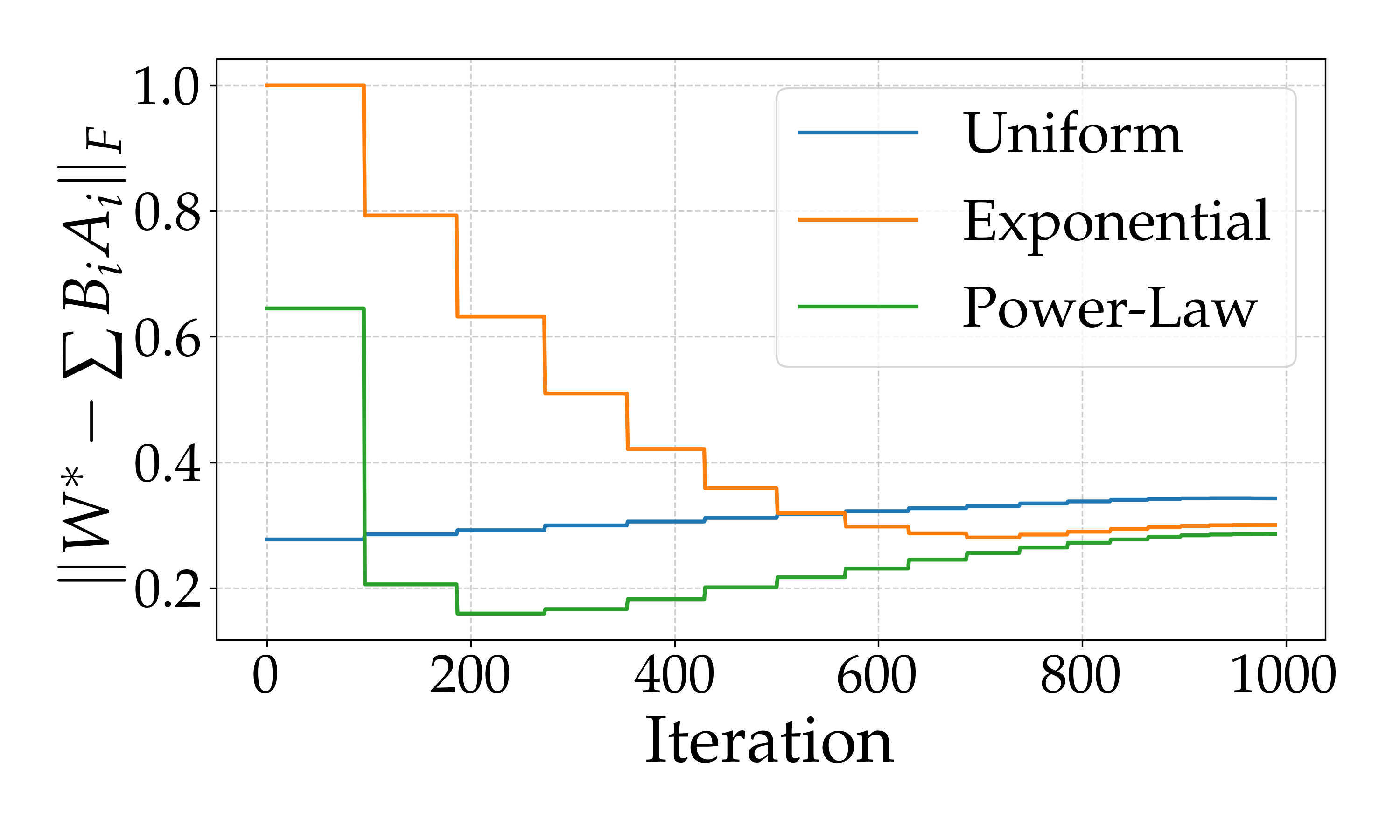}
        \caption{$\kappa = 1$}
    \end{subfigure}
    
    \caption{Effect of singular value profiles under noise. Power-law decay consistently achieves lower reconstruction error, followed by exponential and then uniform profiles, highlighting the benefit of spectral decay even in noisy settings.}
    \label{fig:svp_under_noise}
\end{figure}

\textit{Observation:} Spectral decay plays a critical role in robustness. Power-law decay, with its large leading singular values and wider gaps, allows early components to capture most of the signal, mitigating downstream error propagation. In contrast, uniform profiles lack this protective structure, making them more vulnerable to noise.

\textit{Implications for Practical Use:}
These results suggest that sequential learners can be more robust in the presence of noise by combining two strategies: allocating more iterations to early components and leveraging spectral decay. By front-loading optimization effort where it is impactful—at the beginning of the sequence—and favoring matrices with decaying singular values (especially power-law decay), models maintain lower reconstruction error despite increasing noise levels.

%In the noisy setting, these empirical findings support Theorem~\ref{thm:noisy_gen} where the reconstruction error includes an additive noise-dependent term:

%\begin{equation}
%O\left(\frac{\varepsilon\sqrt{n\log\sfrac{1}{\gamma}}}{\sigma_{\min}(\mathbf{X})}\left(r+\sqrt{\frac
%{\min{r^\star,r}}{\mathcal{T}_{\min}^\star}}\right)\right)
%\end{equation}
%where $\varepsilon$ represents the noise scale. This term increases with noise level, explaining the observed performance degradation.
\begin{comment}
\textit{Implications for Practical Use}:
These results suggest that sequential learners can be more robust in the presence of noise by allocating more iterations to early components. By front-loading optimization effort where it is impactful---at the beginning of the sequence---models can behave better, maintaining lower reconstruction error despite increasing noise levels.
\end{comment}

\textbf{Computational efficiency analysis.}
Beyond approximation quality, we also analyze the computational efficiency of different iteration allocation strategies. Specifically, we investigate how quickly each strategy reduces the reconstruction error to a desired threshold. Figure~\ref{fig:iters_vs_threshold} illustrates, for a range of target error thresholds, the number of iterations required by each allocation strategy to reach that threshold. 

\begin{figure}[h]
    \centering
    \begin{subfigure}[b]{0.48\textwidth}
        \includegraphics[width=\linewidth]{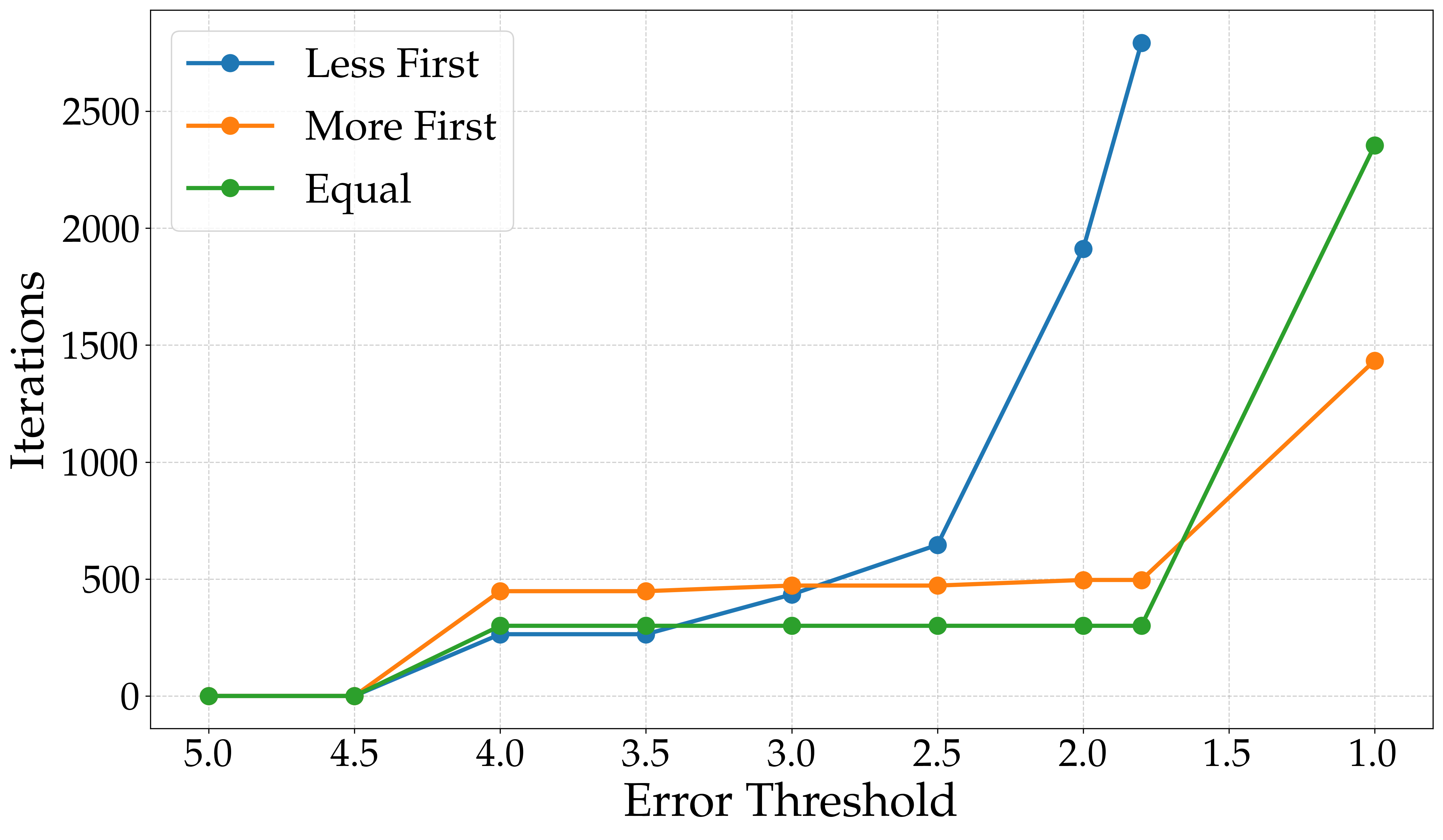}
        \caption{Threshold: 1}
        \label{fig:iters_thresh_1}
    \end{subfigure}
    \hfill
    \begin{subfigure}[b]{0.48\textwidth}
        \includegraphics[width=\linewidth]{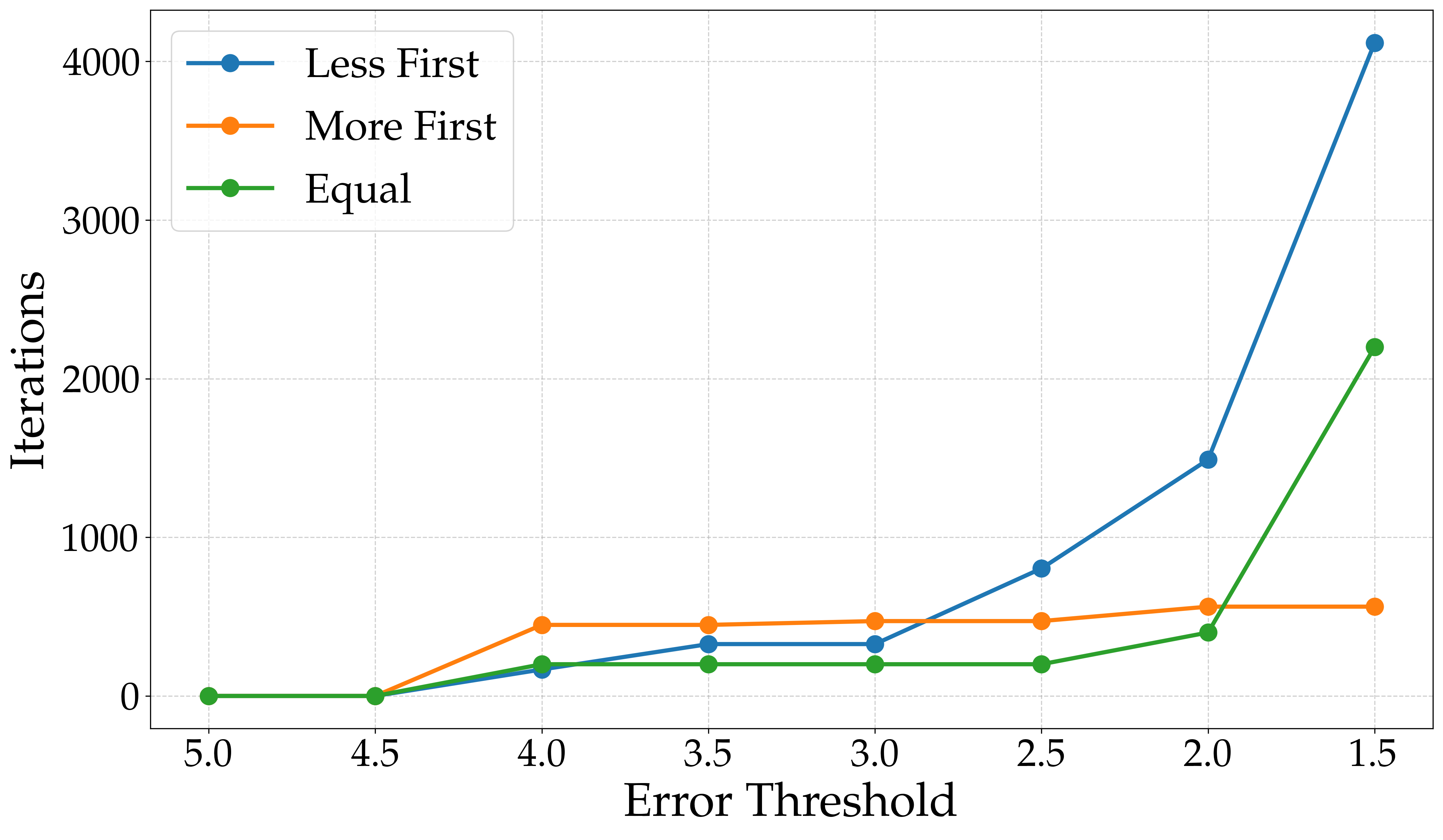}
        \caption{Threshold: 1.5}
        \label{fig:iters_thresh_1.5}
    \end{subfigure}
    
    \vspace{0.5em}
    
    \begin{subfigure}[b]{0.48\textwidth}
        \includegraphics[width=\linewidth]{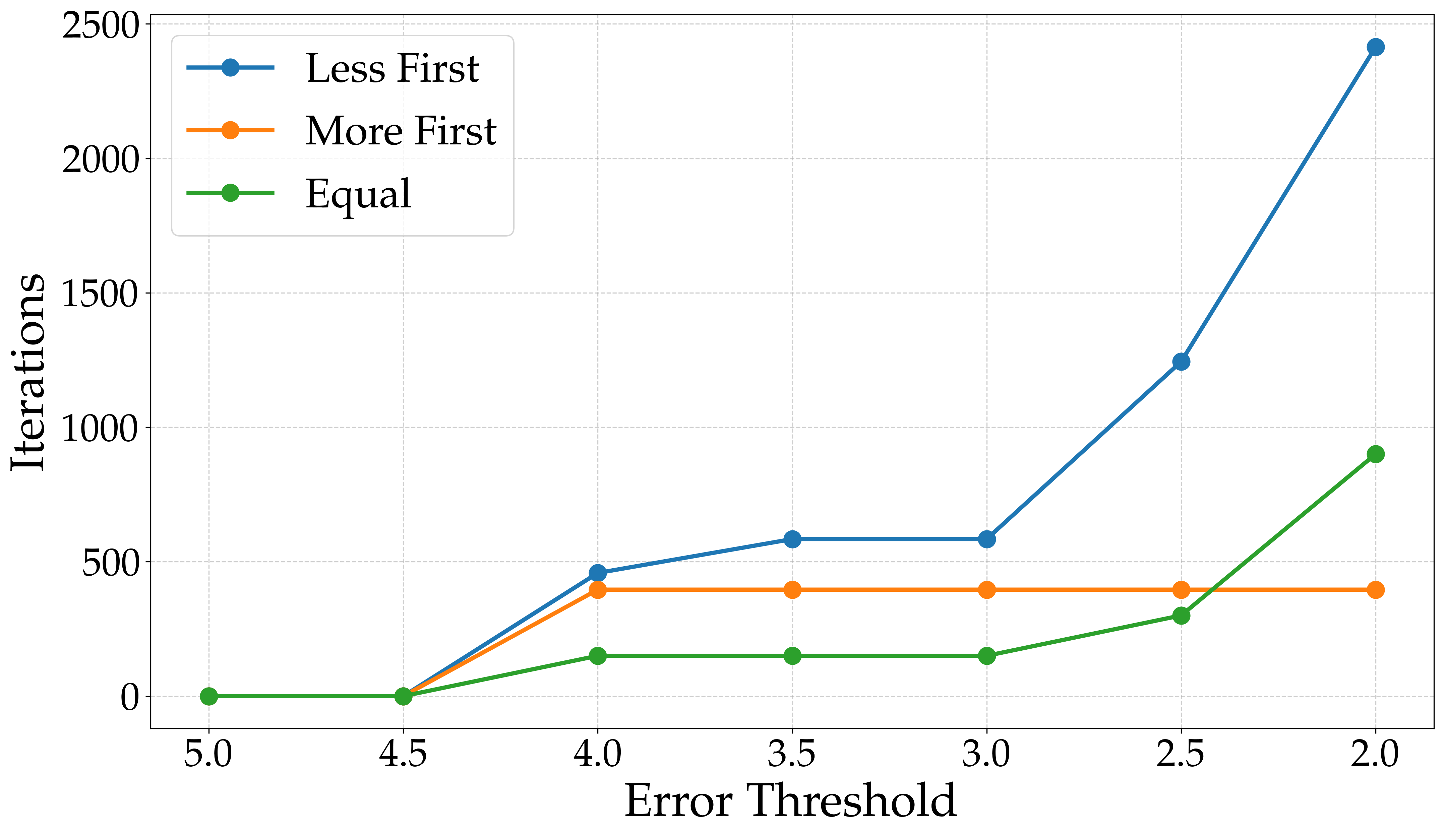}
        \caption{Threshold: 2}
        \label{fig:iters_thresh_2}
    \end{subfigure}
    \hfill
    \begin{subfigure}[b]{0.48\textwidth}
        \includegraphics[width=\linewidth]{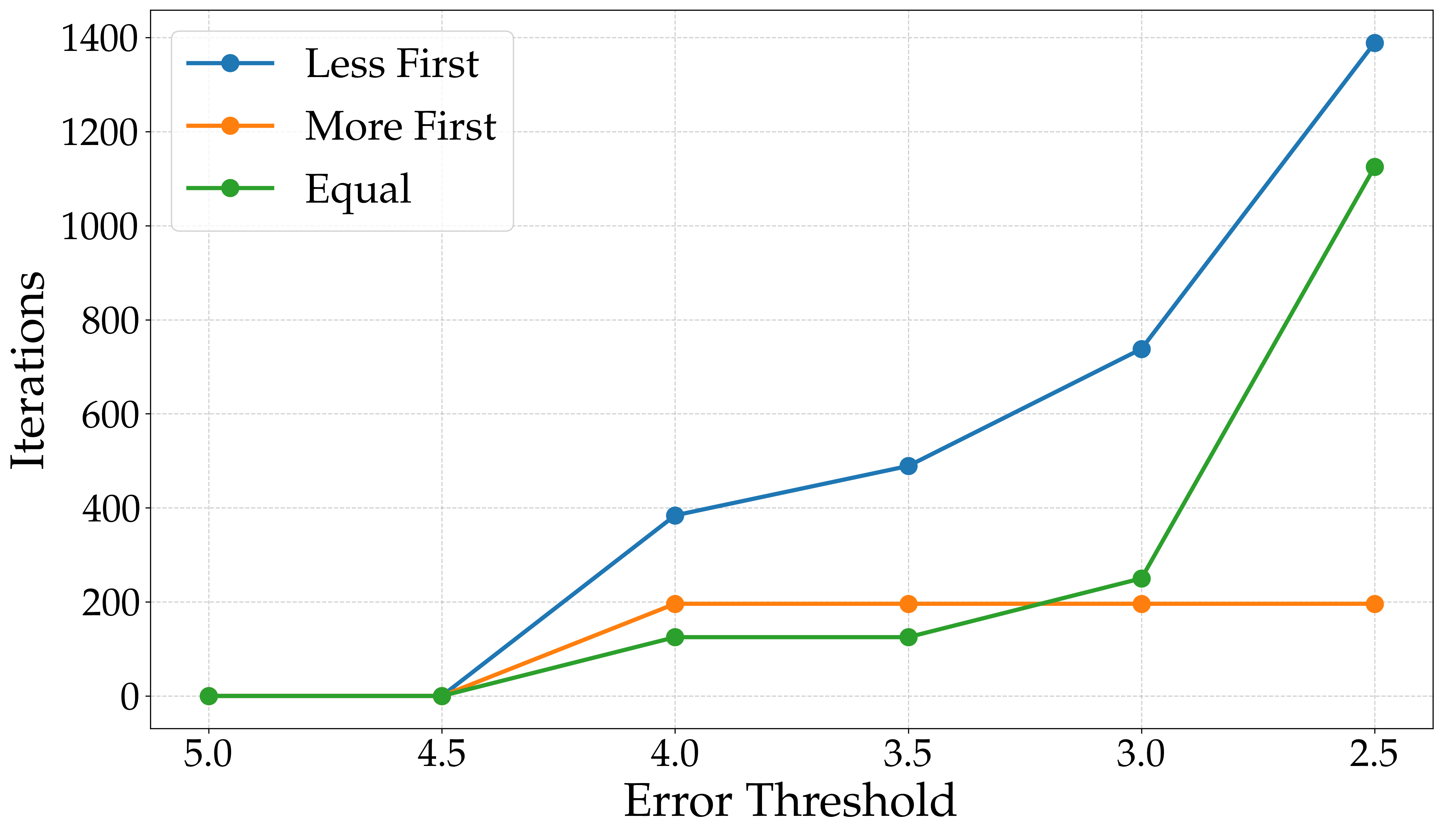}
        \caption{Threshold: 2.5}
        \label{fig:iters_thresh_2.5}
    \end{subfigure}

    \caption{Number of iterations required to reach reconstruction error thresholds for different allocation strategies. Each subplot corresponds to a fixed error threshold. The “more-first” strategy consistently reaches the thresholds faster, especially for tighter reconstruction targets. In subplot (a), the “less-first” strategy fails to reach the threshold even after 10,000 iterations.}
    \label{fig:iters_vs_threshold}
\end{figure}

\textit{Observations:} The more-first iterations strategy consistently reaches target reconstruction thresholds faster than the equal or less-first strategies. This aligns with our intuition that prioritizing the early components—those with the greatest influence on downstream error propagation—leads to quicker convergence. In contrast, less-first allocation delays learning the principal directions, requiring more total iterations to reach the same accuracy. This suggests that our theoretical insights can lead to more computationally efficient algorithms for low-rank approximation.

\section{More results on the LoRA experiments}\label{app:4}

\textbf{Adaptation performance across datasets.} 
Figure \ref{fig:bubbleplot} displays the relationship between parameter efficiency (measured by test accuracy per training epoch) and total training epochs for different model architectures. 
Here, Rank-1 architectures correspond to just using $r=1$ for different number of epochs; Rank-2 architectures correspond to $r=2$, where the components are trained for different \textit{combinations} of total epochs (e.g., some models have been trained with 1$\rightarrow$1 epochs, while others have been trained with $10 \rightarrow 10$ epochs; more about this in the next paragraph), and so on.
The bubble sizes represent the relative efficiency of each configuration.

Across all datasets, we observe that sequential rank-1 approaches (Rank-1, Rank-2, and Rank-3) consistently achieve higher parameter efficiency compared to standard LoRA. 
However, sequential rank-1 models require more total training to achieve comparable accuracy, thus creating a tradeoff to be taken in consideration in practice, but still maintain favorable parameter-to-performance ratios for some cases.

\textbf{Sequential training paths.} 
Figure \ref{fig:sequential_paths} illustrates the effectiveness of different sequential training paths for all the cases, where each path represents a sequence of component training durations. 
For example, path ``1$\rightarrow$3$\rightarrow$5'' indicates a rank-3 LoRA where the first component received 1 epoch of training, the second component 3 epochs, and the third component 5 epochs.

In all cases, it is evident that good first component implies (almost all the times) a better combined final model: front-loaded training schedules perform better, indicating that the first component captures most of the necessary adaptation, with diminishing returns for extensive training of later components.

Figure \ref{fig:low_rank_architectures} depicts a similar picture, where for every rank-$r$ architecture, we depict how well the model performs (the variance bars indicate how good or bad the model ends up, depending on the different number of epochs we spend on each component of the low-rank sequential architecture).

\textbf{Impact of baseline model quality.}
Our experiments across the three datasets reveal that the effectiveness of sequential LoRA adaptation is influenced by, but not dependent on, the quality of the baseline model. 
Even with the relatively poor CIFAR100 baseline, sequential LoRA successfully adapts to new classes, albeit with lower absolute performance compared to the better-initialized MNIST case.

This observation has practical implications: sequential rank-1 adaptation offers a viable approach for model extension even when the initial model is suboptimally trained. 
The method provides a parameter-efficient way to incrementally improve model capabilities without full retraining, regardless of the starting point quality, often leading to better results than regular LoRA, but with the expense of more computation. 
\textit{We again note the interesting property our approach introduces: sequential rank-1 does not require to know apriori the rank of the adaptation; one could check online whether accuracy is sufficient and stop further training. Such a property lacks in standard LoRA: either the user needs to know a good value for $r$, or one needs to consider different $r$ values from scratch before making the final decision.}

\textbf{Error propagation analysis.}
Our theoretical analysis predicted that errors in early components of sequential learning would propagate to later stages. The empirical results across all three datasets confirm this prediction. For all datasets, we observe that when the first component is poorly trained (1 epoch), the final performance of rank-3 models is substantially lower than when the first component receives adequate training (5-10 epochs), even when later components are well-trained.

This error propagation effect is most pronounced for MNIST, where the accuracy difference between paths ``1$\rightarrow$10$\rightarrow$10'' and ``10$\rightarrow$1$\rightarrow$1'' can be as large as 5-7 percentage points. The effect is less dramatic but still observable for CIFAR100, where the overall lower performance baseline makes the relative impact of component quality more uniform.

These findings validate our theoretical error bounds and highlight the importance of carefully allocating computational resources across sequential components, with particular attention to early components that form the foundation for subsequent adaptation steps.

\begin{figure}[h]
    \centering
    \begin{subfigure}[b]{0.75\textwidth}   \includegraphics[width=\linewidth]{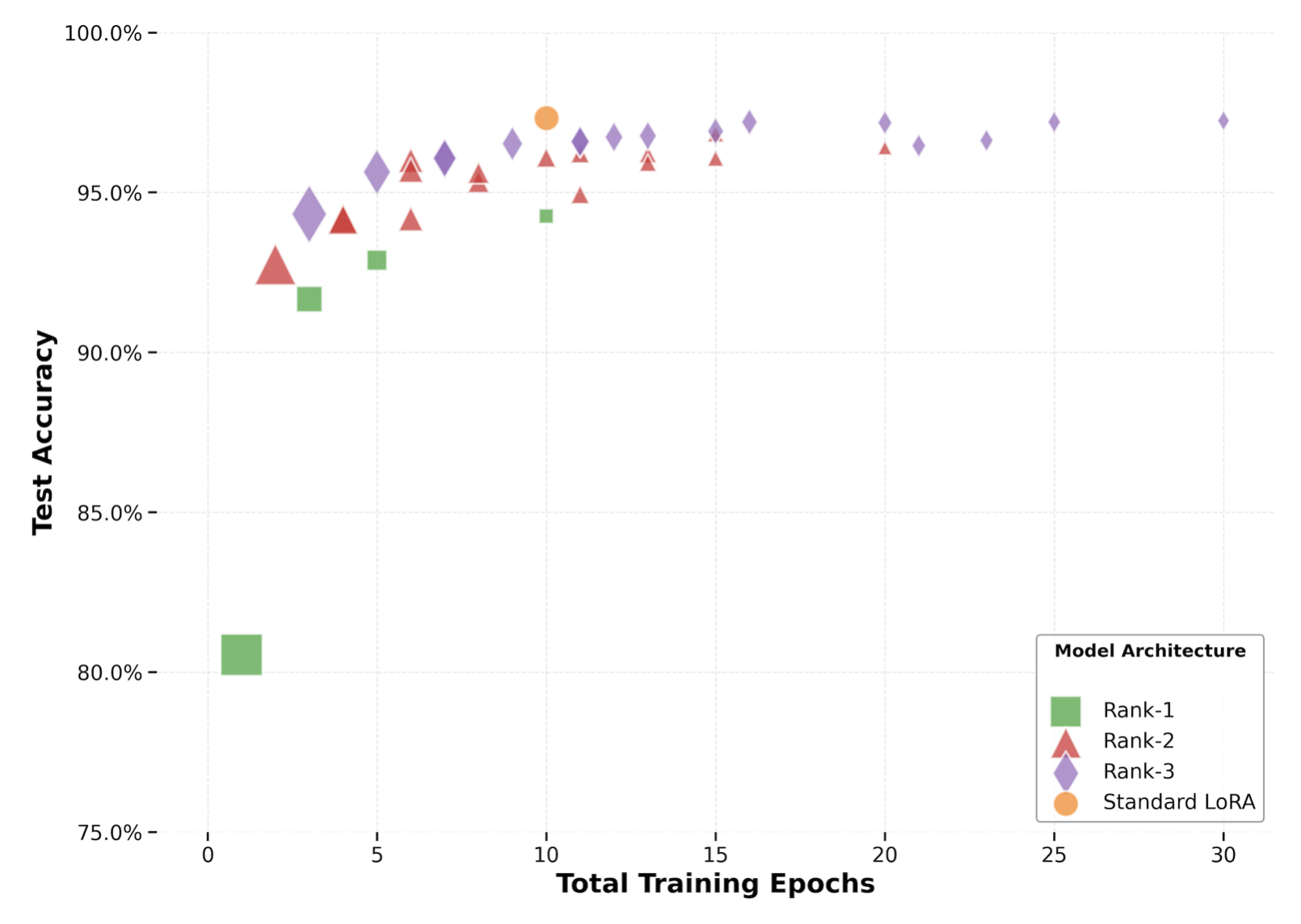}
        \label{fig:MNIST_sequential_paths}
    \end{subfigure}
    \vspace{-0.2cm}
    \inftybar
    \vspace{-0.1cm}
    \begin{subfigure}[b]{0.75\textwidth}
    \includegraphics[width=\linewidth]{CIFAR10_bubbluplot.png}
        \label{fig:CIFAR10_sequential_paths}
    \end{subfigure}
    \vspace{-0.2cm}
    \inftybar
    \vspace{-0.1cm}
    \begin{subfigure}[b]{0.75\textwidth}
    \includegraphics[width=\linewidth]{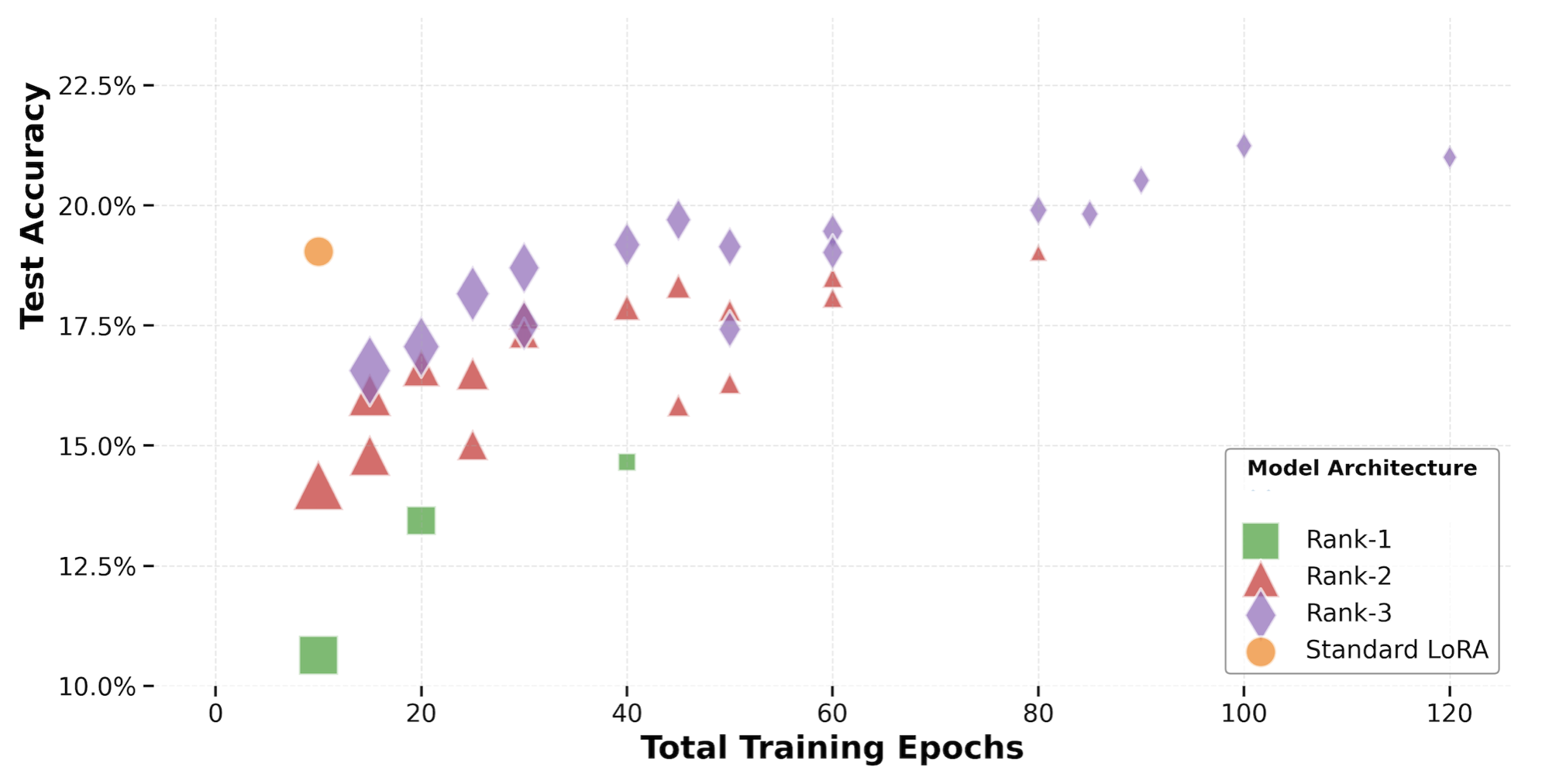}
        \label{fig:CIFAR100_sequential_paths}
    \end{subfigure}
    \caption{Test accuracy of sequential rank-1 LoRA components when adapting to new classes across the three datasets. \textit{Top:} MNIST. \textit{Center:} CIFAR10. \textit{Bottom:} CIFAR100. Note that, on purpose, the pretrained models are trained with good (MNIST), mediocre (CIFAR10) and bad (CIFAR100) accuracy.}
    \label{fig:bubbleplot}
\end{figure}

\begin{figure}[h]
    \centering
    \begin{subfigure}[b]{0.7\textwidth}        \includegraphics[width=\linewidth]{MNIST_sequential_paths.png}
        \label{fig:MNIST_sequential_paths}
    \end{subfigure}
    \begin{subfigure}[b]{0.7\textwidth}
    \includegraphics[width=\linewidth]{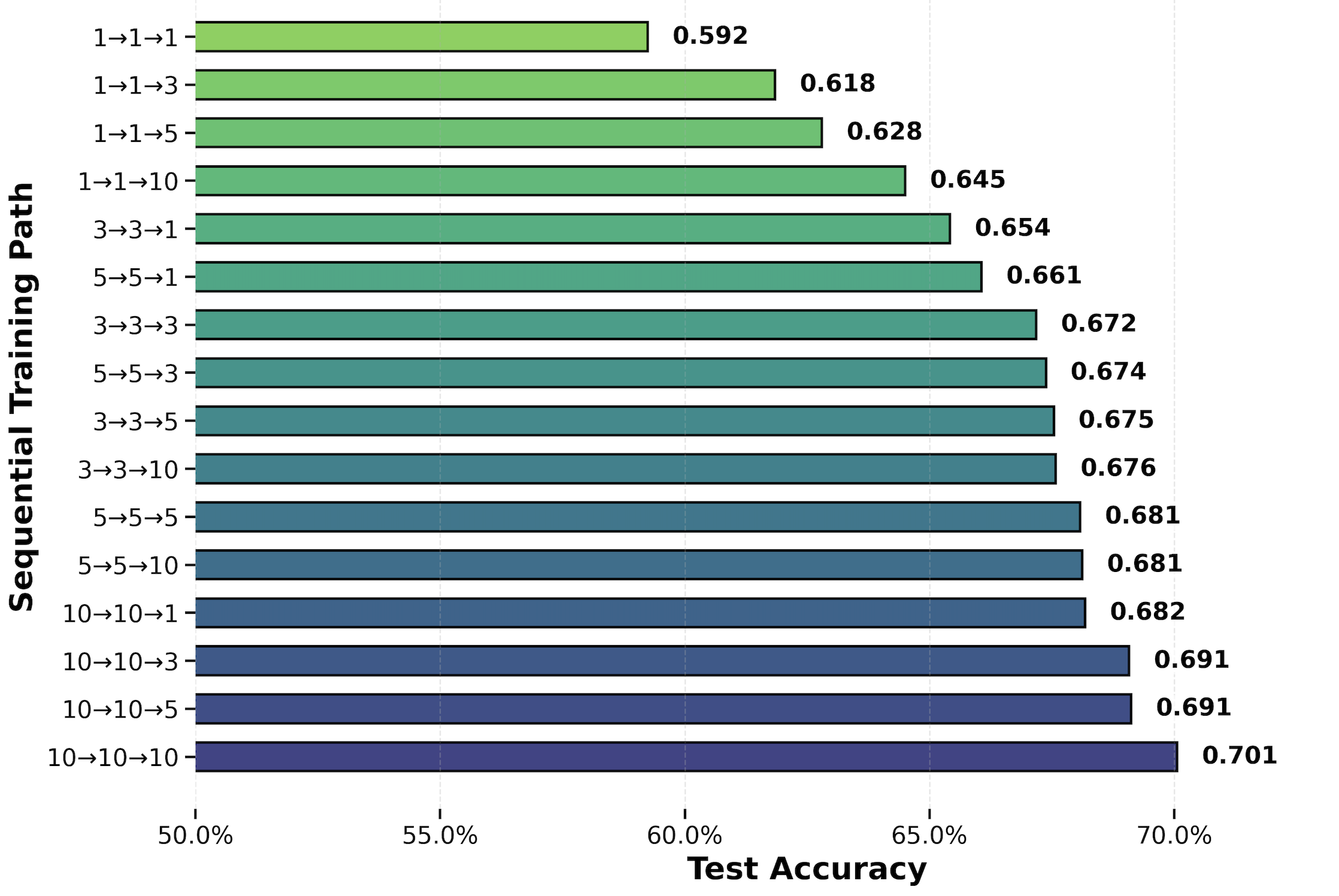}
        \label{fig:CIFAR10_sequential_paths}
    \end{subfigure}

    \begin{subfigure}[b]{0.7\textwidth}
    \includegraphics[width=\linewidth]{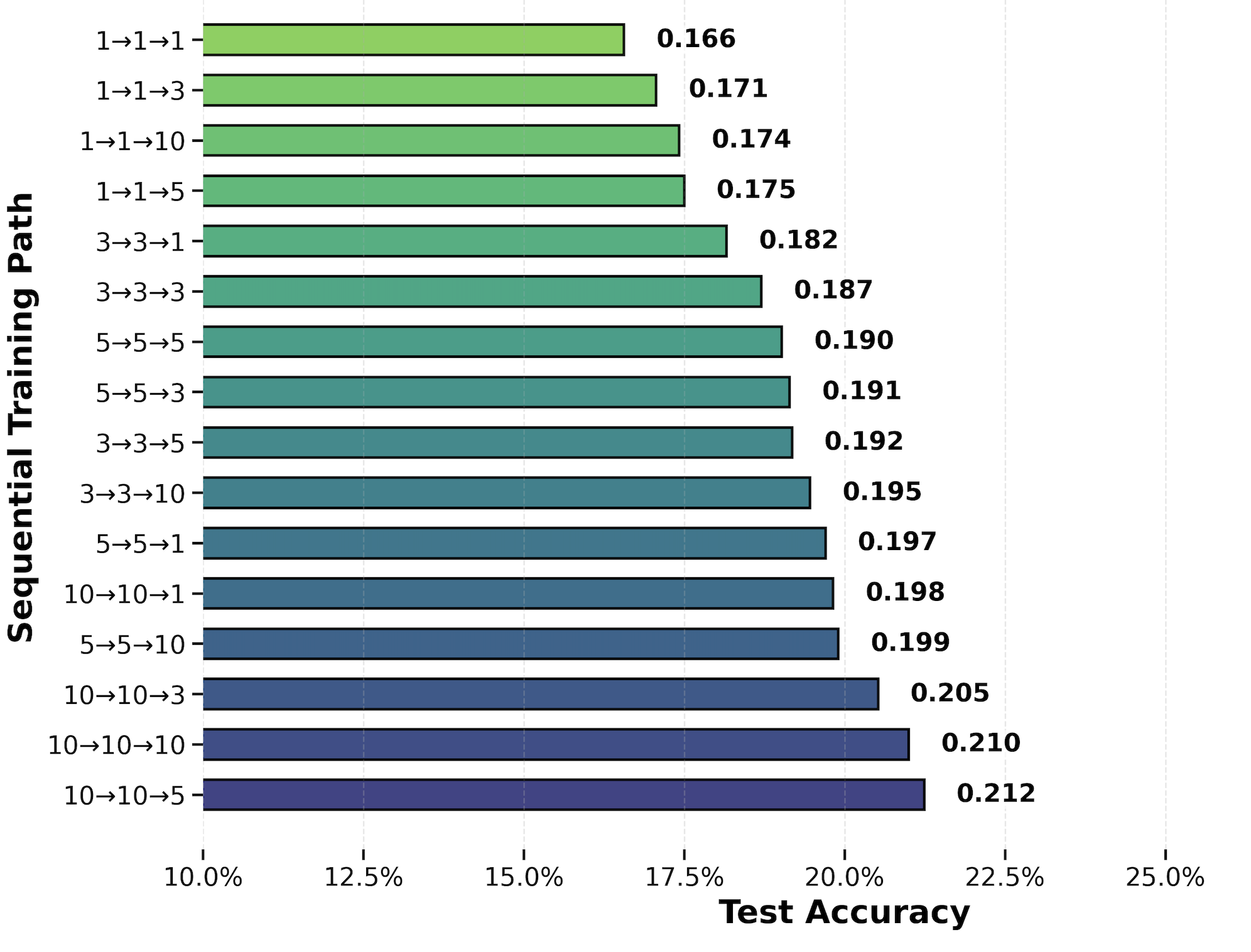}
        \label{fig:CIFAR100_sequential_paths}
    \end{subfigure}
    \caption{The bubble sizes represent the relative efficiency of each configuration.}
    \label{fig:sequential_paths}
\end{figure}

\begin{figure}[h]
    \centering
    \begin{subfigure}[b]{0.7\textwidth}        \includegraphics[width=\linewidth]{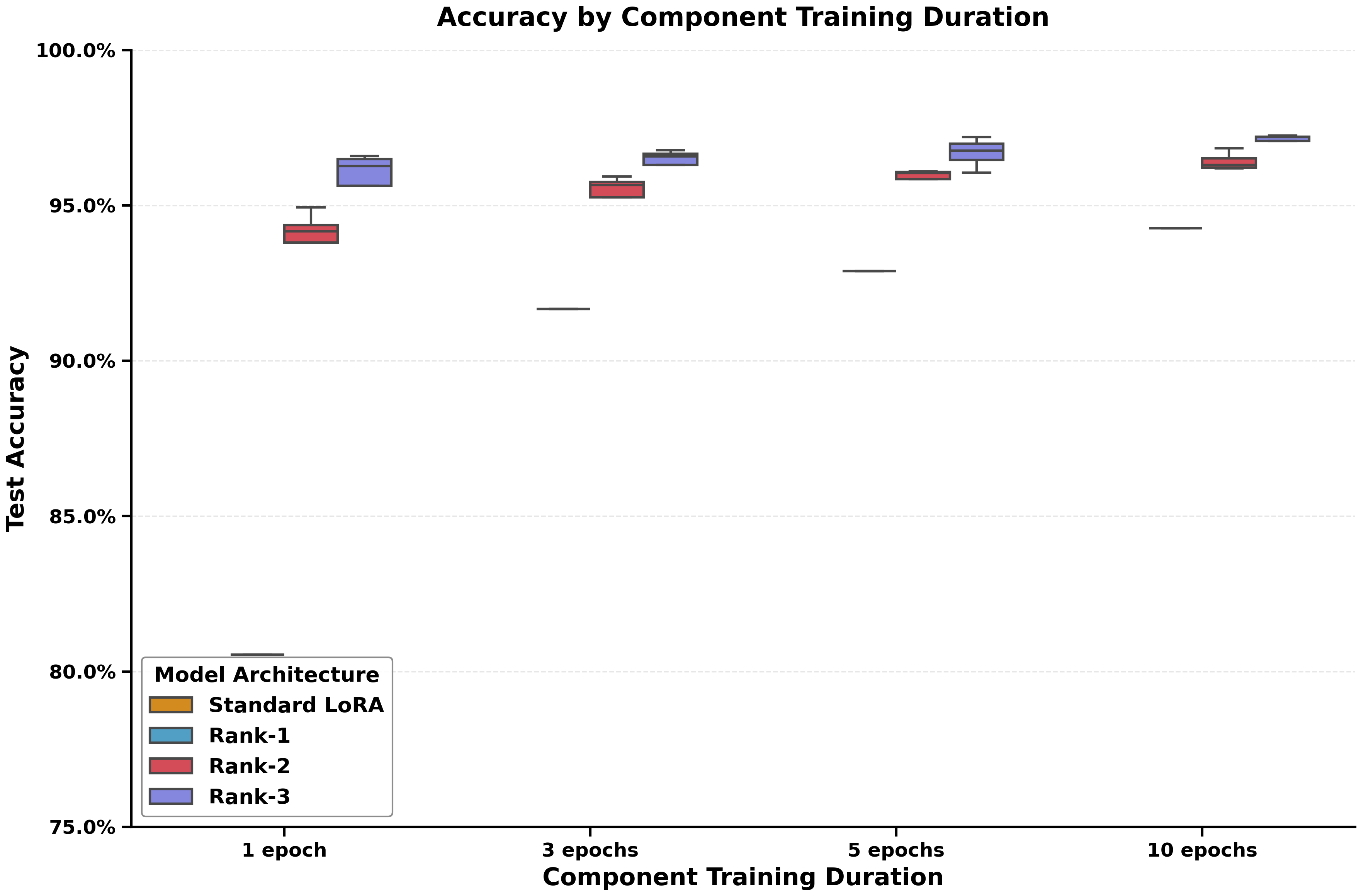}
        \label{fig:MNIST_sequential_paths}
    \end{subfigure}
    \begin{subfigure}[b]{0.8\textwidth}
    \includegraphics[width=\linewidth]{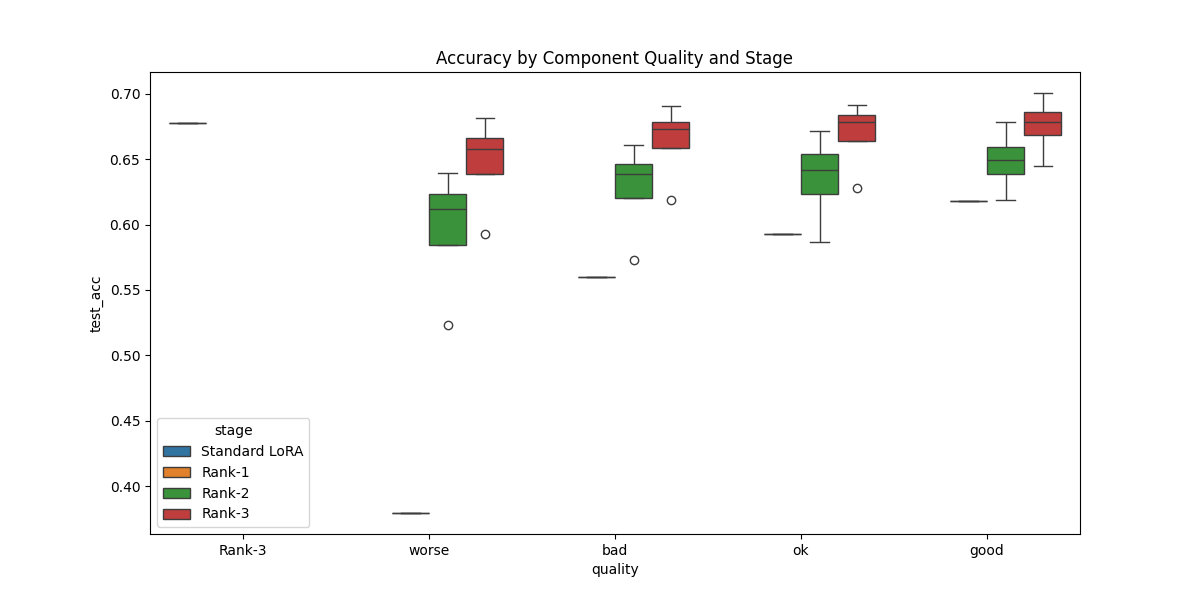}
        \label{fig:CIFAR10_sequential_paths}
    \end{subfigure}

    \begin{subfigure}[b]{0.7\textwidth}
    \includegraphics[width=\linewidth]{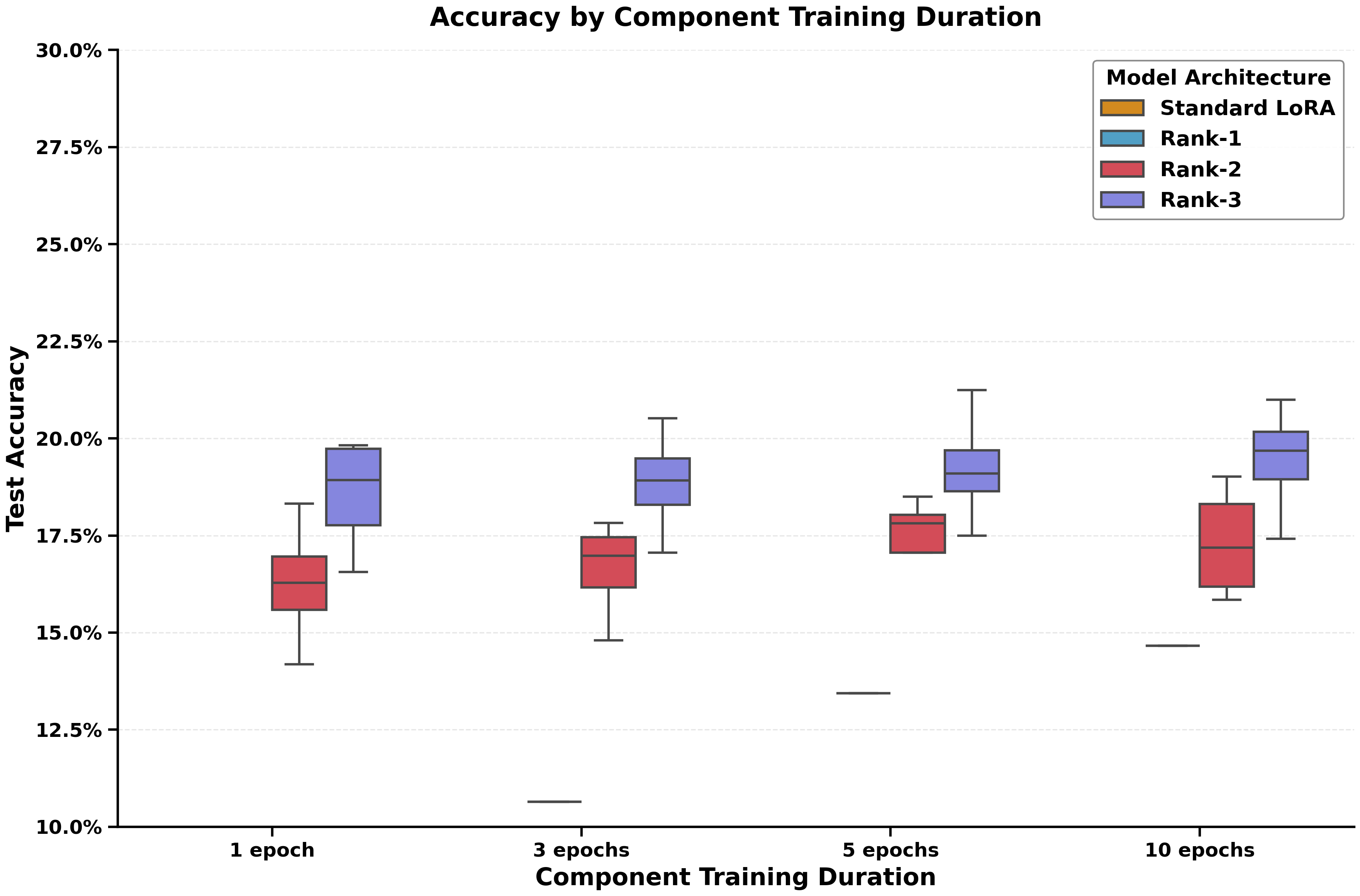}
        \label{fig:CIFAR100_sequential_paths}
    \end{subfigure}
    \caption{For every rank-$r$ architecture, we depict how well the model performs (the variance bars indicate how good or bad the model ends up, depending on the different number of epochs we spend on each component of the low-rank sequential architecture).}
    \label{fig:low_rank_architectures}
\end{figure}

\end{document}